\documentclass[10pt,twocolumn,letterpaper]{article}

\usepackage{cvpr}              % To produce the CAMERA-READY version
\definecolor{cvprblue}{rgb}{0.21,0.49,0.74}
\usepackage[pagebackref,breaklinks,colorlinks,allcolors=cvprblue]{hyperref}
\usepackage{graphicx}
\usepackage{amsmath}
\usepackage{amssymb}
\usepackage{booktabs}
\usepackage{amsthm}  
\newtheorem{statement}{Statement}
\newtheorem{theorem}{Theorem}[section]   % 按章节编号
\newtheorem{lemma}[theorem]{Lemma}

\theoremstyle{definition}

\usepackage{algorithm}
\usepackage{booktabs}
\usepackage{graphicx}
\usepackage{svg}
\usepackage{subcaption}
\usepackage{flushend}
\usepackage{stfloats} % For keeping figre* on the same page
\usepackage{appendix}
\usepackage{listings}
\usepackage{xcolor}
\usepackage{tocloft}
\definecolor{bggray}{rgb}{0.97,0.97,0.97}
\lstdefinestyle{pythonstyle}{
    language=Python,
    backgroundcolor=\color{bggray},
    frame=lines,
    basicstyle=\ttfamily\small,
    numbers=left,
    numberstyle=\tiny,
    stepnumber=1,
    numbersep=8pt,
    breaklines=true,
    showstringspaces=false,
    keywordstyle=\color{blue},
    stringstyle=\color{green!50!black},
    commentstyle=\color{gray},
    tabsize=4
}
\usepackage[noend]{algpseudocode}
\def\paperID{14376} % *** Enter the Paper ID here
\def\confName{CVPR}
\def\confYear{2026}
\title{Functional Mean Flow in Hilbert Space}

\author{
Zhiqi Li \qquad
Yuchen Sun \qquad
Greg Turk \qquad
Bo Zhu \\
Georgia Institute of Technology \\
{\tt\small zli3167@gatech.edu, yuchen.sun.eecs@gmail.com, turk@cc.gatech.edu, bo.zhu@gatech.edu}
}

\begin{document}

%%%%%%%%% BODY TEXT
\maketitle

\begin{abstract}
We present Functional Mean Flow (FMF) as a one-step generative model defined in infinite-dimensional Hilbert space. FMF extends the one-step Mean Flow framework \cite{geng2023onestep} to functional domains by providing a theoretical formulation for Functional Flow Matching and a practical implementation for efficient training and sampling. We also introduce an $x_1$-prediction variant that improves stability over the original $u$-prediction form. The resulting framework is a practical one-step Flow Matching method applicable to a wide range of functional data generation tasks such as time series, images, PDEs, and 3D geometry.

\end{abstract}
\section{Introduction}
\label{sec:intro}

Functional generative models (e.g., \cite{dupont2022data, kerrigan2024functional, bond2024infty}) represent data in the form of continuous functions \cite{dupont2022data}, where the underlying generative process is modeled as a probability distribution defined over function spaces \cite{dupont2022generative}. Compared with standard generative models defined in discrete space, the main advantage of a functional model lies in its ability to subsample coordinates while maintaining a continuous functional representation, effectively decoupling memory and runtime cost from data resolution. This property enables training and sampling at arbitrary spatial or temporal resolutions. For example, \textit{Infty-Diff} \cite{bond2024infty} employs non-local integral operators to map between Hilbert space, achieving up to an $8\times$ subsampling rate without compromising quality.

\begin{figure}[b!]
  \centering
  \includegraphics[width=0.5\textwidth]{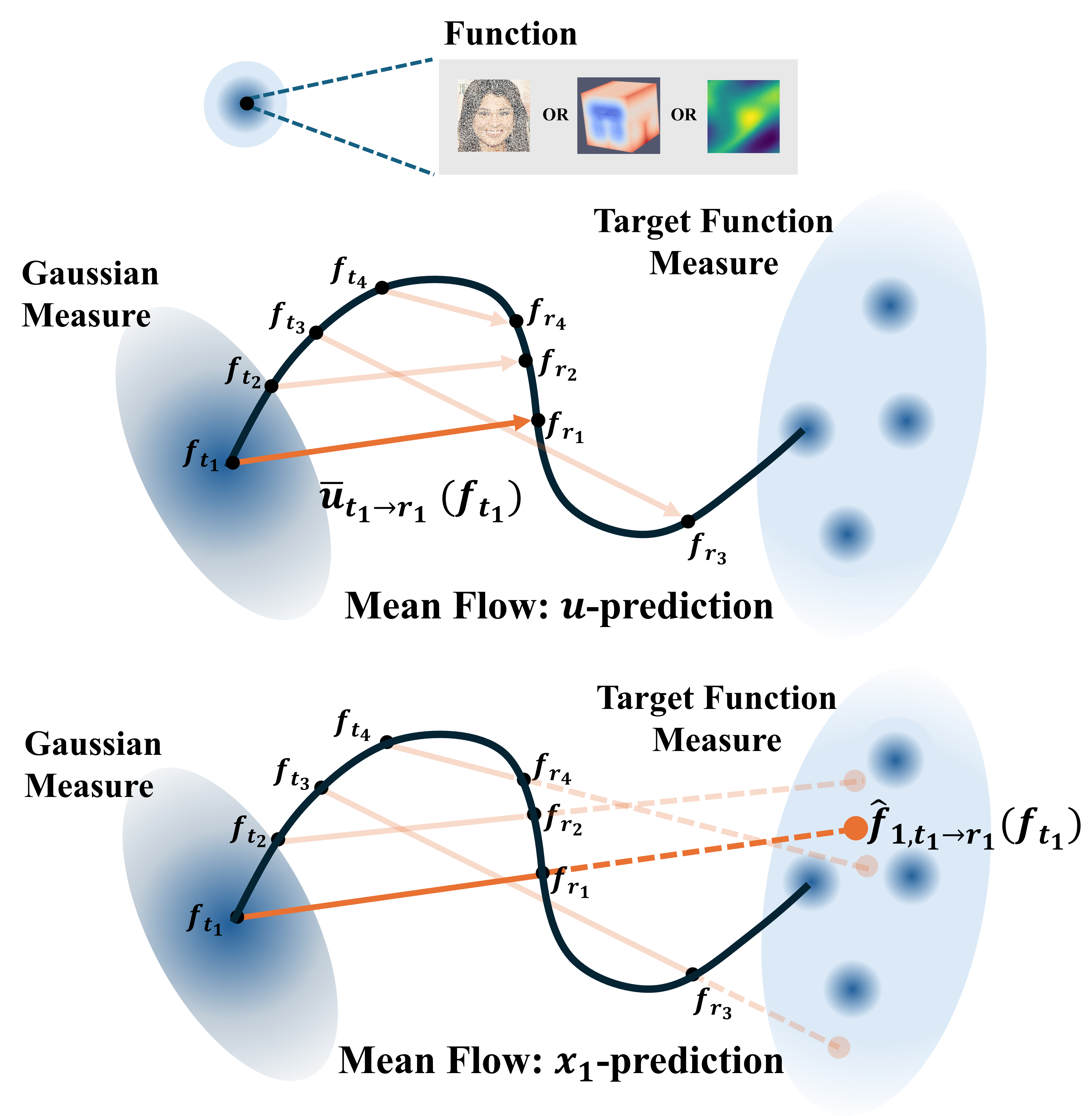}
  \caption{  Illustration of Functional Mean Flow. 
  The figure shows a 2D projection of the infinite-dimensional function space. 
  During generation, the flow transports a Gaussian measure to the target function measure. 
  The $u$-prediction FMF models the mean velocity $\bar u_{t\to r}(f_t)$ between any two points $f_t$ and $f_r$ along the flow trajectory, 
  while the $x_1$-prediction FMF estimates the expected position $\hat{f}_{1,t\to r}(f_t)$ reached by continuing the mean velocity $\bar u_{t\to r}(f_t)$ for the remaining distance $1-t$. 
  Both $u$- and $x_1$-prediction FMFs support one-step generation, formulated respectively as 
  $f_1 = f_0 + \bar u_{0\to 1}(f_0)$ and $f_1 = \hat{f}_{1,0\to 1}(f_0)$.}
  
  %\caption{Mean Flow training moves along a path through function space to create a target function. Traditional Mean Flow training makes use of $u$-prediction, which creates functions at different locations $r$ along a path. During training, our new $x_1$-prediction instead creates functions that are all in the target distribution.}
  \label{fig:training}
\end{figure}

As many other generative models, such as Diffusion \cite{ho2020denoising, song2019generative, song2021score} or Flow Matching \cite{song2021denoising, liu2023flow, lipman2023flow}, the performance of functional generative models is also limited by the need for many sampling steps during inference. To address this bottleneck, recent work explores one-step or few-step methods that directly approximate the endpoint transport. Among them, \textit{Mean Flow} \cite{geng2025mean} provides a principled approach by predicting the time-averaged velocity instead of the instantaneous velocity used in standard Flow Matching. This design captures the overall transport in a single update, enabling efficient one-step sampling and achieving 50\%–70\% better FID than previous one-step models.

Extending one-step generation to Functional Flow Matching is fundamentally challenging in two aspects: (1) the infinite-dimensional Hilbert-space setting makes modeling highly non-trivial, as finite-dimensional intuitions no longer apply and the modeling confronts the inconsistency between marginal and conditional flows, making it infeasible to generalize the finite-dimensional mean-velocity formulation to infinite-dimensional functional spaces; (2) functional derivatives and operator-valued velocity fields introduce numerical instability, complicating optimization and adversely affecting convergence across different functional generation tasks.

%Extending one-step generation to functional flow matching is fundamentally challenging in two aspects: (1) the infinite-dimensional Hilbert-space setting makes modeling highly non-trivial—finite, where dimensional intuitions no longer apply,在建模的时候需要面临the inconsistency between marginal and conditional flows and valid formulations require rigorous mathematical tools from functional analysis and measure theory,;  (2) the numerical instability introduced by functional derivatives and operator-valued velocity fields complicates optimization and adversely affects the convergence of the model across different functional generation tasks. 

%(1) the inconsistency between marginal and conditional flows in Hilbert space breaks the equivalence that underlies finite-dimensional Mean Flow, making it infeasible to generalize the finite-dimensional mean-velocity formulation to infinite-dimensional functional spaces; (2) the numerical instability introduced by functional derivatives and operator-valued velocity fields complicates optimization and adversely affects the convergence of the model across different functional generation tasks. 
To address these challenges, we derive a new transport formulation based on the Fréchet derivative of two-parameter flows, which establishes the Mean Flow formulation in infinite-dimensional spaces and resolves the theoretical inconsistency between conditional and marginal dynamics. In addition, we reformulate the learning objective as an equivalent conditional loss with a stop-gradient approximation and introduce an $x_1$-prediction variant that predicts the expected endpoint by extrapolating the mean velocity, instead of predicting the mean velocity itself. These developments together constitute our proposed framework, \textbf{Functional Mean Flow (FMF)}, which enables stable and efficient one-step functional generation across a wide range of tasks in infinite-dimensional spaces.

We summarize our contributions as follows:  
\begin{enumerate}  
\item  We derived the infinite-dimensional mean-velocity formulation, establishing a mathematically sound framework for one-step generation in Hilbert space. 
\item  We introduce, for the first time, the $x_1$-prediction variant of Mean Flow and show that it exhibits improved stability over the original $u$-prediction formulation on certain tasks, thereby broadening the applicability of the Mean Flow framework.
\item  We demonstrate the effectiveness of the proposed method across a range of functional tasks, including time series modeling, image generation, PDEs, and 3D shape generation.

\end{enumerate} 

%\newpage
%\section{Method}\label{sec:method}

%We first introduce the background of Functional Flow Matching and the relevant notation and building on Functional Flow Matching, we present the Functional MeanFlow method proposed in this paper.

\section{Functional Flow Matching}
Functional Flow Matching (FFM) \cite{kerrigan2024functional} extends classical Flow Matching from finite-dimensional Euclidean spaces to infinite-dimensional function spaces.
Let $(\mathcal{F}, \langle \cdot,\cdot \rangle_\mathcal{F})$ be a separable Hilbert space of functions with the Borel $\sigma$-algebra $\mathcal{B}(\mathcal{F})$, and let $\mu_0 = \mathcal{N}(m_0, C_0)$ be a reference Gaussian measure on $\mathcal{F}$ with mean $m_0 \in \mathcal{F}$ and covariance operator $C_0: \mathcal{F} \to \mathcal{F}$, FFM learns a time-dependent velocity field $u : [0,1] \times \mathcal{F} \to \mathcal{F}$ that transports $\mu_0$ to a target distribution $\mu_1 = \nu$ through a continuous path of measures $(\mu_t)_{t\in[0,1]}$ satisfying the weak continuity equation
\begin{equation}\label{eq:weak-form}
\begin{aligned}
&\int_0^1\!\int_{\mathcal F}
\!\big(\partial_t\psi(g,t)
   \!+\!\langle u_t(g),\nabla_g\psi(g,t)\rangle_\mathcal{F}\big)
   {\mathrm{d}}\mu_t(g){\mathrm{d}}t \!=\! 0,\\
&\mu_{t}|_{t=0}=\mu_0,\;\mu_{t}|_{t=1}=\mu_1,
\end{aligned}
\end{equation}
for all appropriate test functions $\psi : \mathcal{F}\times[0,1]\to\mathbb{R}$.

Sampling $f_0 \sim \mu_0$, one obtains a generated function by integrating
\begin{equation}\label{eq:ode}
\frac{{\mathrm{d}} f_t}{{\mathrm{d}} t} = u(t, f_t),
\quad f_t|_{t=0}=f_0,
\end{equation}
whose terminal state satisfies $f_1 \sim \nu$.  For velocity field $u_t$, the associated flow $\phi_t : \mathcal{F} \to \mathcal{F}$ are defined as maps satisfying $f_t = \phi_t(f_0)$ for all $f_0$ and $f_t$ in \autoref{eq:ode}.  The flow $\phi_t$ satisfies the functional differential equation
\begin{equation}\label{eq:path_velocity_definition}
\frac{\partial}{\partial_t} \phi_t = u_t\circ \phi_t, \qquad \phi_0 = \mathrm{Id}_{\mathcal{F}},
\end{equation}
 where $\mathrm{Id}_{\mathcal{F}}$ denotes the identity operator on $\mathcal{F}$.  The path of measures $(\mu_t)_{t\in[0,1]}$ can be generated by the pushforward of the flow $\mu_t = (\phi_t)_{\sharp}\mu_0$, thereby extending the continuous transport formulation to infinite-dimensional Hilbert spaces.

To make the training $\mathcal{L}(\theta) = \mathbb{E}_{t,g\sim\mu_t} \big[\|u_t(g) - u_t^\theta(g)\|^2_\mathcal{F}\big]$ tractable, where the reference marginal velocity field $u_t$ cannot be computed analytically, FFM introduces conditional velocity $u_t^f$ conditioned on the target function $f\sim \nu$ and corresponding conditional paths of measures $(\mu_t^f)_{t\in[0,1]}$ that interpolate between $\mu_0$ and a $f$-centered measure $\mu_1^f$.  Marginalizing these conditionals yields the global measure path and velocity
\begin{equation}\label{eq:ecpectional_conditional}
    \begin{aligned}
        \mu_t(A) &=\int_{\mathcal{F}}\mu_t^f(A) \mathrm{d}\nu(f),\\
        u_t(g) &=\int_{\mathcal{F}} u_t^f(g)\frac{\mathrm{d}\mu_t^f}{\mathrm{d}\mu_t}(g)\mathrm{d} \nu(f),        
    \end{aligned}
\end{equation}
for arbitrary $A\in \mathcal{B}(\mathcal{F})$ where $\frac{\mathrm{d}\mu_t^f}{\mathrm{d}\mu_t}$ is the Radon–Nikodym derivative.  In practice, the conditional paths $\mu_t^f$ are typically chosen to be Gaussian measure $\mu_t^f =\mathcal{N}(m_t^f, (\sigma_t^f)^2 C_0)$ with $m_t^f = tf$, $\sigma_t^f = 1 - (1-\sigma_{\min})t$ and a small positive number $\sigma_{\text{min}}$. The conditional velocity and its associated flow admits a closed form
\begin{equation}\label{eq:selection_conditional}
\begin{aligned}
\phi_t^f(f_0)&= \sigma_t^ff_0 + m_t^f =(1-(1-\sigma_{\text{min}})t)f_0+tf, \\
u_t^f(g) 
\!&=\! \frac{\dot{\sigma}_t^f}{\sigma_t^f}(g\!-\!m_t^f) \!+\! \dot{m}_t^f
\!=\! \frac{1\!-\!\sigma_{\min}}{1\!-\!(1\!-\!\sigma_{\min})t}(t f\!-\!g) \!+\! f.    
\end{aligned}
\end{equation}
Although the theory requires $\sigma_{\text{min}} > 0$, in practice setting $\sigma_{\text{min}} = 0$ causes no adverse effects \cite{bond2024infty}.

The model is then trained via the conditional loss
\begin{equation}\label{eq:conditional_loss_ffm}
\mathcal{L}_c(\theta)
= \mathbb{E}_{t,f,g\sim \mu_t^f}\!\left[\|u_t^f(g) - u_\theta(t,g)\|_\mathcal{F}^2\right],
\end{equation}
which can be proved equivalent to the marginal loss $\mathcal{L}(\theta)$ up to a constant.  For completeness, the corresponding theorems from \cite{kerrigan2024functional} on Functional Flow Matching are provided in the \textcolor{cvprblue}{Appendix~\ref{appendix:ffm_theorem}}.

\begin{figure*}[t]
  \includegraphics[width=1.0\textwidth]{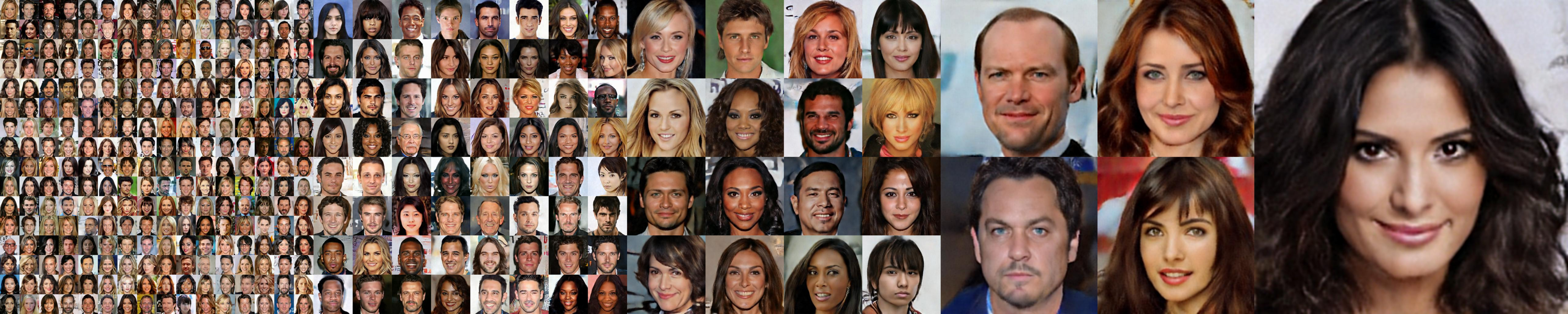}
  \caption{Representing data as functions enables the same model to synthesize images at arbitrary resolutions with different noise levels. The model is trained only on randomly sampled 1/4 subsets of pixels from 256×256 CelebA-HQ images and performs one-step generation. Left to right: 64×64, 128×128, 256×256, 512×512, and 1024×1024.}
  \label{fig:celeb_multi_res}
\end{figure*}

\section{Functional Mean Flow}\label{sec:method}\label{sec:functional_meanflow}
%\subsection{Functional Mean Flow in Hilbert Space}\label{sec:functional_meanflow}
Similar to conventional Flow Matching, Functional Flow Matching also suffers from the drawback that inference requires many integration steps.  To address this limitation, we extend Mean Flow \cite{geng2025mean} to the infinite-dimensional function space for one-step generation.  In addition, We further propose an $\mathbf{x_1}$-prediction variant of Mean Flow, which predicts the intersection between the extrapolated mean velocity line and the terminal point at $t=1$, different from the original $\mathbf{u}$-prediction formulation of Mean Flow. This $x_1$-prediction variant exhibits improved stability on certain task as shown in \autoref{sec:experiment}.

%\bo{Here briefly explain what are u-prediction and x1-prediction for smooth a transition to the next two subsections.}

\subsection{FMF with $u$-prediction}
We first define a two-parameter flow as $\phi_{t\to r} = \phi_r \circ \phi_t^{-1}$, where the inverse map $\phi_t^{-1}$ is guaranteed to exist by the uniqueness of the ODE solution of \autoref{eq:ode}.  Based on $\phi_{t\to r}$, the mean velocity $\bar{u}_{t\to r}: \mathcal{F} \to \mathcal{F}$ is defined as
\begin{equation}\label{eq:mean_velocity_definition}
    \begin{aligned}
        \bar{u}_{t\to r} = \frac{1}{r-t}(\phi_{t\to r} - \mathrm{Id}_{\mathcal{F}}).       
    \end{aligned}
\end{equation}

In Functional Mean Flow, our goal is to learn the target velocity $\bar u_{t\to r}$ through a loss $\mathcal{L}^M(\theta) = \mathbb{E}_{t,r,g\sim\mu_t} \big[\|\bar u_{t\to r}(g) - \bar u_{t\to r}^\theta(g)\|_\mathcal{F}^2\big]$.  However, since the reference mean velocity $\bar{u}_{t\to r}$ has no closed-form expression, similar to Functional Flow Matching, we aim to reformulate the training objective in terms of a conditional field.  The reformulation for Functional Flow Matching relies on the consistency between conditional and marginal velocity fields (\autoref{thm:ffm_thm1}).  In our formulation, however, this consistency breaks down, as $\bar{u}_{t\to r}$ and $\phi_{t\to r}$ do not admit a consistent corresponding conditional field representation.
%The difficulty, however, lies in the fact that $\bar{u}_{t\to r}$ and $\phi_{t\to r}$ does not admit a consistent corresponding conditional field representation, whereas the conditional and marginal velocity fields remain consistent (\autoref{thm:ffm_thm1}), forming the cornerstone of flow matching.
\begin{statement}[Mismatch Between Flow and Marginals of Conditional Flow]
\label{thm:statement}
In general, the marginal flow $\phi^{(1)}_{t\to r}(g) =\int_{\mathcal{F}} \phi^f_{t\to r}(g)\frac{\mathrm{d}\mu_t^f}{\mathrm{d}\mu_t}(g)\mathrm{d}\nu(f)$ obtained by taking the expectation over the conditional two-parameter flows $\phi_{t\to r}^f = \phi_r^f \circ (\phi_t^f)^{-1}$ is not equivalent to the two-parameter flow $\phi^{(2)}_{t\to r} = \phi_r \circ (\phi_t)^{-1}$. Here, the superscripts $^{(1)}$ and $^{(2)}$ denote two different ways of computing the marginal two-parameter flow. (see \textcolor{cvprblue}{Appendix~\ref{appendix:statement}} for proof.)
\end{statement}

To address this issue, we first derive an equivalent reformulation of the mean velocity $\bar{u}_{t\to r}$, which relies on the following theorem:

\begin{theorem}[Initial-Time Derivative of Two-Parameter Flow]
\label{thm:flow-diff}
Assume that the dataset measure $\nu$ satisfies $\int_{\mathcal F}\|f\|_\mathcal{F}^2  \mathrm{d}\nu(f)<\infty$, and the conditions of Functional Flow Matching \cite{kerrigan2024functional} hold.  With the conditional flow and conditional velocity chosen in \autoref{eq:selection_conditional}, the corresponding two-parameter flow $\phi_{t\to r}(g)$ is differentiable with respect to $t$ and Fréchet differentiable with respect to $g$ and satisfies, for any $0 <t < r<1$
\begin{equation}\label{eq:flow-diff}
    \begin{aligned}
        &\frac{\partial }{\partial t} \phi_{t\to r}(g) = -D \phi_{t\to r}(g)[u_t(g)],
    \end{aligned}
\end{equation}
where $D\phi_{t\to r}(g):\mathcal{F}\to \mathcal{F}$ is the Fréchet derivative of $\phi_{t\to r}$ at $g$.  This theorem follows from \textcolor{cvprblue}{Lemmas}~\ref{thm:differentiability_of_flow},\ref{thm:differentiability_of_flow2} and \ref{thm:differentiability_of_flow3} in \textcolor{cvprblue}{Appendix~\ref{sec:flow-diff_supporting}}. (see \textcolor{cvprblue}{Appendix~\ref{sec:flow-diff}} for proof.)
\end{theorem}

With \autoref{thm:flow-diff} and the definition of $\bar{u}_{t\to r}$ in \autoref{eq:mean_velocity_definition}, the mean velocity $\bar{u}_{t\to r}$ can be expressed as 
\begin{equation}\label{eq:derivation_u_prediction}
    \begin{aligned}
        \bar u_{t\to r}(g) &\overset{\textcircled{1}}{=}(r\!-\!t)\frac{\partial}{\partial t}[\frac{1}{r\!-\!t} (\phi_{t\to r} \!-\! \mathrm{Id}_{\mathcal{F}})(g)] \!-\! \frac{\partial}{\partial t}\phi_{t\to r}(g)\\
        &\overset{\textcircled{2}}{=} (r\!-\!t)\frac{\partial}{\partial t}\bar{u}_{t\to r}(g) \!+\! D \phi_{t\to r}(g)[u_t(g)]\\
        &\overset{\textcircled{3}}{=} (r\!-\!t)\big(\frac{\partial}{\partial t}\bar{u}_{t\to r}(g) \!+\! D \bar u_{t\to r}(g)[u_t(g)]\big) \!+\!  u_t(g),
    \end{aligned}
\end{equation}
where $\textcircled{1}$ follows from the product rule, $\textcircled{2}$ is obtained by substituting \autoref{eq:flow-diff}, and $\textcircled{3}$ is obtained by substituting \autoref{eq:mean_velocity_definition}.

In the above expression of $\bar{u}_{t\to r}$, the right-hand side still depends on $\bar{u}_{t\to r}$ itself.  Following \cite{geng2025mean,song2023consistency}, we estimate this term using the current prediction of the model with a stop gradient operation and the velocity field $u_t$ can be written as the marginal form of the conditional velocity $u^f_{t}$, and thus we define the conditional loss as
\begin{equation}\label{eq:meanflow_u_loss}
    \begin{aligned}
        \mathcal{L}_{c}^M(\theta) &= E_{t,r,g \sim \mu_t^f,f\sim \mu_1}\big[||(r-t)\text{sg}(\frac{\partial}{\partial t}\bar u_{t\to r}(g)\\
        &+D\bar u_{t\to r}(g)[u^f_t(g)])+u^f_t(g)-\bar u^\theta_{t\to r}(g)||_\mathcal{F}^2\big],
    \end{aligned}
\end{equation}
where $\text{sg}$ means the stop gradient operation and $\bar u_{t\to r}(g)$ in $\text{sg}$ is approximated by $\bar u_{t\to r}^\theta(g)$.  The following theorem establishes that the conditional loss $\mathcal{L}_{c}^M(\theta)$ is equivalent to the marginal loss $\mathcal{L}^M(\theta)$ up to a constant and can therefore be used to train the Functional Mean Flow model.
\begin{theorem}[Equivalence of Mean Flow Conditional and Marginal Losses]
\label{thm:loss-equivalence}
Under the assumptions of \autoref{thm:flow-diff}, we have $\mathcal{L}_c^M(\theta) = \mathcal{L}^M(\theta) + C$ where $C$ is independent of the model parameters~$\theta$.  (see \textcolor{cvprblue}{Appendix~\ref{sec:proof_loss_equivalence_u}} for proof.)
\end{theorem}

\subsection{FMF with $x_1$-prediction}
\begin{figure}[t]
  \centering
  \begin{subfigure}[t]{0.23\textwidth}
    \centering
    \includegraphics[width=\linewidth]{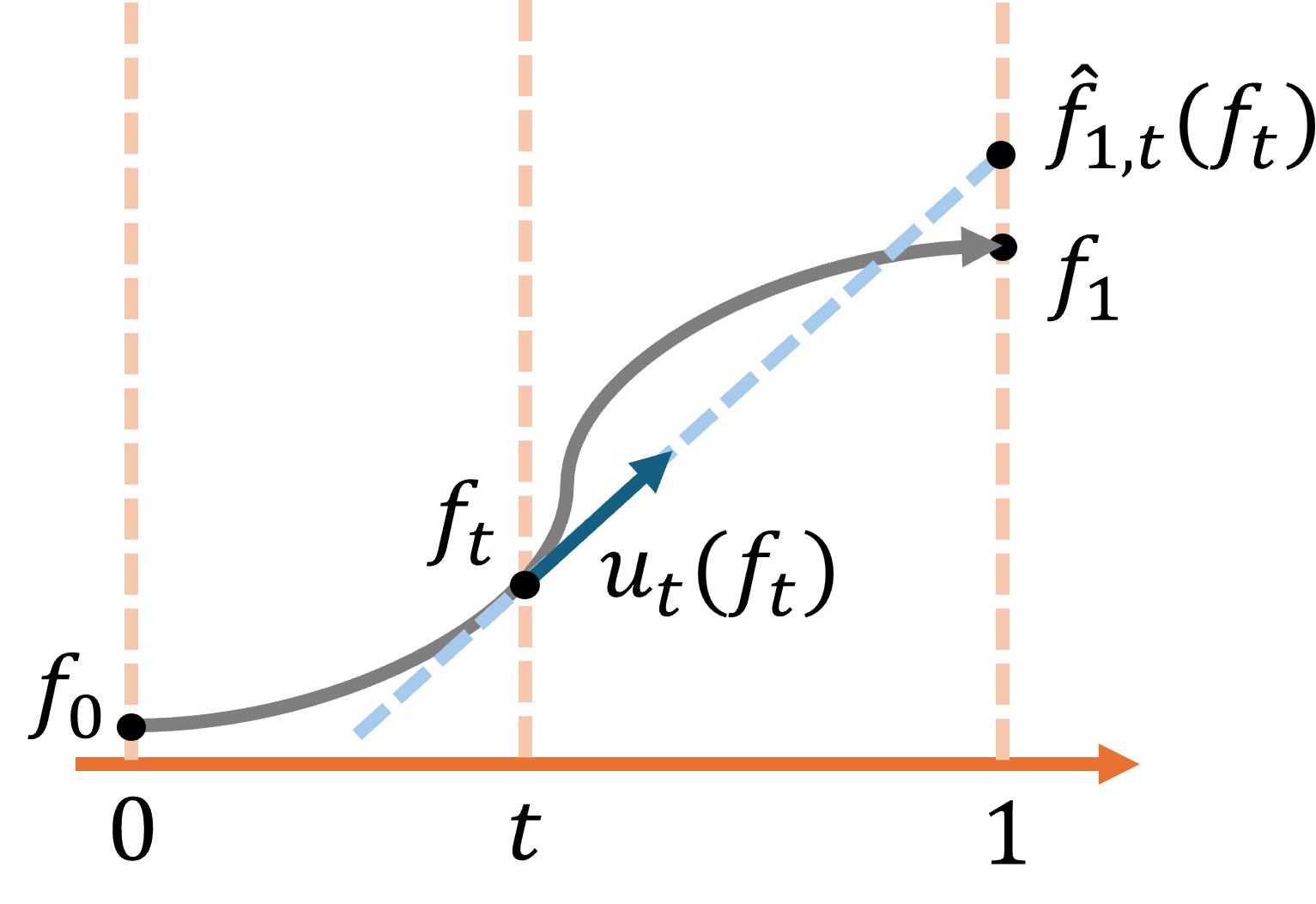}
    \caption{Functional Flow Matching.}
    \label{fig:sub1}
  \end{subfigure}
  \hfill
  \begin{subfigure}[t]{0.23\textwidth}
    \centering
    \includegraphics[width=\linewidth]{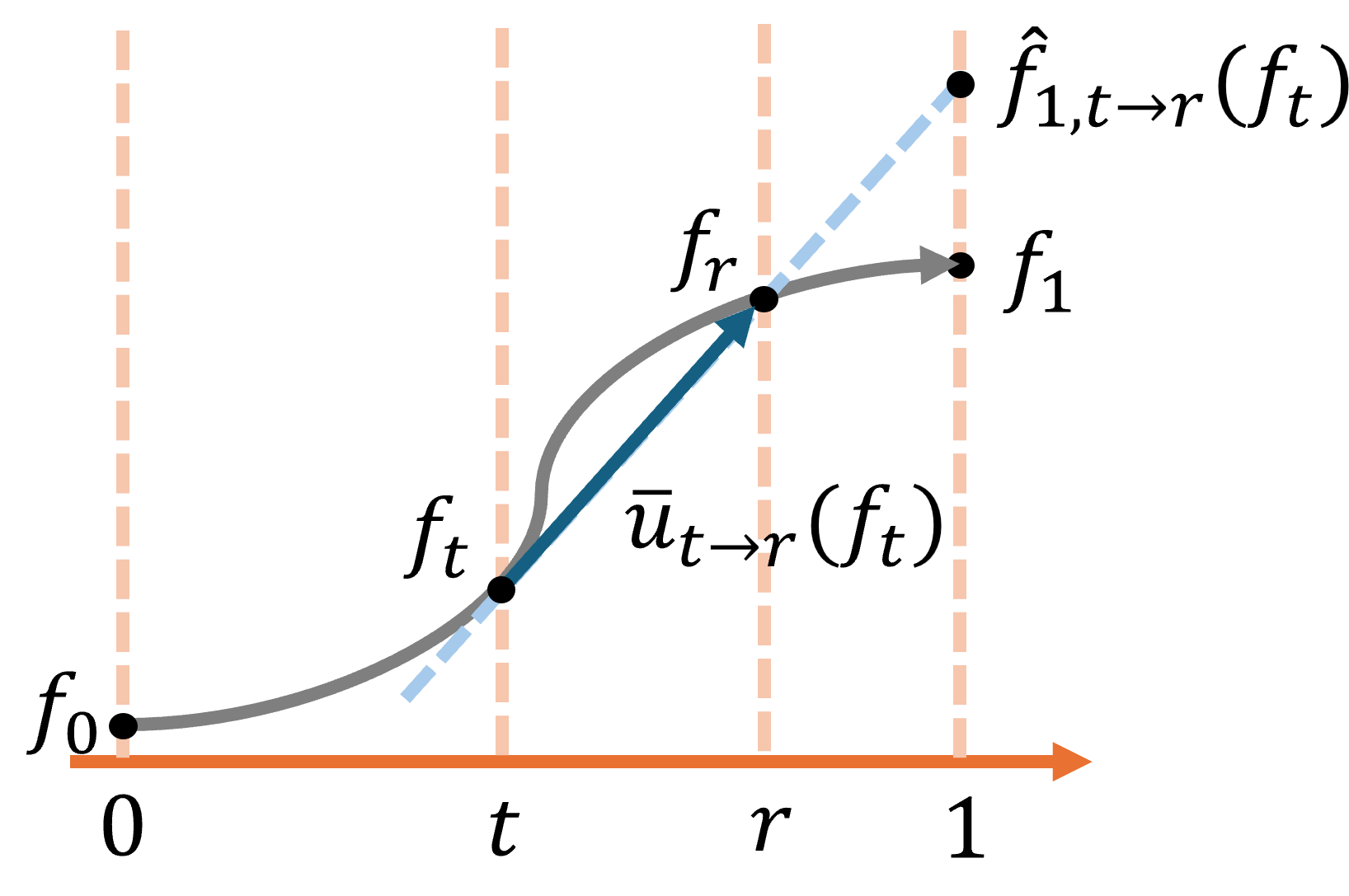}
    \caption{Functional Mean flow.}
    \label{fig:sub2}
  \end{subfigure}
  \caption{From the $x_1$-prediction of Functional Flow Matching to the $x_1$-prediction of Functional Mean Flow. In the left figure, we illustrate the relationship between the $u$-prediction (predicting $u_t(f_t)$) and the $x_1$-prediction (predicting $\hat{f}_{1,t}(f_t)$) in flow matching, which satisfies $\hat{f}_{1,t}(f_t) = (1-t)u_t(f_t) + f_t$. Based on this relationship, we can analogously define the $x_1$-prediction of functional Mean Flow (predict $\hat{f}_{1,t\to r}(f_t)$), satisfying $\hat{f}_{1,t}(f_t) = (1-t)u_{t\to r}(f_t) + f_t$.}
  \label{fig:illustration}
\end{figure}

In addition to the common $u$-prediction, standard Flow Matching also has an $x_1$-prediction variant, as shown in \textcolor{cvprblue}{\autoref{fig:illustration}(a)}.
In Functional Flow Matching, the $u$-prediction estimates the velocity $u_t(f_t)$ at time $t$, while the $x_1$-prediction predicts the intersection $\hat{f}_{1,t}(f_t)$ between the extrapolated $u_t(f_t)$ and $t=1$, satisfying $\hat{f}_{1,t}(f_t) = (1-t)u_t(f_t) + f_t$.  Similarly, in Functional Mean Flow (see \textcolor{cvprblue}{\autoref{fig:illustration}(b)}), the $u$-prediction estimates the mean velocity between $t$ and $r$, and we can define the $x_1$-prediction as the intersection of the extrapolated $\bar u_{t\to r}$ with $t=1$
\begin{equation}\label{eq:relation_u_x1}
    \begin{aligned}
        \hat{f}_{1,t\to r}\!=\!(1\!-\!t)u_{t\to r} \!+\! \mathrm{Id}_{\mathcal{F}} = \frac{1\!-\!t}{r\!-\!t}\phi_{s\to t} \!-\! \frac{1\!-\!r}{r\!-\!t} \mathrm{Id}_{\mathcal{F}}.        
    \end{aligned}
\end{equation}
For the $x_1$-prediction, the Functional Mean Flow loss is $\mathcal{\tilde L}^M(\theta) = \mathbb{E}_{t,r,g\sim\mu_t} \big[\|\hat{f}_{1,t\to r}(g) - \hat{f}_{1,t\to r}^\theta(g)\|_\mathcal{F}^2\big]$.  As with the $u$-prediction, $\mathcal{\tilde L}^M(\theta)$ cannot be optimized directly, and we optimize its corresponding conditional loss instead
\begin{equation}\label{eq:meanflow_x1_loss}
\begin{aligned}
&\mathcal{\tilde L}^M_c(\theta) = E_{t,r,g \sim \mu_t^f,f\sim \mu_1}\big[||\frac{r-t}{1-r}sg((1-t)\frac{\partial}{\partial_t}\hat{f}_{1,t\to r}(g)\\
&+D\hat{f}_{1,t\to r}(g)[\hat{f}_{1,t}^f(g)-g])+\hat{f}_{1,t}^f(g) -\hat{f}_{1,t\to r}^\theta(g)||_\mathcal{F}^2\big], 
\end{aligned}    
\end{equation}
%where $g \sim \mu_t$ is obtained by interpolating between $f_0$ and $f$ as described in \autoref{eq:selection_conditional}.
where $\hat{f}_{1,t}^f(g)$ denotes the conditional value of $\hat{f}_{1,t}(g)$ with respect to $f$, analogous to how $u_t^f(g)$ serves as the conditional counterpart of $u_t(g)$.
$\hat{f}_{1,t}^f(g)$ can be computed as follows (see 
\textcolor{cvprblue}{Appendix~\ref{sec:proof_for_all_x1}} for a detailed derivation):
\begin{equation}\label{eq:conditional_x1_flowmatching}
    \begin{aligned}
        \hat{f}_{1,t}^f(g) = \frac{\sigma_{\text{min}}}{1-(1-\sigma_{\min})t}(g-tf)+f.
    \end{aligned}
\end{equation}
Similar to the equivalent reformulation of $\bar u_{t\to r}$ in \autoref{eq:derivation_u_prediction}, the above $x_1$-prediction conditional loss is derived from the following equivalent reformulation of $\hat{f}_{1,t\to r}$ (see \textcolor{cvprblue}{Appendix~\ref{sec:proof_for_all_x1}} for derivation in details):
\begin{equation}\label{eq:conditional_x1_prediction}
\begin{aligned}
        \hat{f}_{1,t\to r}(g) &= \frac{r-t}{1-t}\big((1-t)\frac{\partial}{\partial t}\hat{f}_{1,t\to r}(g)\\
        &+D\hat{f}_{1,t\to r}(g)[\hat{f}_{1,t}(g)-g]\big)+\hat{f}_{1,t}(g).
\end{aligned}    
\end{equation}
It can be shown that the $x_1$-prediction Functional Mean Flow also admits the following equivalent form:
\begin{theorem}[Equivalence of Mean Flow Conditional and Marginal Losses for $x_1$-prediction]
\label{thm:loss-equivalence_x1}
Under the assumptions of \autoref{thm:flow-diff}, we have $\mathcal{\tilde L}_c^M(\theta) = \mathcal{\tilde L}^M(\theta) + C$ where $C$ is independent of the model parameters~$\theta$.  (see \autoref{sec:proof_for_all_x1} for proof.)
\end{theorem}

In our experiments in \autoref{sec:experiment}, we found that, in general, the $u$-prediction and $x_1$-prediction yield comparable results.  However, in certain task, the $u$-prediction becomes highly unstable and fails to optimize,  whereas the $x_1$-prediction demonstrates much better stability.

\paragraph{Remark.} Although our $x_1$-prediction Mean Flow also predicts the endpoint, it differs from prior methods Consistency Models (CM) \cite{song2023consistency} and Flow Map Matching (FMM) \cite{boffi2025flow}. CM and FMM predict the true future state $f_r$ from the current function $f_t$, whereas our method, inspired by $x_1$-prediction Flow Matching, predicts the intersection of the velocity line with $t{=}1$. In addition, CM cannot fully utilize gradient information, and FMM optimizes quantities inside gradient operators, causing instability and high cost. Our $x_1$-prediction Mean Flow is theoretically equivalent to $u$-prediction Mean Flow and avoids these drawbacks.

%\paragraph{Remark.} 我们的$x_1$-prediction MeanFlow 虽然也是预测终点,但是与之前的预测将来位置的方法,比如Consistency Model或者Flow Map Match \cite{boffi2025flow} 并不相同,Consistency Model或者Flow Map Match预测的基于当前的$f_t$, 预测真实的位置$f_r$(r=1 for Consistency Model), 而我们的$x_1$-prediction MeanFlow基于$x_1$-prediction Flow Matching的intuition, 预测的是速度验证线和t=1的交点。  Consistency Model 不能充分的利用梯度信息, Flow Map Match的被优化对象位于梯度算子中,不稳定且开销大不能scaleup, 而我们的$x_1$-prediciont和meanflow 等价,不存在这些问题

%\cite{boffi2025flow} 中基于通过

\subsection{Algorithm}
\begin{algorithm}[t]
\caption{Functional Mean Flow: Training}
\label{alg:training}
\begin{algorithmic}[1]  % [1] 表示每行编号
\Require dataset $\mathcal{D}$, initial model parameter $\theta$, learning rate $\eta$, Gaussian measure sampler $\mathcal{N}(0, C_0)$, time sampler $\mathcal{T}$
\Repeat
    \State Sample $f \sim \mathcal{D}$, $f_0 \sim \mathcal{N}(0, C_0)$ and $t,r\sim \mathcal{T}$
    \State $g \leftarrow (1-(1-\sigma_{\text{min}})t)f_0+tf$ 
    \If{$u$-prediction}
        \State $u^f_{t} \leftarrow \frac{1-\sigma_{\min}}{1-(1-\sigma_{\min})t}(t f-g) + f $
        \State $\mathcal{L}(\theta) \leftarrow \|(r-t)\text{sg}(\frac{\partial}{\partial t}\bar u^\theta_{t\to r}(g)+$\\$\qquad\qquad\qquad\qquad D\bar u^\theta_{t\to r}(g)[u^f_t])+u^f_t-\bar u^\theta_{t\to r}(g)\|_\mathcal{F}^2$
    \ElsIf{$x_1$-prediction}
        \State $\hat{f}_{1,t}^f(g) \leftarrow \frac{\sigma_{\text{min}}}{1-(1-\sigma_{\min})t}(g-tf)+f$
        \State $\mathcal{L}(\theta) \leftarrow \|\frac{r-t}{1-r}\text{sg}\big((1-t)\frac{\partial}{\partial t}\hat f^\theta_{1,t\to r}(g)+$\\$\qquad \qquad D\hat f^\theta_{1,t\to r}(g)[\hat{f}_{1,t}^f(g)-g]\big)+\hat{f}_{1,t}^f(g) $\\$\qquad \qquad-\hat{f}_{1,t\to r}^\theta(g)\|_\mathcal{F}^2$
    \EndIf
    \State $\theta \gets \theta - \eta \nabla_\theta \mathcal{L}(\theta)$
\Until{convergence}
\end{algorithmic}
\end{algorithm}

Similar to Functional Flow Matching, Functional Mean Flow starts from functions sampled from a Gaussian measure, since white noise is undefined in infinite-dimensional spaces \cite{zhang2024functional,bond2024infty}.  The model also requires a function-to-function network, such as a Neural Operator (see \autoref{sec:experiment} for details on sampling from Gaussian measure and network).  Similar to \cite{geng2025mean,song2023consistency}, the gradient terms in \autoref{eq:meanflow_u_loss} and \autoref{eq:meanflow_x1_loss} can be computed through the JVP operation within the optimization framework.  Based on the above, we obtain the training and sampling algorithms for the $u$-prediction and $x_1$-prediction variants of Functional Mean Flow in \textcolor{cvprblue}{Algorithm \autoref{alg:training}} and \textcolor{cvprblue}{Algorithm \autoref{alg:sampling}}.  For clarity, we include Python code examples in the \textcolor{cvprblue}{Appendix~\ref{appendix:implementation}}.

 \begin{figure*}[t]
    \centering
  \includegraphics[width=\textwidth]{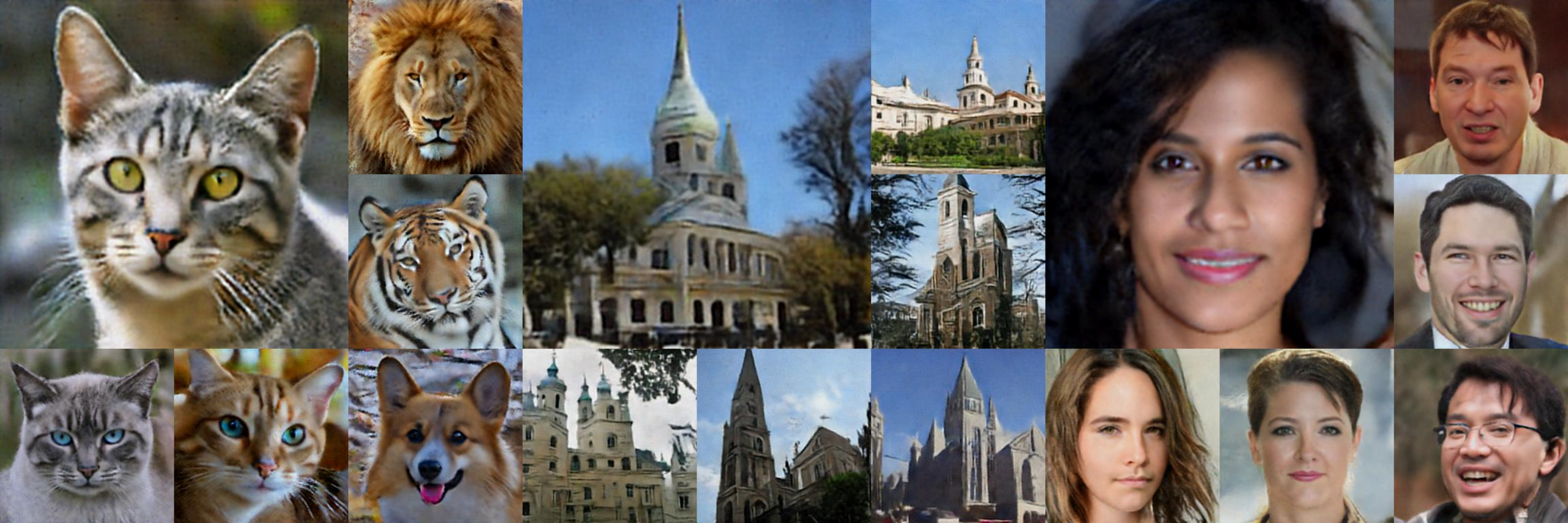}
  \caption{Results on AFHQ, LSUN-Church, and FFHQ. The model is trained on a random 1/4 pixel subset of 256×256 images and evaluated at 256×256 and 512×512 via one-step generation.}
  \label{fig:afhq_church_ffhq}
\end{figure*}

\begin{algorithm}[t]
\caption{Functional Mean Flow: Inference}
\label{alg:sampling}
\begin{algorithmic}[1]  % [1] 表示每行编号
\Require trained model parameter $\theta$, Gaussian measure sampler $\mathcal{N}(0, C_0)$
\State Sample $f_0 \sim \mathcal{N}(0, C_0)$
\If{$u$-prediction}
    \State $f \leftarrow \bar u^\theta_{t\to r}(f_0) + f_0$
\ElsIf{$x_1$-prediction}
    \State $f \leftarrow \hat f^\theta_{1,t\to r}(f_0)$
\EndIf
\end{algorithmic}
\end{algorithm}

\section{Experiment}\label{sec:experiment}
To evaluate the generality and effectiveness of our approach, we conduct experiments on three distinct and representative tasks: real-world functional generation (including time-series data and Navier–Stokes solutions) \cite{kerrigan2024functional,rahman2022gan,lim2025score,kerrigan2023diffusion}, function-based image generation \cite{bond2024infty,le2024brush}, and SDF-based 3D shape generation \cite{zhang2024functional}.
For all tasks, we adopt the neural architecture originally designed for multi-step generation, with only a minor modification that replaces the single time variable $t$ with a pair ($t$, $r$) to meet the requirements of the FMF formulation (see \autoref{appendix:model}); the models are then trained with \textcolor{cvprblue}{Algorithm \autoref{alg:training}}. The experimental results demonstrate that our framework can be seamlessly integrated into various functional generation paradigms, enabling effortless adaptation of existing neural architectures for one-step generation. These include Neural Operators \cite{kerrigan2024functional,rahman2022gan,lim2025score,kerrigan2023diffusion}, hybrid sparse–dense Neural Operators \cite{bond2024infty,le2024brush}, and point-based functional generation models \cite{zhang2024functional}.

\subsection{Real-World Functional Generation}
We now investigate the empirical performance of our FMF model on several real-world functional datasets. For fair comparison, we follow the same experimental setup as prior works \cite{kerrigan2024functional} and adopt the Fourier Neural Operator (FNO) as the backbone to model $\bar u_{t\to r}^\theta(g)$ for $u$-prediction and $\hat f_{1,t\to r}^\theta(g)$ for $x_1$-prediction, which takes functions as both inputs and outputs. The network size and structural parameters are kept identical to previous implementations, and for initial Gaussian measure, a Gaussian processes with a Mat\'ern kernel is used for parametrization (see \textcolor{cvprblue}{Appendix~\ref{appendix:model_scientific}} for details).  

Our functional datasets consist of two categories: (1) Five 1D statistical datasets with diverse correlation structures, including a daily temperature dataset (AEMET) \cite{febrero2012statistical}, a gene expression time-series dataset (Genes) \cite{orlando2008global}, an economic population time-series dataset (Pop.) \cite{bolt2025maddison}, a GDP-per-capita dataset (GDP) \cite{inklaar2018rebasing}, and a labor-force-size dataset (Labor) \cite{IMF_IFS_2022}; and (2) a 2D fluid dynamics dataset consisting of numerical solutions to the Navier–Stokes equations on a 2D torus \cite{li2022learning}.  We compare our method against several function-based generative models, including the multi-step approaches FDDPM \cite{kerrigan2023diffusion}, DDO \cite{lim2025score}, and FFM \cite{kerrigan2024functional} in both its OT and VP variants, as well as the one-step functional generation method GANO \cite{rahman2022gan}. The quantitative and qualitative results are summarized in \autoref{tab:func_bmf_1d} and \autoref{tab:ns_data}.  For the 1D datasets, following \cite{kerrigan2024functional}, we compute a set of statistical functionals: mean, variance, skewness, kurtosis, and autocorrelation for the generated functions, and evaluate the MSE between them and the corresponding ground-truth statistics from the dataset. For the 2D dataset, we evaluate the MSE between the generated and ground-truth Navier–Stokes solutions in terms of both density and spectral representations \cite{kerrigan2024functional,rahman2022gan,lim2025score}. 

Across both 1D and 2D settings, our method achieves the best performance among one-step functional generation methods, while performing comparably to the best multi-step baselines such as FFM.  Detailed descriptions of the training procedures, inference configurations, and evaluation metrics can be found in \textcolor{cvprblue}{Appendix~\ref{appendix:model_scientific}}.

\begin{table}[t] % ✅ 改成 table,而不是 table*
\centering
\small % 或 \scriptsize 让字体略小一点
\setlength{\tabcolsep}{2pt} % 调小列间距
\renewcommand{\arraystretch}{0.9} % 紧凑行距
\caption{Comparison of different functional generative method on 1D datasets. Statistical metrics (mean, variance, skewness, kurtosis, and autocorrelation) are reported across datasets. The best results for the 1-step and multi-step settings are highlighted in bold.}
\label{tab:func_bmf_1d}
\resizebox{\columnwidth}{!}{ % ✅ 改成 \columnwidth
\begin{tabular}{l|c c c c c c}
\hline
Dataset & Mean $\downarrow$& Variance $\downarrow$& Skewness $\downarrow$& Kurtosis $\downarrow$& Autocorrelation $\downarrow$& NFEs \\
\hline
\textbf{AEMET} & & & & & & \\
%\textbf{FMF ($u$-pred)}    & \textbf{3.2e-1} (1.0e-1) & \textbf{7.8e-1} (4.9e-1) & 5.9e-2 (3.4e-2) & 8.7e-2 (5.8e-2) & \textbf{3.8e-4} (1.9e-5) & 1 \\
%\textbf{FMF ($x_1$-pred)}    & 3.8e-1 (6.5e-2) & 1.5e+0 (1.3e+0) & \textbf{4.4e-2} (4.0e-2) & \textbf{8.6e-2} (6.1e-2) & 5.2e-4 (5.2e-6) & 1 \\
\textbf{FMF ($u$-pred)}    & \textbf{5.3e-1} (1.5e-1) & 2.0e+0 (1.3e+0) & 7.4e-2 (3.2e-2) & \textbf{1.4e-1} (5.7e-2) & \textbf{5.2e-4} (9.1e-6) & 1 \\
\textbf{FMF ($x_1$-pred)}    & 5.4e-1 (1.7e-1) & \textbf{1.8e+0} (9.8e-1) & \textbf{6.8e-2} (6.0e-2) & 1.8e-1 (1.1e-1) & 5.6e-4 (1.0e-5) & 1 \\
GANO            & 6.5e+1 (1.9e+2) & 7.1e+1 (4.0e+1) & 4.7e-1 (4.8e+0) & 3.2e-1 (1.0e+0) & 2.0e-3 (2.6e-3) & 1 \\
FFM-OT    & \textbf{8.4e-2} (9.9e-2) & 1.7e+0 (1.1e+0) & 7.7e-2 (6.6e-2) & 3.3e-2 (3.7e-2) & \textbf{3.0e-6} (4.0e-6) & 668 \\
FFM-VP    & 1.3e-1 (1.4e-1) & \textbf{1.5e+0} (1.2e+0) & \textbf{5.2e-2} (4.3e-2) & \textbf{1.7e-2} (1.6e-2) & 6.0e-6 (7.0e-6) & 488 \\
FDDPM            & 2.6e-1 (3.0e-1) & 3.5e+0 (1.0e+0) & 1.1e-1 (4.2e-2) & 3.9e-2 (3.0e-2) & 5.0e-6 (5.0e-6) & 1000 \\
DDO             & 2.4e-1 (2.6e-1) & 6.6e+0 (5.1e+0) & 2.1e-1 (4.1e-2) & 3.8e-2 (3.1e-2) & 6.7e-6 (1.3e-4) & 2000 \\
\hline
\textbf{Genes} & & & & & & \\
\textbf{FMF ($u$-pred)}    & \textbf{1.6e-3} (8.3e-4) & \textbf{3.3e-4} (1.5e-4) & \textbf{3.6e-2} (9.6e-3) & \textbf{9.5e-2} (2.3e-2) & 3.8e-3 (8.4e-4) & 1 \\
\textbf{FMF ($x_1$-pred)}   & 2.1e-3 (5.8e-4) & 2.0e-3 (3.1e-4) & 4.6e-2 (1.1e-2) & 2.1e-1 (3.6e-2) & 5.9e-3 (9.7e-4) & 1 \\
GANO            & 4.6e-2 (3.0e-3) & 7.3e-3 (3.6e-4) & 1.7e+0 (1.3e+0) & 3.1e-1 (8.4e-2) & \textbf{2.0e-3} (1.2e-3) & 1 \\
FFM-OT    & 6.7e-4 (5.4e-4) & 3.9e-3 (2.6e-4) & 2.4e-1 (4.7e-2) & 7.7e-2 (9.0e-3) & 2.5e-4 (1.7e-4) & 386 \\
FFM-VP    & \textbf{4.2e-4} (4.8e-4) & \textbf{7.3e-4} (3.5e-4) & \textbf{1.9e-1} (6.1e-2) & \textbf{4.3e-2} (1.2e-2) & \textbf{1.3e-4} (1.0e-4) & 290 \\
FDDPM            & 4.4e-4 (4.4e-4) & 1.3e-3 (4.6e-4) & 2.5e-1 (1.9e-1) & 5.9e-2 (1.2e-2) & 1.9e-4 (1.2e-4) & 1000 \\
DDO             & 4.2e-3 (1.5e-3) & 1.2e-2 (3.6e-4) & 3.0e-1 (5.7e-2) & 1.3e-1 (1.8e-2) & 1.0e-3 (2.3e-4) & 2000 \\
\hline
\textbf{Pop.} & & & & & & \\
\textbf{FMF ($u$-pred)}    & 7.1e-4 (2.1e-4) & \textbf{1.4e-3} (2.3e-4) & \textbf{2.0e-1} (8.8e-2) & \textbf{6.4e+0} (7.1e+0) & 7.2e-3 (9.0e-4) & 1 \\
\textbf{FMF ($x_1$-pred)}   & \textbf{1.7e-4} (1.2e-4) & 1.6e-3 (1.9e-4) & 3.7e-1 (1.1e-1) & 1.5e+1 (1.7e+1) & \textbf{1.1e-4} (3.9e-5) & 1 \\
GANO            & 4.7e-3 (2.4e-3) & 1.6e-3 (1.5e-3) & 1.0e+0 (9.2e-1) & 2.3e+1 (3.7e+1) & 1.6e-1 (2.8e-1) & 1 \\
FFM-OT    & 6.0e-4 (7.5e-4) & 1.6e-4 (1.6e-4) & 1.1e-1 (6.7e-2) & \textbf{1.8e+0} (1.2e+0) & 7.0e-4 (3.4e-4) & 662 \\
FFM-VP    & \textbf{5.4e-4} (7.6e-4) & 3.0e-4 (2.9e-4) & 1.7e-1 (4.4e-2) & 2.1e+0 (9.2e-1) & 8.9e-2 (9.1e-3) & 494 \\
FDDPM     & 6.6e-4 (6.1e-4) & \textbf{1.2e-4} (1.2e-4) &\textbf{ 9.4e-2} (6.5e-2) & 2.5e+0 (2.2e+0) & \textbf{3.0e-5} (9.2e-6) & 1000 \\
DDO       & 2.3e-3 (1.3e-3) & 2.2e-1 (8.3e-3) & 4.3e-1 (1.5e-2) & 5.2e+0 (1.5e-1) & 5.0e-1 (1.0e-2) & 2000 \\
\hline
\textbf{GDP} & & & & & & \\
\textbf{FMF ($u$-pred)}    & 1.2e-3 (6.8e-4) & \textbf{2.9e-3} (5.1e-4) & 2.9e-1 (8.4e-2) & 2.4e+0 (9.7e-1) & 1.0e-3 (2.5e-4) & 1 \\
\textbf{FMF ($x_1$-pred)}   & \textbf{1.1e-3} (7.8e-4) & 4.0e-3 (6.7e-4) & \textbf{2.2e-1} (8.9e-2) & \textbf{1.8e+0} (5.8e-1) & \textbf{2.9e-4} (1.9e-5) & 1 \\
GANO            & 9.6e+2 (3.1e+3) & 7.4e+2 (2.3e+3) & 5.8e-1 (2.2e-1) & 2.4e+0 (1.0e+0) & 7.1e-2 (1.9e-1) & 1 \\
FFM-OT          & 2.8e-2 (2.8e-3) & 5.3e-3 (1.2e-3) & 6.6e-1 (2.9e-1) & 9.2e+0 (1.6e+1) & 6.1e-4 (4.6e-4) & 536 \\
FFM-VP          & 2.8e-2 (3.4e-3) & 4.9e-3 (1.2e-3) & 5.3e-1 (1.2e-1) & 3.2e+0 (1.4e+0) & 8.7e-2 (1.0e-2) & 494 \\
FDDPM           & \textbf{6.0e-4} (6.5e-4) & \textbf{5.3e-4} (5.3e-4) & \textbf{5.1e-2} (2.6e-2) & \textbf{7.2e-1} (4.0e-1) & \textbf{1.8e-4} (4.3e-5) & 1000 \\
DDO             & 1.3e-2 (2.6e-3) & 1.5e-1 (9.9e-3) & 3.6e-1 (1.6e-2) & 1.9e+0 (1.0e-1) & 3.8e-1 (8.5e-3) & 2000 \\
\hline
\textbf{Labor} & & & & & & \\
\textbf{FMF ($u$-pred)}    & 5.3e-6 (2.5e-6) & \textbf{7.1e-8} (1.2e-8) & 3.3e-1 (7.9e-2) & 1.3e+1 (5.6e+0) & \textbf{1.1e-2} (2.9e-3) & 1 \\
\textbf{FMF ($x_1$-pred)}   & \textbf{5.1e-6} (2.4e-6) & 1.2e-7 (2.0e-8) & \textbf{2.7e-1} (5.9e-2) & 7.9e+0 (4.3e+0) & 2.1e-2 (4.3e-3) & 1 \\
GANO            & 4.7e-5 (3.8e-5) & 2.4e-7 (1.6e-7) & 6.6e-1 (2.2e-1) & \textbf{5.7e+0} (3.0e+0) & 3.3e-2 (1.1e-2) & 1 \\
FFM-OT    & 1.0e-2 (1.2e-4) & 4.2e-7 (1.7e-7) & 1.1e+0 (4.7e-1) & 2.5e+1 (5.1e+1) & 5.5e-2 (4.5e-3) & 308 \\
FFM-VP    & 9.6e-3 (6.1e-5) & 3.5e-7 (7.8e-8) & 1.1e+0 (1.2e-1) & \textbf{7.0e+0} (9.5e-1) & 2.6e-2 (4.0e-3) & 320 \\
FDDPM           & \textbf{6.4e-6} (4.1e-6) & \textbf{6.1e-8} (6.2e-8) & \textbf{2.5e-1} (1.6e-1) & 7.5e+0 (7.1e+0) & \textbf{1.1e-2} (5.2e-3) & 1000 \\
DDO             & 1.3e-5 (5.6e-6) & 3.6e-6 (4.4e-7) & 7.3e-1 (5.9e-2) & 7.2e+0 (1.4e-1) & 3.7e-1 (1.6e-2) & 2000 \\
\hline
\end{tabular}}
\end{table}

\begin{table}[t]
\centering
\caption{MSEs between the density and spectra of the real and generated samples on the Navier--Stokes dataset.  The best results for the 1-step and multi-step settings are highlighted in bold.}\label{tab:ns_data}
\begin{tabular}{lcc}
\toprule
 & Density $\downarrow$& Spectrum $\downarrow$ \\
\midrule 
\textbf{FMF ($u$-pred)}& 9.7e-5 & 1.2e3 \\
\textbf{FMF ($x_1$-pred)}& \textbf{8.0e-5} & \textbf{5.6e2} \\
GANO & 2.5e-3 & 3.2e4 \\
FFM-OT & \textbf{3.7e-5} & \textbf{9.3e1} \\
DDPM & 9.9e-5 & 5.0e2 \\
DDO & 2.9e-2 & 1.6e5 \\
\bottomrule
\end{tabular}
\end{table}

\subsection{Image Generation Based on Functional}

Infty-Diff~\cite{bond2024infty,le2024brush} observed that purely Neural Operator–based functional generation methods struggle to scale to large datasets.  To overcome this limitation, Infty-Diff introduced a hybrid sparse–dense Neural Operator that efficiently learns from higher-resolution functional data (e.g., 256×256 images).  The model first employs a sparse Neural Operator to flexibly represent functions sampled at random points, followed by a dense U-Net/UNO backbone that refines features on a dense grid obtained through k-nearest-neighbor (KNN) sampling.  We follow the network design of \cite{bond2024infty} and adopt the same model capacity, with only minimal modifications to convert the original multi-step diffusion formulation into a single-step FMF generation framework, and for the initial Gaussian measure, we employ white noise with a mollifier kernel, consistent with their implementation (see \textcolor{cvprblue}{Appendix~\ref{appendix:model_image}} for details).

\begin{figure}[t]
  \centering
  \includegraphics[width=0.8\linewidth]{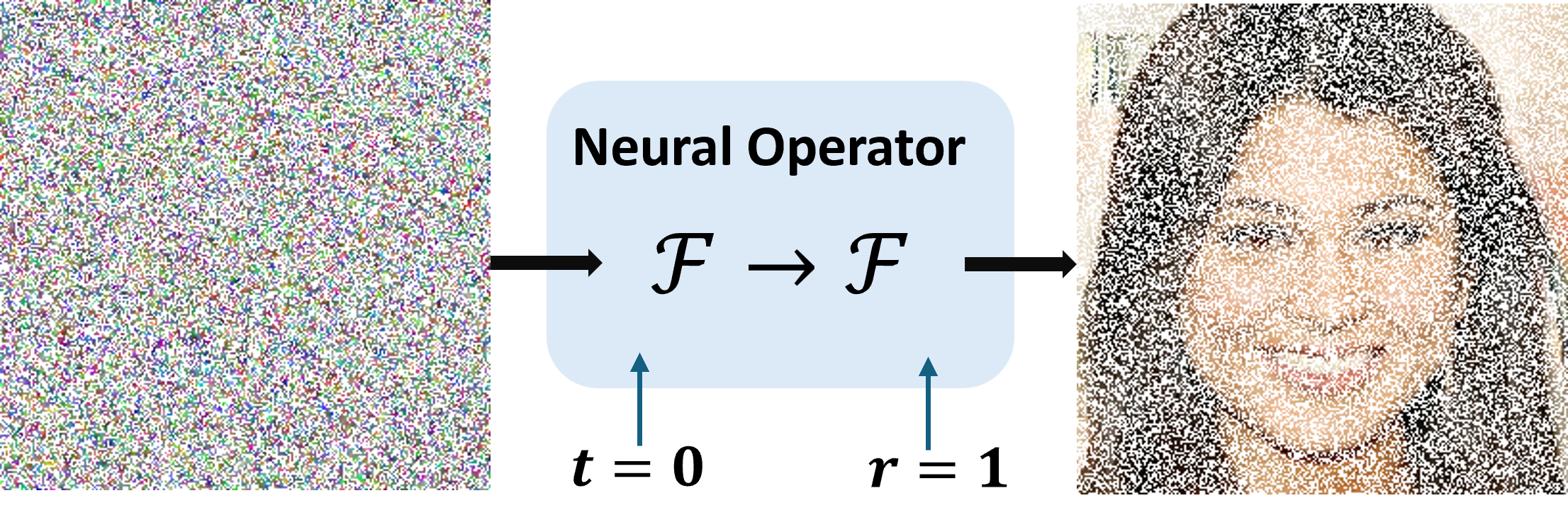}
    \vspace{-0.2cm}
  \caption{In Functional Mean Flow, both the input and output are modeled as continuous functions, enabling training and image generation to be defined over arbitrary pixel coordinates instead of being restricted to a discrete grid.}\label{fig:illustration_image_generation}
  \vspace{-0.6cm}
\end{figure}

Although function-based image generation typically exhibits slightly lower perceptual fidelity compared to conventional pixel-based diffusion models, it provides significantly greater flexibility. As pointed out in Infty-Diff, functional-based models can be trained using only a random subset (e.g., one-quarter of the pixels) from a 256-resolution dataset, and can still generate images at arbitrary resolutions (e.g., 128, 256, or 512).  Because the model operates directly in the functional space, its input and output can be defined on any pixel coordinates rather than being constrained to a fixed grid (see \autoref{fig:illustration_image_generation}).

Following Infty-Diff, we train our model on three unconditional image generation datasets, CelebA-HQ~\cite{karras2017progressive}, FFHQ~\cite{karras2019style}, and LSUN-Church~\cite{yu2015lsun}, as well as one conditional generation dataset, AFHQ~\cite{choi2020stargan}.
Qualitative results are shown in \autoref{fig:celeb_multi_res} and \autoref{fig:afhq_church_ffhq}, with additional results provided in \textcolor{cvprblue}{Appendix~\ref{appendix:additional_image_results}}. During training, the model observes only $25\%$ of the pixels from 256×256 images, while at inference it generate images at 64, 128, 256, 512, and even 1024 resolutions. We report quantitative comparisons with other function-based methods in \autoref{tab:unconditional_comparison}, where following Infty-Diff we primarily evaluate using the FID$_\text{CLIP}$ \cite{kynkaanniemi2023role} metric to assess function-based generative methods.
For completeness, we also report conventional FID scores \cite{heusel2017gans} for reference.  Since our function-based generation framework is inherently resolution-agnostic, we evaluate models trained on 256-resolution datasets at 64, 128, 256, 512, and 1024 resolutions, and report the corresponding FID$_\text{CLIP}$ results in \autoref{tab:across_resolution}.  Our model achieves state-of-the-art performance among one-step function-based methods and produces results comparable to the multi-step function-based generation of Infty-Diff.
Additional details on the training setup, inference procedure, and the internal structure of Infty-Diff are provided in \textcolor{cvprblue}{Appendix~\ref{appendix:model_image}}.

\begin{table}[t] % ✅ 改成 table
\centering
\small % 或 \scriptsize,让表格更紧凑
\setlength{\tabcolsep}{2pt} % 调小列间距
\renewcommand{\arraystretch}{0.9} % 行距稍紧
\caption{Evaluation of FID$_\text{CLIP}$ \cite{kynkaanniemi2023role} against previous infinite-dimensional approaches trained on coordinate subsets. For completeness, since several prior works report Inception FID, we additionally provide the Inception FID of our method, indicated with an asterisk (*).  The best results for the 1-step and multi-step settings are highlighted in bold.}\label{tab:unconditional_comparison}
\resizebox{\columnwidth}{!}{ % ✅ 用 columnwidth 替代 textwidth
\begin{tabular}{lcccccc}
\toprule
\textbf{Method} & \textbf{Step} & \textbf{CelebAHQ-64} & \textbf{CelebAHQ-128} & \textbf{CelebAHQ-256} & \textbf{FFHQ-256} & \textbf{Church-256} \\
\midrule
%\multicolumn{7}{l}{\textbf{Infinite-Dimensional}} \\ % ✅ 修正列数（7列）
%\midrule
D2F \cite{dupont2022data}  & 1 & 40.4$^*$ & --    & -- & --    & -- \\

GEM \cite{du2021learning}       & 1 & 14.65 & 23.73 & -- & 35.62 & 87.57 \\
GASP \cite{dupont2022generative} & 1 & 9.29  & 27.31 & -- & 24.37 & 37.46 \\
%CM-Func \cite{liuone} & 1 & 29.49* & 69.30* & --  & -- & -- \\
FMF (Ours) & 1 & \textbf{3.48 (14.73*)} & \textbf{7.18 (30.35*)} & \textbf{9.17 (33.32*)} & \textbf{11.37 (37.67*)} &   \textbf{26.57(35.63*)} \\
$\infty$-Diff \cite{bond2024infty} & 100 & \textbf{4.57} & \textbf{3.02} & -- & \textbf{3.87} & \textbf{10.36} \\
DPF \cite{zhuang2023diffusion}   & 1000 & 13.21$^*$ & --    & -- & --    & -- \\
%\midrule
%\multicolumn{7}{l}{\textbf{Finite-Dimensional}} \\
%\midrule
%%CIPS \cite{anokhin2021image}  & -- & --   & --   & -- & 5.29 & 10.80 \\
%%StyleSwin \cite{zhang2022styleswin} & -- & --   & 3.39 & -- & 3.25 & 8.28 \\
%%UT \cite{bond2022unleashing}  & 100  & --   & --   & -- & 3.05 & 5.52 \\
%%StyleGAN2 \cite{karras2020analyzing}& -- & -- & 2.20 & -- & 2.35 & 6.21 \\
%%Reflow \cite{liu2022reflow} & 1 & -- & -- & -- & -- & -- \\
%LDM \cite{rombach2022highresolution} & 50 & --  & -- & 2.94 (7.37*) & 3.83 (11.92*) & 8.66 (4.13*)\\
%Patch-DM \cite{ding2024patched} & 50 & -- & -- & -- & 10.02* & 5.49* \\
%Consistency-FM \cite{yang2024consistency}& 6 & -- & -- & 36.4* & -- & -- \\
%Shortcut Models \cite{frans2025shortcut} &1 & -- & -- & 20.5* & -- & -- \\
\bottomrule
\end{tabular}
}
\label{tab:fid_clip}
\end{table}

\begin{table}[t]
\centering
\small % 或 \scriptsize,让表格更紧凑
\renewcommand{\arraystretch}{0.9} % 行距稍紧
\caption{Comparison of results across different resolutions. 
All results are trained on $256 \times 256$ images, 
using only $\tfrac{1}{4}$ of the pixels as input. 
Numbers are FID$_\text{CLIP}$ \cite{kynkaanniemi2023role}. The generation resolution is increased up to the maximum resolution of the dataset.}\label{tab:across_resolution}
\label{tab:dataset_resolution}
\begin{tabular}{lccccc}
\toprule
Dataset & 64 & 128 & 256 & 512 & 1024 \\
\midrule
\multicolumn{6}{l}{\textbf{Unconditional Generation}} \\
\midrule
CelebA-HQ & 3.48 & 5.86 & 9.17 & 9.70 & 10.96 \\
FFHQ      & 4.42 & 7.70 & 11.37 & 12.34 & -- \\
LSUN-Church    & 12.07 & 17.89 & 26.51 & -- & -- \\
\midrule
\multicolumn{6}{l}{\textbf{Conditional Generation}} \\
\midrule
AFHQ      & 3.10 & 6.19 & 9.24 & 11.55 & -- \\
\bottomrule
\end{tabular}

\end{table}

\subsection{3D Shape Generation}

\begin{figure}[t]
    \centering
    % Row 1
    \begin{subfigure}[t]{0.48\linewidth}
        \centering
        \includegraphics[width=\linewidth]{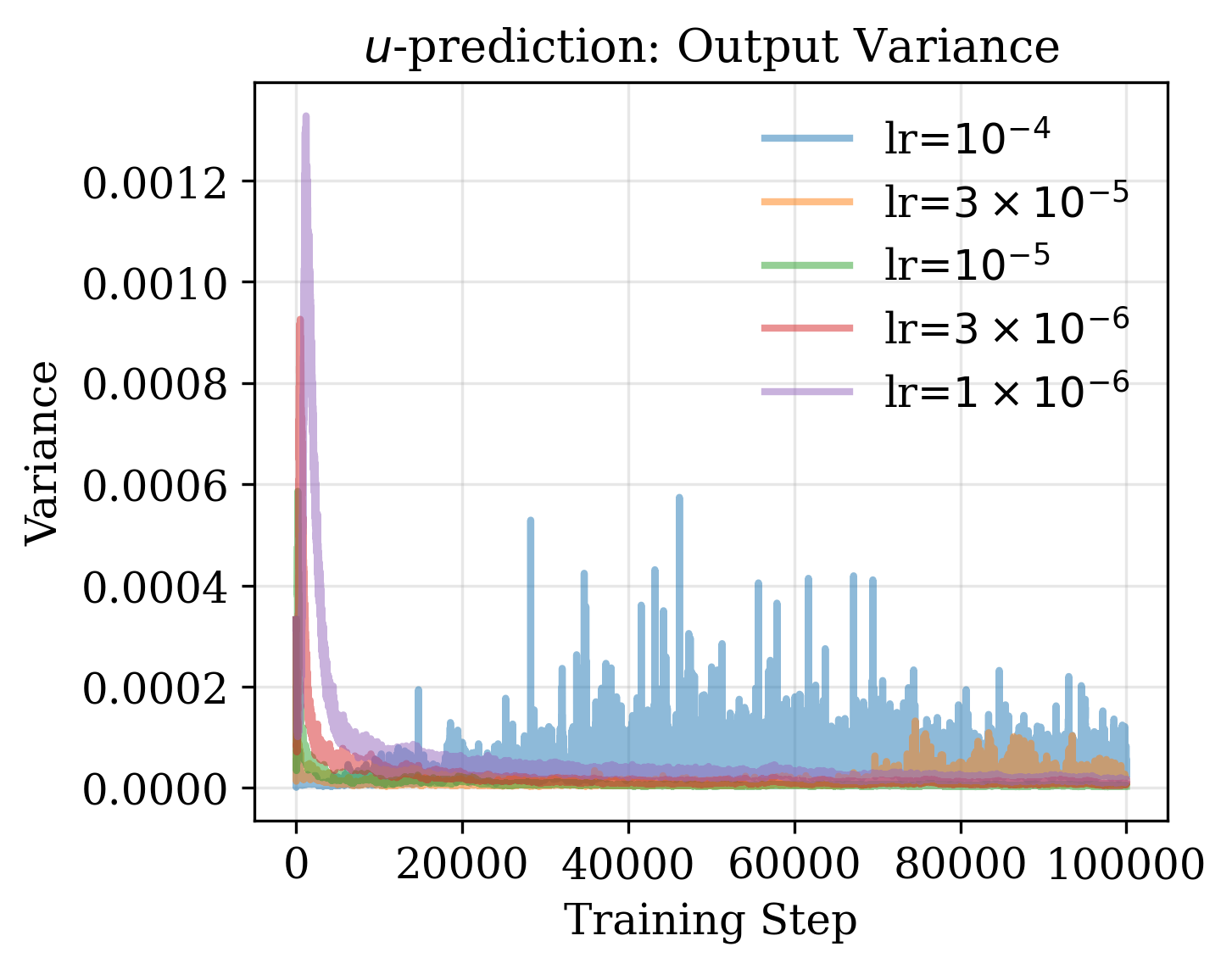}
        \caption{Variance of $u$-prediction}
    \end{subfigure}
    \hfill
    \begin{subfigure}[t]{0.48\linewidth}
        \centering
        \includegraphics[width=\linewidth]{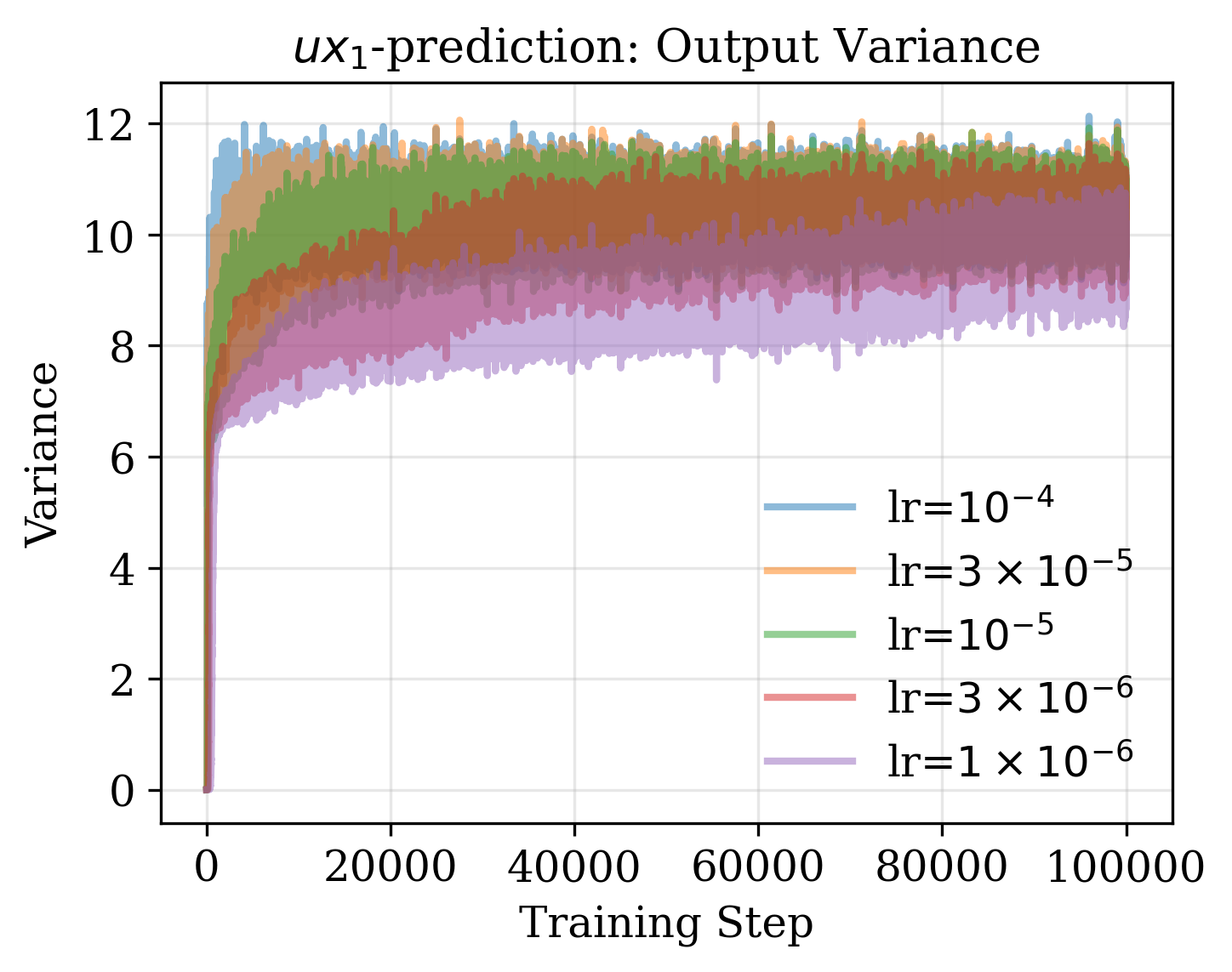}
        \caption{Variance of $x_1$-prediction}
    \end{subfigure}

    \vspace{0.5em}

    % Row 2
    \begin{subfigure}[t]{0.48\linewidth}
        \centering
        \includegraphics[width=\linewidth]{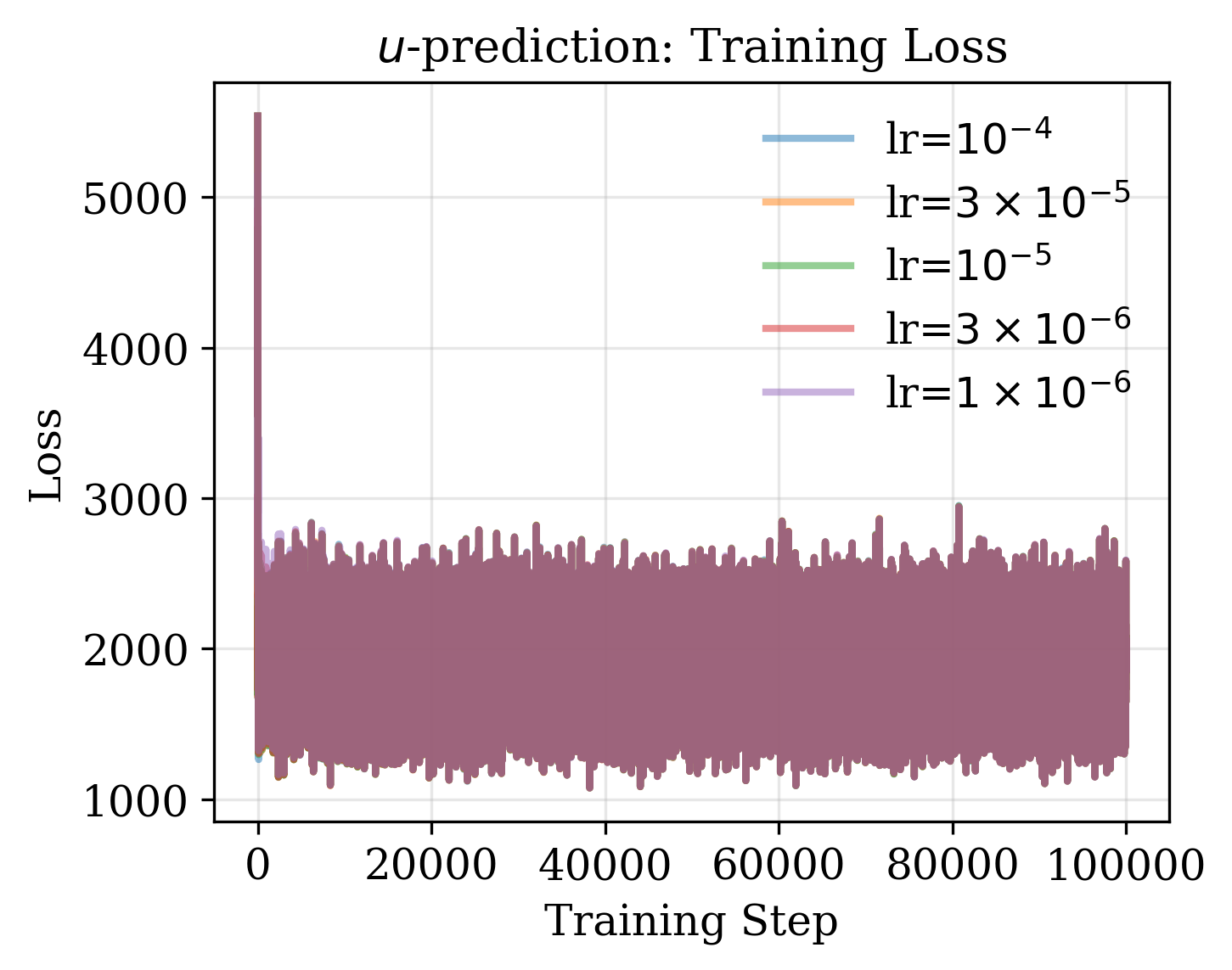}
        \caption{Loss of $u$-prediction}
    \end{subfigure}
    \hfill
    \begin{subfigure}[t]{0.48\linewidth}
        \centering
        \includegraphics[width=\linewidth]{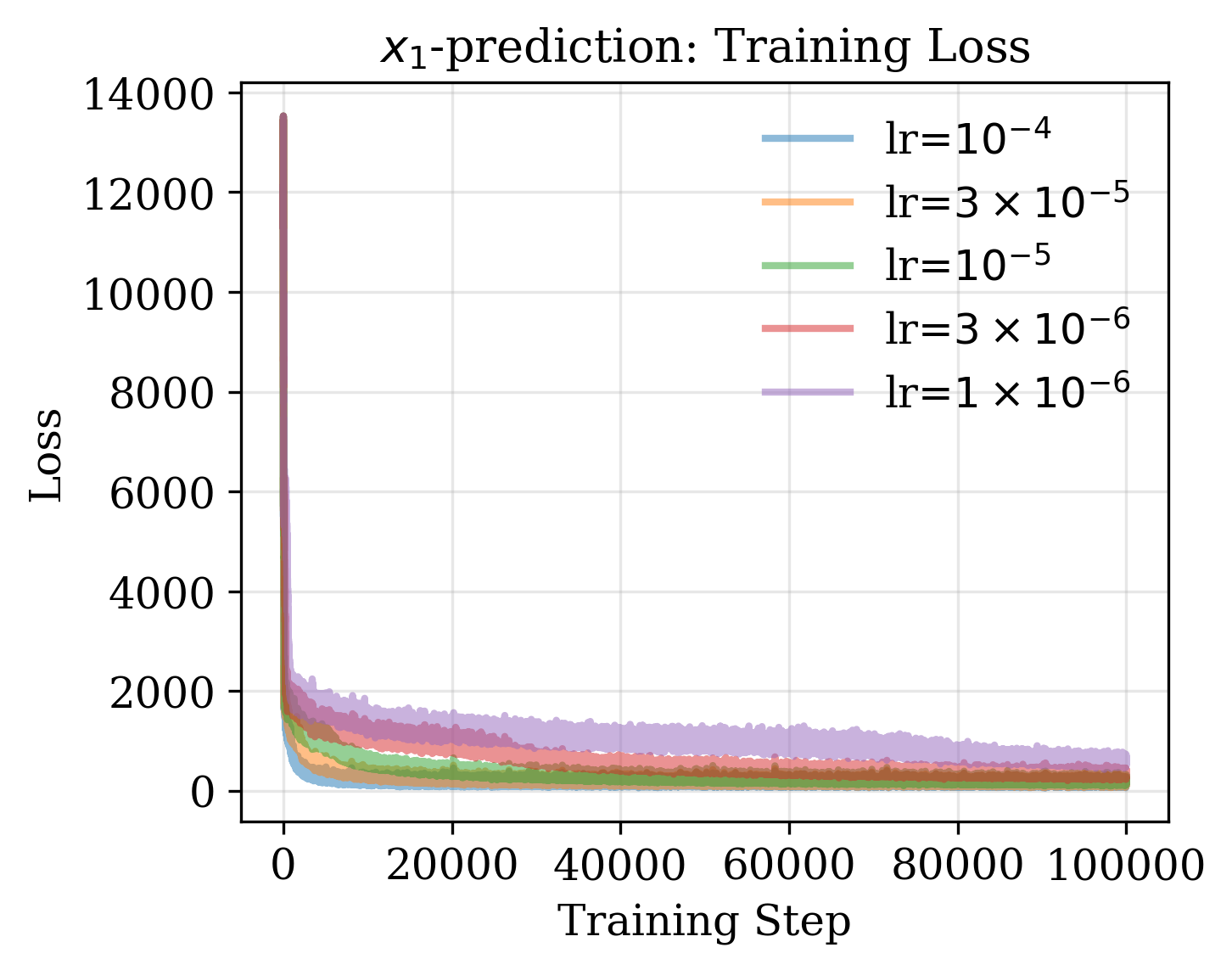}
        \caption{Loss of $x_1$-prediction}
    \end{subfigure}
    \caption{
        Training behavior of $u$- vs.\ $x_1$-prediction FMF under varying learning rates.
        The $u$-prediction model exhibits spatial-variance collapse and unstable losses,
        whereas the $x_1$-prediction model maintains stable variance and smooth optimization.
    }
    \vspace{-0.5cm}
    \label{fig:training_behavior}
\end{figure}

To further validate the applicability of our method to function-based generation tasks, we extend it to SDF-based 3D shape generation, where shape generation is achieved by directly generating its SDF.  We adopt the framework of Functional Diffusion~\cite{zhang2024functional}, where both the input function and output function are represented by randomly sampled points and their corresponding function values: the input function $f_c$ is represented by a set of context points $\{x_c^i\}_{i=1}^n$ with context values $\{v_c^i\}_{i=1}^n$, where $v_c^i = f_c(x_c^i)$, and the output function $f_q$ is represented by query points $\{x_q^j\}_{j=1}^m$ and their predicted query values $\{v_q^j\}_{j=1}^m$, where $v_q^j = f_q(x_q^j)$.  Following the Perceiver~\cite{jaegle2021perceiver} framework, Functional Diffusion performs cascaded cross- and self-attention between the context embedding and the learnable functional vector $\mathcal{X}$. It then applies cross-attention with the query points $\{x_q^j\}_{j=1}^m$ to generate the corresponding query values $\{v_q^j\}_{j=1}^m$, yielding the output function represented as $(\{x_q^j\}_{j=1}^m, \{v_q^j\}_{j=1}^m)$.  We follow this network design and adopt the same model capacity for a fair comparison, and consistent with Functional Diffusion, we construct the initial Gaussian measure using linear interpolation over a coarse grid (see \textcolor{cvprblue}{Appendix~\ref{appendix:model_3D}} for details).

In our experiments, however, we found that this framework is not well-suited for $u$-prediction FMF: training becomes unstable even at small learning rates, with severe collapse. To illustrate this behavior, we perform a 2D experiment on MNIST \cite{lecun1998mnist}, converted into signed distance fields (SDFs) and trained under the Functional Diffusion framework using FMF.
We monitor the batch-averaged spatial variance $\frac{1}{m}\sum_{j=1}^m \bigl(v_q^j - \frac{1}{m}\sum_{j=1}^m v_q^j)^2$ of the network output: an SDF should satisfy $|\nabla f| = 1$ and therefore maintain nontrivial spatial variation. Once the variance vanishes and remains near zero for an extended period (collapse), the model degenerates into a constant field and cannot recover. The experimental results are shown in \autoref{fig:training_behavior}, and see \textcolor{cvprblue}{Appendix~\ref{appendix:additional_instability}} for details and evidence that a collapsed $u$-prediction model cannot generate meaningful outputs.

%Moreover, \autoref{fig:training_behavior} shows that a collapsed $u$-prediction model cannot generate meaningful outputs (see \autoref{appendix:additional_instability} ).

%In contrast, the $x_1$-prediction formulation exhibits strong stability within the same architecture, which motivated our development of the $x_1$-prediction Mean Flow.  

%After converting ShapeNet meshes into SDFs, we randomly sample context points $C_x, C_v$ and query points $Q_x, Q_v$; note that $C$ and $Q$ can be disjoint to ensure flexibility in function representation.

We further train our $x_1$-prediction FMF on the 3D ShapeNet-CoreV2~\cite{chang2015shapenet} dataset, following the same preprocessing method as~\cite{zhang2024functional}, where each mesh is converted to a voxelized SDF and then sampled into point–value pairs. As reported in~\cite{zhang2024functional}, function-based 3D shape generation models can solve the challenging task of reconstructing an entire SDF field from as few as 64 surface points $\{C_l\}_{l=1}^{64}$. As demonstrated in \autoref{tab:reconstruction_metrics} and \autoref{fig:3D_shape_result}, our one-step formulation achieves this task with comparable accuracy, highlighting the robustness and effectiveness of the proposed $x_1$-prediction FMF. Details on the dataset processing, task and metrics description, model architecture, and training and inference procedures are provided in \textcolor{cvprblue}{Appendix~\ref{appendix:model_3D}}.

\begin{table}[h]
\centering
\caption{Quantitative comparison of reconstruction quality. The model is trained on the ShapeNet dataset, where the conditional input consists of 64 points sampled from the target surface. The model is required to reconstruct the surface based on these 64 points. Step denotes the number of inference steps.}
\begin{tabular}{lcccc}
\toprule
Method &Step& Chamfer $\downarrow$ & F-Score $\uparrow$ & Boundary $\downarrow$ \\
\midrule
Ours &1& 0.060 & 0.584 & 0.011  \\
3DS2VS&18 & 0.144 & 0.608 & 0.016 \\
FD &64& 0.101 & 0.707 & 0.012\\
\bottomrule
\end{tabular}
\vspace{1mm}
\raggedleft
\footnotesize{($\downarrow$ lower is better; $\uparrow$ higher is better.)}
\label{tab:reconstruction_metrics}
\end{table}

\section{Related Work}
\label{sec:related_work}
\paragraph{Functional Generation.} Functional generation extends generative modeling to infinite-dimensional settings, drawing theoretical support from stochastic equations on Hilbert space \cite{da2014stochastic}. It enables both training and sampling at arbitrary resolutions, making large-scale generation more computationally feasible. For instance, Infty-Brush \cite{le2024brush} demonstrates controllable image generation at resolutions up to 4096 × 4096 pixels. Recent studies have investigated discrete-time diffusion models on Hilbert space \cite{kerrigan2023diffusion,zhuang2023diffusion, zhang2024functional, lim2025score}, while concurrent works have explored their continuous-time counterparts \cite{franzese2023continuous, hagemann2025multilevel, pidstrigach2024infinite}.  Distinct from functional diffusion models, Functional Flow Matching \cite{kerrigan2024functional} avoids injecting random noise during generation, enabling the production of high-quality samples with fewer NFE (Number of Function Evaluations). Beyond diffusion and flow-based approaches, researchers have also proposed functional GANs \cite{rahman2022gan} and functional energy-based models \cite{lim2023energy}, further enriching the landscape of infinite-dimensional generative modeling.

\paragraph{Few-step Diffusion/Flow Models.} Reducing the sampling steps is vital to improve the efficiency of diffusion/flow models. Distillation techniques play a key role in enabling few-step generation. Several studies have explored distilling diffusion models \cite{salimans2022progressive, meng2023distillation, geng2023onestep, axel2024adver, luo2024diff, yin2024one, zhou2024score} and flow models \cite{liu2023flow}. In parallel, consistency models \cite{song2023consistency} were introduced as independently trainable one-step generators that do not rely on distillation. Subsequent works have focused on enhancing their training stability and sample quality \cite{song2023improved, geng2025consistency, lu2025simplifying}. Inspired by consistency models, recent research has incorporated self-consistency principles into related frameworks, such as enforcing consistency in the velocity field of Flow Matching \cite{yang2024consistency}, Shortcut Model \cite{frans2025shortcut}, and stochastic interpolation across time steps \cite{zhou2025inductive}. While standard consistency models rely on a single time variable, Flow Map Matching \cite{boffi2025flow} learns displacement maps parameterized by two time variables. Mean Flow \cite{geng2025mean} further extend this idea by learning the average velocity over time via the time derivative of the Mean Flow identity, achieving state-of-the-art one-step generation performance on ImageNet.

\begin{figure}[t]
  \centering
  \includegraphics[width=0.5\textwidth]{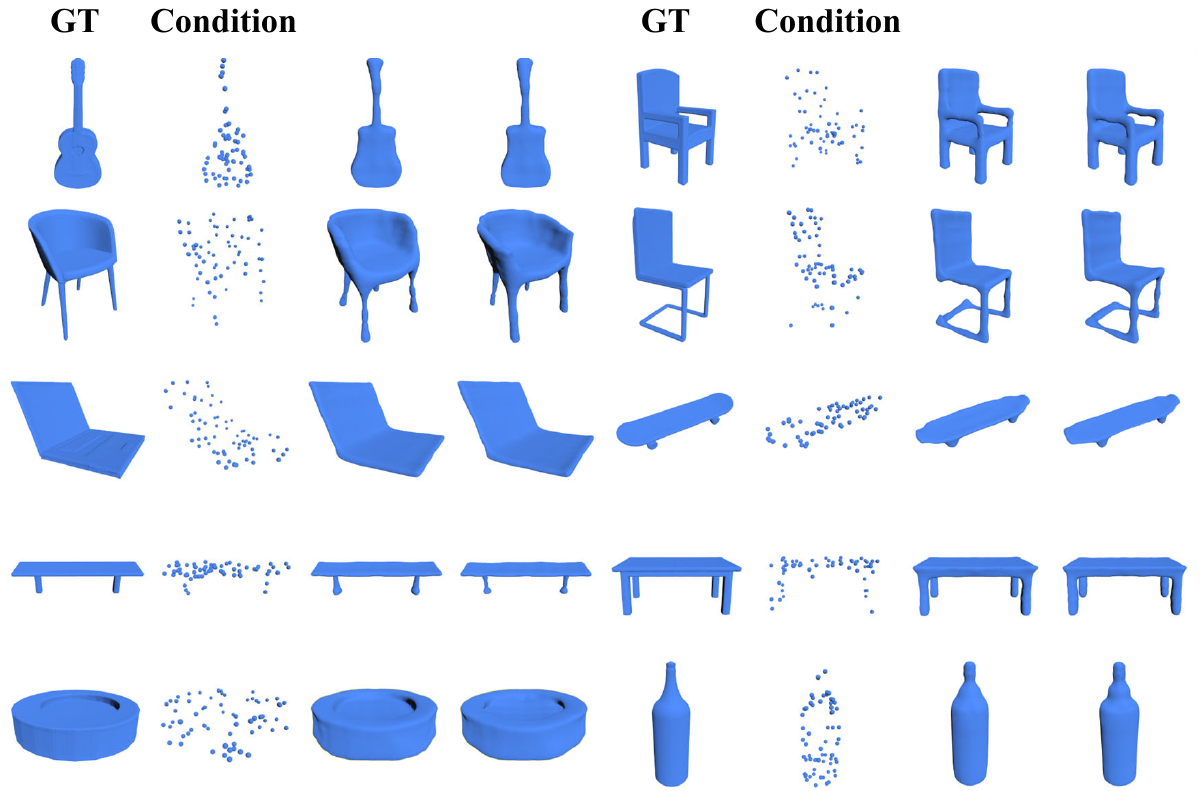}
  \caption{Results of 3D shape generation. This is a highly challenging task~\cite{zhang2024functional}, where the generative model is \textbf{ONLY} conditioned on 64 randomly sampled points from the target surface and required to reconstruct the entire geometry. We apply the $x_1$-prediction FMF within the Functional Diffusion framework, reducing the original 64-step generation process to a single step. The \textbf{GT} column shows the ground-truth surfaces, while the \textbf{Condition} column visualizes the 64 conditioning points provided to the model.}
  \label{fig:3D_shape_result}
\end{figure}

\section{Conclusion}
We proposed Functional Mean Flow as a unified one-step flow matching framework in infinite-dimensional Hilbert space. We introduced an $x_1$-prediction variant of Mean Flow, which exhibits improved training stability and robustness over the original $u$-prediction formulation. Experiments on image-function synthesis, 3D signed distance field modeling, solving PDEs, and time-series prediction demonstrate the versatility and effectiveness of our method. Future work will explore broader functional modalities and further investigate the advantages of the $x_1$-prediction formulation beyond the current domains.

{\small
\bibliographystyle{ieeenat_fullname}
\bibliography{image_generation,functional,one_step_method}
}

\clearpage
\appendix
In this appendix, we provide background theorems of Functional Flow Matching (\autoref{appendix:ffm_theorem}), theoretical derivations of our method (\autoref{appendix:missing_proofs}),implementation details (\autoref{appendix:model}), example Python code (\autoref{appendix:implementation}), and additional experiments and results (\autoref{appendix:additional_results}).

\section{Related Theorem in Functional Flow Matching}\label{appendix:ffm_theorem}
In this section, we elaborate on key theorems from \cite{kerrigan2024functional}, which provide the theoretical foundation for functional flow matching.  Flow matching aims to learn $u_t^\theta$ by minimizing $\mathcal{L}(\theta) = \mathbb{E}_{t,g\sim\mu_t} \big[\|u_t(g) - u_t^\theta(g)\|^2_\mathcal{F}\big]$.   However, since the reference function $u_t(g)$ here does not exist in closed form, functional flow matching construct a conditional velocity field $u_t^f(g)$ to serve as the optimization target for $u_t^\theta(g)$ instead.

The conditional velocity $u_t^f(g)$ induces a flow $\phi_t^f$ by \autoref{eq:path_velocity_definition} that push-forward $\mu_0$  to $\mu_t^f=(\phi_t^f)_\sharp \mu_0$. With $\mu_t^f$, the marginal measure path $\mu_t$ and the marginal velocity field $u_t$ can be obtained by taking the expectation with respect to the data measure $\nu$ by \autoref{eq:ecpectional_conditional}.  However, the connection between the marginal measure path $\mu_t$ and the velocity field $u_t$, derived from the expectations of their conditional counterparts $\mu^f_t$ and $u^f_t$, is non-trivial. Under some smoothness conditions, \cite{kerrigan2024functional} established this relationship in the following theorem.
\begin{theorem}\label{thm:ffm_thm1}~~~~~~~~~~~~~~~~~~~~~~~~~~~~~~
    Assume that $\int_{0}^{1} \int_{\mathcal{F}}\int_{\mathcal{F}} 
    \|u_t^f(g)\|\mathrm{d}\mu_t^f(g)\mathrm{d}\nu(f)\mathrm{d}t < \infty$.  If $\mu_t^f \ll \mu_t$ for $\nu$-a.e.\ $f$ and almost every $t \in [0,1]$,  then the vector field
    \begin{equation}
        \begin{aligned}
u_t(g) = \int_{\mathcal{F}} 
    u_t^f(g)  \frac{\mathrm{d}\mu_t^f}{\mathrm{d}\mu_t}(g)  \mathrm{d}\nu(f)            
        \end{aligned}
    \end{equation}
generates the marginal path of measures $(\mu_t)_{t \in [0,1]}$ which are defined as $\mu_t(A) =\int_{\mathcal{F}}\mu_t^f(A)\mathrm{d}\nu(f)$ via \autoref{eq:ecpectional_conditional}, $\forall A\in \mathcal{B}(\mathcal{F})$.  That is, $(u_t)_{t \in [0,1]}$ and $(\mu_t)_{t \in [0,1]}$ jointly satisfy the continuity \autoref{eq:weak-form}.  
Here, $\frac{\mathrm{d}\mu_t^f}{\mathrm{d}\mu_t}$ denotes the Radon–Nikodym derivative of the conditional measure with respect to the marginal.
\end{theorem}

In the theorem above, the condition $\mu_t^f \ll \mu_t$ denotes that $\mu_t^f$ is absolutely continuous with respect to $\mu_t$, while $\nu$-a.e. indicates that the statement holds for $\nu$-almost every $f$. However, the requirement $\mu_t^f \ll \mu_t$ for $\nu$-a.e. $f$ and almost every $t \in [0,1]$ is generally difficult to guarantee, since the marginal distribution $\mu_t$ itself is hard to know. To address this issue, \cite{kerrigan2024functional} established the following theorem, which reformulates this condition in terms of the known, constructible conditional measures $\mu_t^f$.
\begin{theorem}
Consider a probability measure $\nu$ on $\mathcal{F}$ and a collection of measures $\mu_t^f$ parameterized by $f \in \mathcal{F}$. 
Suppose that the collection of parameterized measures $\mu_t^f$ is $\nu$-a.e. mutually absolutely continuous. 
Define the marginal measure $\mu_t$ via \autoref{eq:ecpectional_conditional}. 
Then, $\mu_t^f \ll \mu_t$ for $\nu$-a.e.\ $f$.
\end{theorem}

Furthermore, \cite{kerrigan2024functional} employed the Feldman–H\'ajek theorem~\cite{da2014stochastic} together with Lemma 6.15 of~\cite{stuart2010inverse} to show that, under the conditional velocity and corresponding conditional measures chosen in \autoref{eq:selection_conditional}, the assumption in the above theorem—that the measures $\mu_t^f$ are $\nu$-a.e.\ mutually absolutely continuous—holds as long as the data distribution is supported on the Cameron–Martin space of $C_0$, i.e., $\nu(C^{1/2}_0(\mathcal{F})) = 1$.  In practice, one typically does not verify whether the data distribution is supported on $C^{1/2}_0(\mathcal{F})$; instead, this condition is often enforced implicitly through data preprocessing~\cite{bond2024infty,lim2023energy}. However, in the experiments reported in~\cite{kerrigan2024functional}, the authors observed that such preprocessing was not strictly necessary.

Using the constructed conditional velocity, the conditional loss can be written as $\mathcal{L}_c(\theta)
= \mathbb{E}_{t,f,g\sim \mu_t^f}\!\left[\|u_t^f(g) - u^\theta_t(g)\|_\mathcal{F}^2\right]$ in \autoref{eq:conditional_loss_ffm}.  Based on the above \autoref{thm:ffm_thm1}, \cite{kerrigan2024functional} proved that optimizing $\mathcal{L}(\theta)$ is equivalent to optimizing $\mathcal{L}_c(\theta)$, where the optimization reference function $u_t^f(g)$ is known, thus allowing directly optimizing $\mathcal{L}_c(\theta)$ to obatin $u^\theta_t(g)$.

\begin{theorem}
Assume that the true and model vector fields are square-integrable, i.e., $\int_{0}^{1} \int_{\mathcal{F}} \|u_t(g)\|^2_\mathcal{F}  \mathrm{d}\mu_t(g)  \mathrm{d}t < \infty
\quad$ and $\int_{0}^{1} \int_{\mathcal{F}} \|u^\theta_t(g\|_\mathcal{F}^2  \mathrm{d}\mu_t(g)  \mathrm{d}t < \infty$.
Then, $\mathcal{L}(\theta) = \mathcal{L}_c(\theta) + C$, where $C \in \mathbb{R}$ is a constant independent of $\theta$.
\end{theorem}

\section{Missing Proofs}\label{appendix:missing_proofs}
In the following analysis, for notational simplicity and without causing ambiguity, we will use the same symbol $\|\cdot\|$ to denote both the norm $\|h\|_{\mathcal{F}}$ of a function $h \in \mathcal{F}$ and the operator norm $\|h\|_{\mathcal{L}(\mathcal{F},\mathcal{F})}$ of a bounded linear map $O:\mathcal{F} \to \mathcal{F}$.  For instance, $Dv_t(g)$ is a bounded linear operator, mapping $h \in \mathcal{F}$ to $Dv_t(g)[h] \in \mathcal{F}$, so its norm is written as $\|Dv_t(g)\|$ for simplicity.

\subsection{Proof of \textcolor{cvprblue}{Statement \ref{thm:statement}}}\label{appendix:statement}
\paragraph{\textcolor{cvprblue}{Statement \ref{thm:statement}}}(Mismatch Between Flow and Marginals of Conditional Flow)
In general, the marginal flow $\phi^{(1)}_{t\to r}(g) =\int_{\mathcal{F}} \phi^f_{t\to r}(g)\frac{\mathrm{d}\mu_t^f}{\mathrm{d}\mu_t}(g)\mathrm{d}\nu(f)$ obtained by taking the expectation over the conditional two-parameter flows $\phi_{t\to r}^f = \phi_r^f \circ (\phi_t^f)^{-1}$ is not equivalent to the two-parameter flow $\phi^{(2)}_{t\to r} = \phi_r \circ (\phi_t)^{-1}$. Here, the superscripts $^{(1)}$ and $^{(2)}$ denote two different ways of computing the marginal two-parameter flow.
\begin{proof}  The theorem essentially reveals a discrepancy between the mean of instantaneous velocity fields and the mean of non-instantaneous flow trajectories. 

To prove that $\phi^{(1)}_{t\to r}(g) \neq \phi^{(2)}_{t\to r}(g)$, it suffices to show that $\frac{\mathrm{d}}{\mathrm{d}r}\phi^{(1)}_{t\to r}(g) \neq \frac{\mathrm{d}}{\mathrm{d}r}\phi^{(2)}_{t\to r}(g)$ at $t=0$. We therefore compute and compare
$\frac{\mathrm{d}}{\mathrm{d}r}\phi^{(1)}_{t\to r}(g)$ and $\frac{\mathrm{d}}{\mathrm{d}r}\phi^{(2)}_{t\to r}(g)$ when $t=0$.

%The theorem essentially states a mismatch between the mean of instantaneous velocities and the mean of non-instantaneous flow paths.  注意到我们想要证明$\frac{\mathrm{d}}{\mathrm{d}r}\int_{\mathcal{F}}\phi^{(1)}_{t\to r}(g)\mathrm{d}\mu_t(g)\neq \frac{\mathrm{d}}{\mathrm{d}r}\int_{\mathcal{F}}\phi^{(2)}_{t\to r}(g)\mathrm{d}\mu_t(g)$, 那就只用证明$t=0$的时候不相等就可以了 Let's compute and compare $\frac{\mathrm{d}}{\mathrm{d}r}\phi^{(1)}_{t\to r}(g)\mathrm{d}\mu_t(g)$ and  $\frac{\mathrm{d}}{\mathrm{d}r}\phi^{(2)}_{t\to r}(g)\mathrm{d}\mu_t(g)$. when t =0 
\begin{equation}
    \begin{aligned}
        &\frac{\mathrm{d}}{\mathrm{d}r}\phi^{(1)}_{t\to r}(g)\\
        &\overset{\textcircled{1}}{=}\frac{\mathrm{d}}{\mathrm{d}r}\int_{\mathcal{F}} \phi^f_{t\to r}(g)\frac{\mathrm{d}\mu_t^f}{\mathrm{d}\mu_t}(g)\mathrm{d}\nu(f)\\
        &=\int_{\mathcal{F}} \frac{\mathrm{d}}{\mathrm{d}r}(\phi^f_{t\to r}(g))\frac{\mathrm{d}\mu_t^f}{\mathrm{d}\mu_t}(g)\mathrm{d}\nu(f)\\
        &\overset{\textcircled{2}}{=}\int_{\mathcal{F}} u_r^f(\phi^f_{t\to r}(g))\frac{\mathrm{d}\mu_t^f}{\mathrm{d}\mu_t}(g)\mathrm{d}\nu(f)\\
        &\overset{\textcircled{3}}{=}\int_{\mathcal{F}} u_r^f(\phi^f_{r}(g))\mathrm{d}\nu(f)\\
        &\overset{\textcircled{4}}{=}\int_{\mathcal{F}}(f-(1-\sigma_{\text{min}})g)\mathrm{d}\nu(f)\\
        &=\int_{\mathcal{F}}f\mathrm{d}\nu(f)-(1-\sigma_{\text{min}})g\\
        &=m_\nu-(1-\sigma_{\text{min}})g,\\
    \end{aligned}
\end{equation}
where $m_\nu$ denotes the expectation of the dataset distribution $\nu$, which is a constant depending only on the dataset.
In the above derivation, $\textcircled{1}$ follows from the definition of $\phi^{(1)}_{t\to r}$,  $\textcircled{2}$ follows from the definition of the two-parameter flow introduced in \autoref{sec:functional_meanflow},  $\textcircled{3}$ substitutes $t=0$ and uses the facts that $\mu_0^f = \mu_0$ and $\phi^f_{0\to r} = \phi^f_{r}$, and $\textcircled{4}$ follows from the specific choice of the conditional flow defined in \autoref{eq:selection_conditional}.

\begin{equation}
    \begin{aligned}
        &\frac{\mathrm{d}}{\mathrm{d}r}\phi^{(2)}_{t\to r}(g)\\
        &\overset{\textcircled{1}}{=} \frac{\mathrm{d}}{\mathrm{d}r}\phi_r\circ \phi_t^{-1}(g)\\
        &\overset{\textcircled{2}}{=} u_r(\phi_r\circ \phi_t^{-1}(g))\\
        &\overset{\textcircled{3}}{=} u_r(\phi_r(g)).\\
    \end{aligned}
\end{equation}
In the above derivation, $\textcircled{1}$ follows from the definition of $\phi^{(2)}_{t\to r}$,  $\textcircled{2}$ follows from the definition of the flow in \autoref{eq:path_velocity_definition}, and $\textcircled{3}$ substitutes $t=0$ and uses the facts that $\phi_{0} = \text{Id}_\mathcal{F}$.

For $\frac{\mathrm{d}}{\mathrm{d}r}\phi^{(1)}_{t\to r}(g)- \frac{\mathrm{d}}{\mathrm{d}r}\phi^{(2)}_{t\to r}(g)$ we have:
\begin{equation}
    \begin{aligned}
        &\frac{\mathrm{d}}{\mathrm{d}r}\phi^{(1)}_{t\to r}(g)- \frac{\mathrm{d}}{\mathrm{d}r}\phi^{(2)}_{t\to r}(g)\\
        &=m_\nu-(1-\sigma_{\text{min}})g-u_r(\phi_r(g)).
    \end{aligned}
\end{equation}

Therefore, for $\phi^{(1)}_{t\to r} = \phi^{(2)}_{t\to r}$ to hold, a necessary condition is $\dot \phi_r(g) = u_r(\phi_r(g)) = m_\nu-(1-\sigma_{\text{min}})g$,$\forall g\in \mathcal{F}$ and $t\in [0,1]$, which can be solved as 
\begin{equation}
    \begin{aligned}
        \phi_r(g) = (m_\nu-(1-\sigma_{\text{min}})g)r+g, \forall g \in \mathcal{F},\ t \in [0,1].
    \end{aligned}
\end{equation}
This implies that $\nu = \mu_1 = (\phi_1)_\sharp \mu_0 = \mathcal{N}(m_\nu,\sigma_\text{min}^2C_0)$, which contradicts the arbitrariness of the dataset distribution $\nu$. Therefore, the equality $\phi^{(1)}_{t\to r} = \phi^{(2)}_{t\to r}$ cannot be satisfied.
\end{proof}

\subsection{Supporting Lemmas for \autoref{thm:flow-diff}}\label{sec:flow-diff_supporting}
As the computation of the Fréchet derivative of $\phi_{t\to r}$ is required in Theorem \ref{thm:flow-diff}, we begin by establishing the following lemma, which asserts the Fréchet differentiability of $\phi_{t\to r}$.

\begin{lemma}[Fr\'echet differentiability of $\phi_{t\to r}$ in Hilbert space]\label{thm:differentiability_of_flow}
For every radius $R>0$ and $\mathcal{B}_R=\{g\in \mathcal{F}|\|g\|< R\}$, assume $\{u_t\}_{t\in[0,1]}$ satisfies:
\begin{enumerate}
    \item[\textnormal{(A1)}] (\textit{Continuity}) $u(t,x)$ is measurable and integrable in $t$ and Lipschitz continuous in $x\in \mathcal{B}_R$, which means there exists integrable $L_R \in L^1(0,1)$, i.e., $\int_0^1 |L_R(t)|\mathrm{d}t<\infty$ such that
    \begin{equation}
        \begin{aligned}
        \|u_t(x) - u_t(y)\| \le L_R(t)\|x - y\|, \forall x,y \in \mathcal{B}_R;            
        \end{aligned}
    \end{equation}
    \item[\textnormal{(A2)}] (\textit{Bounded Fr\'echet differentiability}) For each $t$, $u_t \in C^1(\mathcal{F}; \mathcal{F})$ is continuously Fr\'echet differentiable, and there exists integrable $M_R \in L^1(0,1)$, i.e., $\int_0^1 |M_R(t)|\mathrm{d}t<\infty$ such that
        \begin{equation}
        \begin{aligned}
        \|Du_t(x)\| \le M_R(t), \forall x\in \mathcal{B}_R.
        \end{aligned}
    \end{equation}
 \end{enumerate}
Then the associated two-parameter flow $\phi_{t\to r}(g)$ is continous for $t$ and $r$ and is Fréchet differentiable for all $g \in \mathcal{B}_R$ satisfying $\phi_{t\to \tau}(g) \in \mathcal{B}_R$ for all $\tau \in [t,r]$.
Let $J_r(g) = D\phi_{t\to r}(g)$. Then $J_r(g)$ satisfies the equation 
\begin{equation}
    \begin{aligned}
 \frac{\partial}{\partial r}J_r(g) = Du_r(\phi_{t\to r}(g))\circ J_r(g),\quad J_t(g) = I,       
    \end{aligned}
\end{equation}
and $D\phi_{t\to r}(g)$ is continuous with respect to $g$.
\end{lemma}

\begin{proof}
Under (A1), the existence and uniqueness theorem for ODEs on Banach spaces ensures that for each $g \in \mathcal{B}_R$, \autoref{eq:path_velocity_definition} admits a unique solution $\phi_t$, and $\phi_t(g)$ is continuous with respect to both $g$ and $t$.  Consequently, the two-parameter flow $\phi_{t\to r}(g) = \phi_r(\phi_t^{-1}(g))$ is continuous with respect to $t$, $r$, and $g$.

Define $ \mathcal{\tilde B}_R^{t,r} = \{ g \in \mathcal{B}_R \mid \phi_{t\to\tau}(g) \in \mathcal{B}_R,\forall \tau \in[t,r] \}$. For any $x, y \in  \mathcal{\tilde B}_R^{t,r}$, we have:
\begin{equation}
    \begin{aligned}
        &\frac{\mathrm{d}}{\mathrm{d}r}\|\phi_{t\to r}(x)-\phi_{t\to r}(y)\|^2 \\
        &= 2\langle \phi_{t\to r}(x)-\phi_{t\to r}(y), \frac{\mathrm{d}}{\mathrm{d}r}\phi_{t\to r}(x)-\frac{\mathrm{d}}{\mathrm{d}r}\phi_{t\to r}(y)\rangle\\
        &\le 2\|\phi_{t\to r}(x)-\phi_{t\to r}(y)\|\|\frac{\mathrm{d}}{\mathrm{d}r}\phi_{t\to r}(x)-\frac{\mathrm{d}}{\mathrm{d}r}\phi_{t\to r}(y)\|\\
        &\le 2\|\phi_{t\to r}(x)-\phi_{t\to r}(y)\|\|u_r(\phi_{t\to r}(x))-u_r(\phi_{t\to r}(y))\|\\
        &\le 2L_R(r)\|\phi_{t\to r}(x)-\phi_{t\to r}(y)\|^2.
    \end{aligned}
\end{equation}
The above inequality can be solved as $\|\phi_{t\to r}(x)-\phi_{t\to r}(y)\|\le e^{\int_t^r L_R(\tau) d\tau } \|x-y\|$ which implies that $\phi_{t\to r}$ is Lipschitz continuous on $ \mathcal{\tilde B}_R^{t,r}$.

Now we proceed to prove that $\phi_{t \to r}$ is Fr\'echet differentiable in $ \mathcal{\tilde B}_R^{t,r}$.  For any $g \in  \mathcal{\tilde B}_R^{t,r}$ and $h\in \mathcal{F}$, since $\phi_{t\to r}$ is Lipschitz continuous on $ \mathcal{\tilde B}_R^{t,r}$, it follows that when $\epsilon$ is sufficiently small, we can have $g + \epsilon h \in  \mathcal{\tilde B}_R^{t,r}$.  Consider the difference quotient $\eta_\tau^\epsilon = \frac{\phi_{t \to \tau}(g + \epsilon h) - \phi_{t \to \tau}(g)}{\epsilon}$.  Denote the difference as $\Delta_\tau^\epsilon=\phi_{t \to \tau}(g + \epsilon h) - \phi_{t \to \tau}(g)$.  Since the flow $\phi_{t \to r}$ can be expressed as the integral of its velocity field $\phi_{t \to r}(g) = g + \int_t^r u_{\tau}\big(\phi_{t \to \tau}(g)\big) \mathrm{d}\tau$, we thus have:
\begin{equation}
    \begin{aligned}
        \eta_r^\epsilon &= \frac{\phi_{t \to r}(g + \epsilon h) - \phi_{t \to r}(g)}{\epsilon}\\
        &=h+\frac{1}{\epsilon} \int_t^r u_{\tau}\big(\phi_{t \to \tau}(g+\epsilon h)\big)- u_{\tau}\big(\phi_{t \to \tau}(g)\big)\mathrm{d}\tau\\
        &=h+ \frac{1}{\epsilon} \int_t^r u_{\tau}\big(\phi_{t \to \tau}(g)+\Delta_\tau^\epsilon\big)- u_{\tau}\big(\phi_{t \to \tau}(g)\big)\mathrm{d}\tau.\\
    \end{aligned}
\end{equation}
Since each $u_\tau$ is continuously Fréchet differentiable, we can integrate the derivative $Du_\tau$ along the line segment connecting $x \in \mathcal{F}$ and $x + v \in \mathcal{F}$ as
\begin{equation}
    \begin{aligned}
        u_\tau(x+v) - u_\tau(x) = \int_0^1 Du_\tau(x+\theta v)[v]\mathrm{d}\theta.
    \end{aligned}
\end{equation}
By taking $x = \phi_{t \to \tau}(g)$ and $v = \Delta_\tau^\epsilon$ in the above expression, we obtain:
\begin{equation}\label{eq:eta_equation}
    \begin{aligned}
        \eta_r^\epsilon&=h+ \frac{1}{\epsilon} \int_t^r u_{\tau}\big(\phi_{t \to \tau}(g)+\Delta_\tau^\epsilon\big)- v_{\tau}\big(\phi_{t \to \tau}(g)\big)\mathrm{d}\tau\\
        &=h+ \frac{1}{\epsilon} \int_t^r\int_0^1 Du_\tau(\phi_{t \to \tau}(g)+\theta \Delta_\tau^\epsilon)[\Delta_\tau^\epsilon] \mathrm{d}\theta \mathrm{d}\tau\\
                &=h+ \int_t^r\int_0^1 Du_\tau(\phi_{t \to \tau}(g)+\theta \Delta_\tau^\epsilon)[\frac{\Delta_\tau^\epsilon}{\epsilon}] \mathrm{d}\theta \mathrm{d}\tau\\
                &=h+ \int_t^r K_\tau^\epsilon[\eta_\tau^\epsilon] \mathrm{d}\tau,
    \end{aligned}
\end{equation}
where we write $ K_\tau^\epsilon= \int_0^1 Du_\tau(\phi_{t \to \tau}(g)+\theta \Delta_\tau^\epsilon) \mathrm{d}\theta$. By assumption (A2), we have $|K_\tau^\epsilon| \le M_R(\tau)$ a.e.  Therefore, by applying the Gr\"onwall's inequality, we obtain:
\begin{equation}
    \begin{aligned}
       \sup_{\tau\in [t,r]}\|\eta_\tau^\epsilon\| \le \|h\| e^{\int_t^rM_R(\xi) \mathrm{d}\xi}. 
    \end{aligned}
\end{equation}
Hence, $\eta^\epsilon_\tau$ is uniformly bounded for all $\tau \in [t, r]$.  Since $\Delta_\tau^\epsilon = \epsilon \eta^\epsilon_\tau$, we have $\Delta_\tau^\epsilon \to 0$ as $\epsilon \to 0$.  Because $Du_\tau$ is continuous, we have $K_\tau^\epsilon = \int_0^1 Du_\tau\big(\phi_{t \to \tau}(g) + \theta \Delta_\tau^\epsilon\big) \mathrm{d}\theta \to K_\tau^0 = \int_0^1 Du_\tau\big(\phi_{t \to \tau}(g)\big) \mathrm{d}\theta$  pointwise in $\tau$ as $\epsilon \to 0$.   Moreover, by assumption (A2), the following inequality holds:
%it follows that
%\begin{equation}
%    \begin{aligned}
%   \|\Delta_\tau^\epsilon\| \le |\epsilon|       \sup_{\tau\in [t,r]}\|\eta_\tau^\epsilon\| \le |\epsilon|\|h\| e^{\int_s^tM(\xi) \mathrm{d}\xi}
%    \end{aligned}
%\end{equation}
%Therefore, as $\epsilon \to 0$, we have $\Delta_\tau^\epsilon \to 0$ uniformly for all $\tau \in [t, r]$.  Next, we utilize the uniform convergence of $\Delta_\tau^\epsilon$ to prove the convergence of $K_\tau^\epsilon$ in the integral sense.  Since $\Delta_\tau^\epsilon \to 0$ uniformly, it follows that as $\epsilon \to 0$, $Du_\tau\big(\phi_{t \to \tau}(g) + \theta \Delta_\tau^\epsilon\big)
%\to
%Du_\tau\big(\phi_{t \to \tau}(g)\big)$.  Because $Du_\tau$ is continuous, we have $K_\tau^\epsilon
%= \int_0^1 Du_\tau\big(\phi_{t \to \tau}(g) + \theta %\Delta_\tau^\epsilon\big) \mathrm{d}\theta
%\to
%K_\tau^0
%= \int_0^1 Du_\tau\big(\phi_{t \to \tau}(g)\big) \mathrm{d}\theta$ pointwise in $\tau$.  Moreover, by assumption (A2), the following inequality holds:
\begin{equation}
    \begin{aligned}
        \|K_\tau^\epsilon\|\le\int_0^1 \|Du_\tau(\phi_{t \to \tau}(g)+\theta \Delta_\tau^\epsilon)\|\mathrm{d}\theta\le M_R(\tau).
    \end{aligned}
\end{equation}
Therefore, by the Dominated Convergence Theorem, we have $\int_t^r\|K_\tau^\epsilon-K_\tau^0\|\mathrm{d}\tau \to 0$.  

Finally, we use the above results to prove the convergence of $\eta_r^\epsilon$.  It suffices to show that the family ${\eta_r^\epsilon}$ forms a Cauchy sequence.  For sufficiently small $\epsilon, \epsilon' > 0$, we compute the difference
\begin{equation}
    \begin{aligned}
        \eta_r^\epsilon-\eta_r^{\epsilon'}&=\int_t^rK_\tau^\epsilon[\eta_{\tau}^\epsilon] \mathrm{d}\tau - \int_t^rK_\tau^{\epsilon'}[\eta_{\tau}^{\epsilon'}] \mathrm{d}\tau\\
        &=\int_t^rK_\tau^\epsilon[\eta_{\tau}^\epsilon-\eta_{\tau}^{\epsilon'}] \mathrm{d}\tau + \int_t^r(K_\tau^{\epsilon}-K_\tau^{\epsilon'})[\eta_{\tau}^{\epsilon'}] \mathrm{d}\tau.
    \end{aligned}
\end{equation}
Applying the triangle inequality yields
\begin{equation}
    \begin{aligned}
        \|\eta_r^\epsilon-\eta_r^{\epsilon'}\|&\le\int_t^r\|K_\tau^\epsilon\|\|\eta_{\tau}^\epsilon-\eta_{\tau}^{\epsilon'}\| \mathrm{d}\tau \\
        & \qquad + \int_t^r\|K_\tau^{\epsilon}-K_\tau^{\epsilon'}\|\|\eta_{\tau}^{\epsilon'}\|. \mathrm{d}\tau. 
    \end{aligned}
\end{equation}
Since $\int_t^r \|K_\tau^\epsilon - K_\tau^0\| \mathrm{d}\tau \to 0$,  define $A_{\epsilon,\epsilon'} = |h|\exp^{\int_t^r M(\xi)\mathrm{d}\xi }
  \int_t^r \|K_\tau^\epsilon - K_\tau^0\|+\|K_\tau^{\epsilon'} - K_\tau^0\|\mathrm{d}\tau$, then $A_{\epsilon,\epsilon'} \to 0$ and $\int_t^r \|K_\tau^\epsilon\|\|\eta_\tau^\epsilon - \eta_\tau^{\epsilon'}\|\mathrm{d}\tau
  \le A_{\epsilon,\epsilon'}$.  Let $q_r = \|\eta_\tau^\epsilon - \eta_\tau^{\epsilon'}\|$, then $q_r \le A_{\epsilon,\epsilon'} + \int_t^r q_\tau M_R(\tau)\mathrm{d}\tau$.  Applying the integral form of Grönwall’s inequality gives
\begin{equation}
    \begin{aligned}
        q_r \le A_{\epsilon,\epsilon'}e^{\int_t^r M_R(\tau)d\tau}
    \end{aligned}
\end{equation}
Since $A_{\epsilon,\epsilon'} \to 0$ and $\int_t^r M_R(\tau)\mathrm{d}\tau$ is bounded, we obtain $q_r = \|\eta_\tau^\epsilon - \eta_\tau^{\epsilon'}\| \to 0$, which shows that $\eta_r^\epsilon$ is a Cauchy sequence and hence convergent.  Therefore, $\phi_{t\to r}(g)$ is Gâteaux differentiable. Moreover, since $K_{\tau}^{\epsilon}$ is independent of the direction of $h$, the quantity $A_{\epsilon,\epsilon'}$ is independent of the direction of $h$ which ensures uniform convergence of the directional difference quotients on the unit ball.  The limit $\eta_\tau^{0}$ satisfies the linear integral equation $\eta_r^0 = h+\int_t^r Du_\tau(\phi_{t\to \tau}(g))[\eta_{\tau}^0]\mathrm{d}\tau$.  This is because $K_\tau^\epsilon[\eta_\tau^\epsilon]\to K_\tau^0[\eta_\tau^0] = Du_\tau(\phi_{t\to\tau}(g))[\eta_\tau^0]$  and $\|K_\tau^\epsilon[\eta^\epsilon_\tau]\|\le M_R(\tau)\|h\| e^{\int_t^rM_R(\xi) \mathrm{d}\xi}$ with $M_R\in L^1(0,1)$, and then the Dominated Convergence Theorem allows us to pass the limit in \autoref{eq:eta_equation} and obtain the linear integral equation of $\eta_\tau^{0}$.  This integral equation follows directly that $\eta_r^{0}$ is linear in $h$.  By applying Grönwall’s inequality, we can easily further obtain a uniform bound on $\|\eta_r^{0}\|$, implying that $\eta_r^{0}$ defines a bounded linear operator in $h$.  Hence, $\phi_{t\to r}(g)$ is in fact Fréchet differentiable. Define the Fréchet derivative as $\eta_r^{0} = D\phi_{t\to r}(g)$, and then with $J_\tau(g) = D\phi_{t\to \tau}(g)$ we have
\begin{equation}
    \begin{aligned}
J_r(g)[h]
= h+\int_t^r Du_\tau\big(\phi_{t\to \tau}(g)\big)[J_\tau(g)[h]]\mathrm{d}\tau.  
    \end{aligned}
\end{equation}
Applying Gr\"onwall's inequality to the above, we obtain that for all $\forall\tau \in [t,r]$, $\|J_\tau(g)\|\le e^{\int_t^\tau M_R(\xi)\mathrm{d}\xi}$. Hence, $Du_\tau\big(\phi_{t\to \tau}(g)\big)[J_\tau(g)[h]]$ is integrable over $[t,r]$, which implies that $J_r(g)[h]$ is absolutely continuous with respect to $r$. Therefore, we have:
\begin{equation}
    \begin{aligned}
        \frac{\mathrm{d}}{\mathrm{d}r} J_r(g) = Du_r\big(\phi_{t\to r}(g)\big)\circ J_r(g), \quad J_t(g) = \text{Id}_\mathcal{F}.
    \end{aligned}
\end{equation}

For any $g\in \mathcal{\tilde B}_R^{t,r}$, since $\phi_{t\to r}$ is Lipschitz continuous, there exists a neighborhood of $g$, denoted by $B_{g,\epsilon}=\{h\in \mathcal{F}|\|g-h\|<\epsilon\}$, such that $B_{g,\epsilon}\subset \mathcal{\tilde B}_R^{t,r}$.
For any $g' \in B_{g,\epsilon}$, using the integral equation above, we have:
\begin{equation}
    \begin{aligned}
        &J_r(g)-J_r(g')\\
        &= \int_t^r Du_\tau\big(\phi_{t\to \tau}(g)\big)\circ J_\tau(g)\mathrm{d}\tau \\
        &- \int_t^r Du_\tau\big(\phi_{t\to \tau}(g')\big)\circ J_\tau(g')\mathrm{d}\tau\\
        &= \int_t^r Du_\tau\big(\phi_{t\to \tau}(g)\big)\circ (J_\tau(g)-J_\tau(g'))\\
        &+ (Du_\tau\big(\phi_{t\to \tau}(g)\big)-Du_\tau\big(\phi_{t\to \tau}(g')\big))\circ J_\tau(g')\mathrm{d}\tau.\\
    \end{aligned}
\end{equation}
Taking norms on both sides of the above equation, we have
\begin{equation}
    \begin{aligned}
        &\|J_r(g)-J_r(g')\|\\ 
        &\le \int_t^r \|Du_\tau\big(\phi_{t\to \tau}(g)\big)\|\|J_\tau(g)-J_\tau(g')\|\mathrm{d}\tau\\
        &+ \int_t^r\|Du_\tau\big(\phi_{t\to \tau}(g)\big)-Du_\tau\big(\phi_{t\to \tau}(g')\big)\|\|J_\tau(g')\|\mathrm{d}\tau.\\
    \end{aligned}
\end{equation}
Set $\epsilon_{g'}(r) = \int_t^r\|Du_\tau\big(\phi_{t\to \tau}(g)\big)-Du_\tau\big(\phi_{t\to \tau}(g')\big)\|\|J_\tau(g')\|\mathrm{d}\tau$. Since both $Du_\tau$ and $\phi_{t\to \tau}$ continuous, we have $\|Du_\tau\big(\phi_{t\to \tau}(g)\big)-Du_\tau\big(\phi_{t\to \tau}(g')\big)\|\to 0$, as $g'\to 0$.  Moreover, because$\|Du_\tau\big(\phi_{t\to \tau}(g)\big)-Du_\tau\big(\phi_{t\to \tau}(g')\big)\|\|J_\tau(g')\|\le 2M_R(\tau)e^{\int_t^\tau M_R(\xi)\mathrm{d}\xi}$ by the Dominated Convergence Theorem, we have $\epsilon_{g'}(r)\to 0$ as $g'\to g$. Meanwhile, $\|J_r(g)-J_r(g')\|$ satisfies:
\begin{equation}
    \begin{aligned}
        &\|J_r(g)-J_r(g')\|\\
        &\le \int_t^r \|Du_\tau\big(\phi_{t\to \tau}(g)\big)\|\|J_\tau(g)-J_\tau(g')\|\mathrm{d}\tau\\
        &+ \int_t^r\|Du_\tau\big(\phi_{t\to \tau}(g)\big)-Du_\tau\big(\phi_{t\to \tau}(g')\big)\|\|J_\tau(g')\|\mathrm{d}\tau\\
        &\le \int_t^r M_R(\tau)\|J_\tau(g)-J_\tau(g')\|\mathrm{d}\tau + \epsilon_{g'}(r).
    \end{aligned}
\end{equation}
By applying Gr\"onwall's inequality, we have:
\begin{equation}
    \begin{aligned}
        \|J_r(g)-J_r(g')\|\le \epsilon_{g'}(r) e^{\int_t^r  M_R(\tau) \mathrm{d}\tau}.
    \end{aligned}
\end{equation}
Since $\epsilon_{g'}(r)\to 0$ as $g'\to g$, it follows that $\|J_r(g)-J_r(g')\|\to 0$.  Hence, $J_r(g)$ is continuous at $g$.
\end{proof}

Since $u_t$ is not directly available, and during training we construct the conditional velocity field $u_t^f$ and represent $u_t$ as expectation of $u_t^f$ as \autoref{eq:ecpectional_conditional}, we need to reformulate the assumptions in \autoref{thm:differentiability_of_flow} in terms of $u_t^f$.

\begin{lemma}\label{thm:differentiability_of_flow2} Let $\mu_t^{f}\ll\mu_t$ for $\nu$-a.e. $f$ and almost every $t\in [0,1]$ and thus define $\rho_t^{f}(g):=\frac{\mathrm{d}\mu_t^{f}}{\mathrm{d}\mu_t}(g)$, which is validated by \autoref{thm:ffm_thm1}. Assume that for $\nu$-a.e. $f$ and a.e. $t\in[0,T]$ and any radius $R>0$, there exist measurable nonnegative functions $A_{R,f}(t)$, $B_{R,f}(t)$, $C_{R,f}(t)$, $E_{R,f}(t)$, $L_{R,f}(t)$, $M_{R,f}(t)$:
\begin{enumerate}
    \item[\textnormal{(B1)}](Continuity) $u_t^{f}(g)$ and $\rho_t^{f}(g)$ are measurable in $t$ and uniform Lipschitz in $g\in \mathcal{B}_R$: $\forall x,y\in \mathcal{B}_R$
    \begin{equation}
        \begin{aligned}
            \|u_t^{f}(x)-u_t^{f}(y)\|\le L_{R,f}(t)\|x-y\|,\\
            |\rho_t^{f}(x)-\rho_t^{f}(y)|\le E_{R,f}(t)\|x-y\|.
        \end{aligned}
    \end{equation}    
    \item[\textnormal{(B2)}] ($C^1$ in $g$ with bounds) $u_t^{f}\in C^1(\mathcal F;\mathcal F)$ and $\rho_t^{f}\in C^1(\mathcal F;\mathbb R)$ are continuously Fr\'echet differentiable and $\forall g \in \mathcal{B}_R$
        \begin{equation}
            \begin{aligned}
        \|Du_t^{f}(g)\|\le M_{R,f}(t),\qquad \|D\rho_t^{f}(g)\|\le A_{R,f}(t).        
            \end{aligned}
        \end{equation}
    \item[\textnormal{(B3)}] (Integrable envelopes) $\forall g \in \mathcal{B}_R$,  $\|u_t^{f}(g)\|\le B_{R,f}(t)\in L^1(0,1)$ and $0 \le \rho_t^{f}(g)\le C_{R,f}(t)$ are bounded and 
    $L_R(t) \in  L^1(0,1)$ and $M_R(t)\in  L^1(0,1)$ and $U_R(t)\in  L^1(0,1)$ are well-defined and integrable, which are defined as
    \begin{equation}
        \begin{aligned}
            &L_R(t):=\int_{\mathcal F}\big(L_{R,f}(t)C_{R,f}(t)\\
            &\qquad\qquad\qquad+E_{R,f}(t)B_{R,f}(t)\big)\mathrm{d}\nu(f),\\
            &M_R(t):=\int_{\mathcal F}\big(M_{R,f}(t)C_{R,f}(t)\\
            &\qquad\qquad\qquad+A_{R,f}(t)B_{R,f}(t)\big)\mathrm{d}\nu(f),\\
            &U_R(t):=\int_{\mathcal F}B_{R,f}(t)C_{R,f}(t)\mathrm{d}\nu(f).
        \end{aligned}
    \end{equation}
 \end{enumerate}

Then the marginal field $u_t$ satisfies the hypotheses \textnormal{(A1)}–\textnormal{(A2)} as
\begin{align*}
\|u_t(x)-u_t(y)\| &\le L_R(t)\|x-y\|,  \forall x,y\in \mathcal{B}_R,\\
\|Du_t(g)\| &\le M_R(t), \forall g\in \mathcal{B}_R.
\end{align*}
\end{lemma}
\begin{proof}
   For any $x, y \in \mathcal{B}_R$, compute $u_t(x) - u_t(y)$.
   \begin{equation}
       \begin{aligned}
           &u_t(x) - u_t(y)\\
           &=\int_\mathcal{F} u_t^f(x)\rho_t^f(x) \mathrm{d}\nu(f) - \int_\mathcal{F} u_t^f(y)\rho_t^f(y) \mathrm{d}\nu(f)\\
           &=\int_\mathcal{F} [u_t^f(x)-u_t^f(y)]\rho_t^f(x) \mathrm{d}\nu(f) \\
           &\qquad+ \int_\mathcal{F} u_t^f(y)[\rho_t^f(x)-\rho_t^f(y)] \mathrm{d}\nu(f).
       \end{aligned}
   \end{equation}
   Taking norms on both sides gives:
\begin{equation}
    \begin{aligned}
        \|u_t(x) - u_t(y)\|&\le \int_\mathcal{F} L_{R,f}(t)\|x-y\|C_{R,f}(t)\mathrm{d}\nu(f)\\
        &+ \int_\mathcal{F} B_{R,f}(t)E_{R,f}(t)\|x-y\|\mathrm{d}\nu(f)\\
        &=L_R(t)\|x-y\|.
    \end{aligned}
\end{equation}

By (B3) we have $L_R\in L^{1}(0,1)$.  Moreover, since $u_t^{f}(g)$ and $\rho_t^{f}(g)$ are measurable and bounded, $u_t(g)$ is measurable in $t$. Since $\|u^f_t(g)\| \le B_{R,f}(t)$ and $\|u_t(g)\|\le \int_\mathcal{F}\|u^f_t(g)\|\rho_t^f\mathrm{d}\nu(f) \le \int_\mathcal{F}B_{R,f}(t)C_{R,f}(t)\mathrm{d}\nu(f)= U_R(t)\in L^1(0,1)$, we have that $u_t(g)$ is integrable with respect to $t$. Hence, (A1) holds.

To prove that $u_t$ is Fr\'echet differentiable, we explicitly write out $Du_t(g)$ and then show that this $Du_t(g)$ indeed serves as the Fr\'echet derivative of $u_t$, thereby establishing the Fr\'echet differentiability of $u_t$:
\begin{equation}\label{eq:definition_Dv_expectation}
    \begin{aligned}
        Du_t(g) =\int(Du_t^f(g)\rho_t^f(g)+u_t^f(g)\otimes D\rho_t^f(g))\mathrm{d}\nu(f),
    \end{aligned}
\end{equation}
where $\otimes$ is the tensor product symbol and $u_t^f(g)\otimes D\rho_t^f(g)):\mathcal{F}\to \mathcal{F}$ is calculated as $(u_t^f(g)\otimes D\rho_t^f(g)))[h] = D\rho_t^f(g)[h] u_t^f(g)$.

Since $u_t^{f}\rho_t^{f}$ is Fréchet differentiable, the difference $u_t^{f}(g+h)\rho_t^{f}(g+h) - u_t^{f}(g)\rho_t^{f}(g)$ can be written in integral form, and by the product rule for differentiation, we have:
\begin{equation}
    \begin{aligned}
        &u_t^{f}(g+h)\rho_t^{f}(g+h) - u_t^{f}(g)\rho_t^{f}(g)\\
        &= \int_0^1 D(u_t^{f}\rho_t^{f})(g+\theta h)[h] \mathrm{d}\theta\\
        &=\int_0^1 (Du_t^f(g+\theta h)\rho_t^f(g+\theta h)\\
        &\qquad+u_t^f(g+\theta h)\otimes D\rho_t^f(g+\theta h))[h] \mathrm{d}\theta.
    \end{aligned}
\end{equation}
Set the difference $R_f(h) = u_t^{f}(g+h)\rho_t^{f}(g+h) - u_t^{f}(g)\rho_t^{f}(g)-(Du_t^f(g)\rho_t^f(g)+u_t^f(g)\otimes D\rho_t^f(g))[h]$, which could be used to represent the difference $u_t^f(g+h)\rho_t^f(g+h)-u_t^f(g)\rho_t^f(g)-Du_t(g)[h]$ as $u_t^f(g+h)\rho_t^f(g+h)-u_t^f(g)\rho_t^f(g)-Du_t(g)[h] = \int_\mathcal{F} R_f(h) \mathrm{d}\nu(f)$. Here, $\|h\|$ is small enough to ensure $\|g+h\|<R$.  Then we calculate $R_f(h)$ as
\begin{equation}
    \begin{aligned}
        &R_f(h)\\
        &=\int_0^1 (Du_t^f(g+\theta h)\rho_t^f(g+\theta h)\\
        &\qquad+u_t^f(g+\theta h)\otimes D\rho_t^f(g+\theta h))[h] \mathrm{d}\theta\\
        &-\int_0^1 (Du_t^f(g)\rho_t^f(g)+u_t^f(g)\otimes D\rho_t^f(g))[h]\mathrm{d}\theta\\   
        &=\int_0^1 ((Du_t^f(g+\theta h)-Du_t^f(g))\rho_t^f(g+\theta h)\\
        &+Du_t^f(g)(\rho_t^f(g+\theta h)-\rho_t^f(g))\\
        &+(u_t^f(g+\theta h)-u_t^f(g))\otimes D\rho_t^f(g+\theta h)\\        &+u_t^f(g)\otimes(D\rho_t^f(g+\theta h)-D\rho_t^f(g)))[h]\mathrm{d}\theta.
    \end{aligned}
\end{equation} 
Taking norms on both sides gives:
\begin{equation}
    \begin{aligned}
        &\|R_f(h)\|\\
        &\le\int_0^1 (\|Du_t^f(g+\theta h)-Du_t^f(g)\|\|\rho_t^f(g+\theta h)\|\\
        &+\|Du_t^f(g)\|\|\rho_t^f(g+\theta h)-\rho_t^f(g)\|\\
        &+\|u_t^f(g+\theta h)-u_t^f(g)\|\| D\rho_t^f(g+\theta h)\|\\        
        &+\|u_t^f(g)\|\|D\rho_t^f(g+\theta h)-D\rho_t^f(g))\|\|h\|\mathrm{d}\theta\\
        &\le\int_0^1 4M_{R,f}(t)C_{R,f}(t)+4A_{R,f}(t)B_{R,f}(t) \mathrm{d}\theta\\
        &=[4M_{R,f}(t)C_{R,f}(t)+4A_{R,f}(t)B_{R,f}(t)]\|h\|.
    \end{aligned}
\end{equation}
Thus, since $M_R(t)$ is well defined, it follows that $4M_{R,f}(t)C_{R,f}(t)+4A_{R,f}(t)B_{R,f}(t)$ is integrable with respect to $\mathrm{d}\nu(f)$. Moreover, because $u_t^{f}\rho_t^{f}$ is Fréchet differentiable, we have $\frac{\|R_f(h)\|}{\|h\|}\to0$. Hence, by the Dominated Convergence Theorem, $\int_{\mathcal F}\frac{\|R_f(h)\|}{\|h\|}\mathrm{d}\nu(f)\to 0$, which implies that $\int_{\mathcal F}R_f(h)\mathrm{d}\nu(f)=o(\|h\|)$, that is,
\begin{equation}
    \begin{aligned}
        &\int_{\mathcal F}R_f(h)\mathrm{d}\nu(f)\\
        &=\int_\mathcal{F}u_t^{f}(g+h)\rho_t^{f}(g+h) - u_t^{f}(g)\rho_t^{f}(g)\\
        &-(Du_t^f(g)\rho_t^f(g)+u_t^f(g)\otimes D\rho_t^f(g))[h] \mathrm{d}\nu(f)\\
        &=u_t(g+h)-u_t(g)-D u_t(g)[h] = o(\|h\|).\\
    \end{aligned}
\end{equation}
Hence, $Du_t(g)$ is the Fréchet derivative of $u_t(g)$, and thus $u_t$ is Fréchet differentiable.

Taking norms on both sides of \autoref{eq:definition_Dv_expectation} gives:
\begin{equation}
    \begin{aligned}
        \|Du_t(g)\| &\le \int_\mathcal{F}(\|Du_t^f(g)\|\|\rho_t^f(g)\|+\|u_t^f(g)\|\|D\rho_t^f(g))\|\mathrm{d}\nu(f)\\
        &\le \int_\mathcal{F}M_{R,f}(t)C_{R,f}(t)+B_{R,f}(t)A_{R,f}(t)\mathrm{d}\nu(f).\\
        &=M_R(t)
    \end{aligned}
\end{equation}
Since $M_R(t)$ is well-defined and integrable, condition (A2) holds.
\end{proof}

Following the same reasoning as in \cite{kerrigan2024functional}, our next step is to incorporate the specific choices of the conditional path and conditional velocity in \autoref{eq:selection_conditional}, in order to translate the assumptions on $u_t^{f}$ and $\rho_t^{f}$ in \autoref{thm:differentiability_of_flow2} into the corresponding assumptions that the dataset must satisfy.

\begin{lemma}\label{thm:differentiability_of_flow3}
We choose the conditional measures and conditional velocity in \autoref{eq:selection_conditional} as
\begin{equation}
    \begin{aligned}
        &\mu_t^f = \mathcal{N}(m_t^f, (\sigma_t^f)^2 C_0),\\
        &m_t^f = t f, \sigma_t^f = 1 - (1 - \sigma_{\min}) t,\\
        &u_t^f(g)= \frac{\dot\sigma_t^f}{\sigma_t^f}(g - m_t^f) + \dot m_t^f= I_t g + J_t f,\\
&I_t = -\frac{1 - \sigma_{\min}}{1 - (1 - \sigma_{\min})t}, J_t = 1 - tI_t,
    \end{aligned}
\end{equation}
where $\sigma_{\min} > 0$. Assume that
\begin{enumerate}
    \item[\textnormal{(C1)}](Finite Second Moment) The data distribution $\nu$ satisfies $\int_{\mathcal{F}}\|f\|^2\mathrm{d}\nu(f) < \infty$.
 \end{enumerate}
Then the assumptions (B1)–(B3) in \autoref{thm:differentiability_of_flow2} holds.
\end{lemma}

\begin{proof}
    First, we compute $\rho_t^f(g)$, where $\rho_t^f(g) = \frac{\mathrm{d}\mu_t^f}{\mathrm{d}\mu_t}$ denotes the Radon–Nikodym derivative between measures $\mu_t^f$ and $\mu_t$.  In infinite-dimensional spaces, since there is no Lebesgue measure, the Radon–Nikodym derivative must be taken with respect to a reference measure. We choose the reference measure $\rho_0^f = \rho_0 = \mathcal{N}(0, C_0)$ and first compute the Radon–Nikodym derivatives $\frac{\mathrm{d}\mu_t^f}{\mathrm{d}\mu_0}$ and $\frac{\mathrm{d}\mu_t}{\mathrm{d}\mu_0}$ with respect to this reference measure.  Through Cameron–Martin theorem, $\frac{\mathrm{d}\mu_t^f}{\mathrm{d}\mu_0}$ and $\frac{\mathrm{d}\mu_t}{\mathrm{d}\mu_0}$ can be calculated as
    \begin{equation}
        \begin{aligned}
            \frac{\mathrm{d}\mu_t^f}{\mathrm{d}\mu_0}&=e^{\left(
\left\langle \frac{m_t^f}{\sigma_t^2}, g \right\rangle
- \frac{1}{2} \left\| \frac{m_t^f}{\sigma_t} \right\|^2
\right)},\\
\frac{\mathrm{d}\mu_t}{\mathrm{d}\mu_0} &= \int_\mathcal{F}e^{\left(
\left\langle \frac{m_t^{f'}}{\sigma_t^2}, g \right\rangle
- \frac{1}{2} \left\| \frac{m_t^{f'}}{\sigma_t} \right\|^2
\right)}\mathrm{d}\nu(f').\\
        \end{aligned}
    \end{equation}
    By the Radon–Nikodym ratio formula, we have:
    \begin{equation}
        \begin{aligned}
            \frac{\mathrm{d}\mu_t^f}{\mathrm{d}\mu_t} &= \frac{\mathrm{d}\mu_t^f/\mathrm{d}\mu_0}{\mathrm{d}\mu_t/\mathrm{d}\mu_0}\\
            &= \frac{e^{(\langle a_t^f, g \rangle - b_t^f)}}{\int_\mathcal{F}e^{(\langle a_t^{f'}, g \rangle - b_t^{f'})}\mathrm{d}\nu(f')},\\
        \end{aligned}
    \end{equation}
where $a_t^f = \frac{m_t^f}{\sigma_t^2}$ and $b_t^f = \frac{1}{2} \left\| \frac{m_t^f}{\sigma_t} \right\|^2 = \frac{1}{2}\|a_t^f\|^2$. 

Let $s_t^f(g) = e^{(\langle a_t^f, g \rangle - b_t^f)}$ and $Z_t(g) = \int_{\mathcal{F}} s_t^{f'}(g)\mathrm{d}\nu(f')$.  We then compute $D\rho_t^f(g) = D\left(\frac{s_t^f(g)}{Z_t(g)}\right)$.  First, $Ds_t^f(g)[h] = s_t^f(g)\langle a_t^f, h\rangle$, while $DZ_t(g)[h]$ is given by
\begin{equation}
    \begin{aligned}
        DZ_t(g)[h] &= \int_\mathcal{F} Ds_t^{f'}(g)[h]\nu(df')\\
        & = \int_\mathcal{F} s_t^{f'}(g)\langle a_t^{f'}, h\rangle\nu(df')\\
        & = Z_t(g)  \langle \bar a_t(g),h \rangle,
    \end{aligned}
\end{equation}
where $\bar a_t(g) = \int_\mathcal{F} a_t^{f'}\rho_t^{f'}(g)\mathrm{d}\nu(f')$ is defined as the $\rho_t^{f'}(g)$-weighted average of $a_t^{f'}$ with respect to $\nu$.  Then $D\rho_t^f(g)[h]$ can be calculated as
\begin{equation}
    \begin{aligned}
        D\rho_t^f(g)[h]&=\frac{Ds_t^f(g)[h]Z_t(g)-s_t^f(g)DZ_t(g)[h]}{Z_t(g)^2}\\
        &=\frac{s_t^f(g)}{Z_t(g)}(\langle a_t^f,h\rangle- \langle \bar a_t,h\rangle)\\
        &=\rho_t^f(g)\langle a_t^f-\bar a_t(g),h\rangle,
    \end{aligned}
\end{equation}

Next, we estimate the bounds for $\rho_t^{f}$ and $D\rho_t^{f}(g)$.  Since $a_t^{f} = \frac{m_t^{f}}{(\sigma_t^{f})^2} = \frac{t}{(\sigma_t^{f})^2} f$ is bounded by $|a_t^{f}|\le \sigma_{\min}^{-2}|f|$ and the inequality $Rr - \tfrac{1}{2}r^2 \le \tfrac{1}{2}R^2$ holds for all $r>0$, we have
\begin{equation}
    \begin{aligned}
    s_t^f(g) &= e^{(\langle a_t^f, g \rangle - b_t^f)}\\
    &\le e^{(R\|a_t^f\| - \frac{1}{2}\|a_t^f\|^2)}\\
    &\le e^{\frac{1}{2}R^2}.
    \end{aligned}
\end{equation}
For $Z_t(g)$ we have
\begin{equation}
    \begin{aligned}
        Z_t(g) &= \int_{\mathcal{F}} s_t^{f'}(g)\mathrm{d}\nu(f')\\
        &= \int_{\mathcal{F}} e^{(\langle a_t^{f'}, g \rangle - b_t^{f'})}\mathrm{d}\nu(f')\\
        &\ge \int_{\mathcal{F}} e^{(-R\|a_t^{f'}\| - \frac{1}{2}\|a_t^{f'}\|^2)}\mathrm{d}\nu(f').
    \end{aligned}
\end{equation}
Let $c_{R}(t) = \int_{\mathcal{F}} e^{(-R\|a_t^{f'}\| - \frac{1}{2}\|a_t^{f'}\|^2)}\mathrm{d}\nu(f')$ and $C_{R,f}(t) = \frac{e^{\frac{R^2}{2}}}{c_{R}(t)}$. Then we have:
\begin{equation}
    \begin{aligned}
        \rho_t^f(g)\le C_{R,f}(t).       
    \end{aligned}
\end{equation}
Since $\int_\mathcal{F}\|f\|^2\mathrm{d}\nu(f)<\infty$, let $F_1 = \int_\mathcal{F}\|f\|\mathrm{d}\nu(f)$ and  $F_2 = \int_\mathcal{F}\|f\|^2\mathrm{d}\nu(f)$. Then we have
\begin{equation}
    \begin{aligned}
        \|D\rho_t^f(g)\| &= \sup_{\|h\|=1} \|D\rho_t^f(g)[h]\|\\
        &\le \rho_t^f(g) (\|a_t^f\|+\|\bar a_t(g)\|)\\
        &\le \rho_t^f(g) (\sigma_{\text{min}}^{-2}\|f\|+\int_\mathcal{F} \|a_t^{f'}\|\rho_t^{f'}(g)\mathrm{d}\nu(f'))\\
        &\le C_{R,f}(t)(\sigma_{\text{min}}^{-2}\|f\|+\sigma_{\text{min}}^{-2}C_{R,f}(t)F_1).
    \end{aligned}
\end{equation}
Define the $g$-independent quantity $A_{R,f}(t)$ as $A_{R,f}(t) = C_{R,f}(t)(\sigma_{\text{min}}^{-2}\|f\|+\sigma_{\text{min}}^{-2}C_{R,f}(t)F_1)$ and hence we have:
\begin{equation}
    \begin{aligned}
        \|D\rho_t^f(g)\| \le A_{R,f}(t),
    \end{aligned}
\end{equation}
$\forall x,y\in B_R$, by writing $\rho_t^f(x)-\rho_t^f(y)=\int_{0}^{1} D\rho_t^f(y+\theta(x-y))[x-y]\mathrm{d}\theta$, we have
\begin{equation}
    \begin{aligned}
        \|\rho_t^f(x)-\rho_t^f(y)\|&\le \int_{0}^{1} \|D\rho_t^f(y+\theta(x-y))\|\|x-y\|\mathrm{d}\theta \\
        &\le A_{R,f}(t)\|x-y\|. 
    \end{aligned}
\end{equation}
Thus we can write $E_{R,f}(t) = A_{R,f}(t)$ and have:
\begin{equation}
    \begin{aligned}
                \|\rho_t^f(x)-\rho_t^f(y)\|\le E_{R,f}(t)\|x-y\|.
    \end{aligned}
\end{equation}
For the velocity field $u_r^{f}(g)$, we have:
\begin{equation}
    \begin{aligned}
        \|u_r^{f}(g)\| &= \|I_tg+J_tf\|\\
        &\le \|I_t\|\|g\|+\|J_t\|\|f\|\\
        &\le  \sigma_{\text{min}}^{-1}R+\|f\|(1+\sigma_{\text{min}}^{-1}).   
    \end{aligned}
\end{equation}
Thus, by setting $B_{R,f}(t)=\sigma_{\min}^{-1}R+\|f\|(1+\sigma_{\min}^{-1})$, we have:
\begin{equation}
    \begin{aligned}
        \|u_r^{f}(g)\|\le B_{R,f}(t).
    \end{aligned}
\end{equation}
Since $Du_t^{f}(g)[h]=I_t h$, it follows that $\|Du_t^{f}(g)\|=|I_t|<\sigma_{\min}^{-1}$. By taking $M_{R,f}(t)=\sigma_{\min}^{-1}$, we have:
\begin{equation}
    \begin{aligned}
        \|Du_t^{f}(g)\|\le M_{R,f}(t).
    \end{aligned}
\end{equation}
Since $u_t^{f}(x) - u_t^{f}(y) = I_t(x - y)$, by taking $L_{R,f}(t) = \sigma_{\min}^{-1}$, we have:
\begin{equation}
    \begin{aligned}
        \|u_t^{f}(x) - u_t^{f}(y)\|&\le |I_t|\|x-y\|\\
        &\le \sigma_{\min}^{-1} \|x-y\|\\
        &=L_{R,f}(t)\|x-y\|.
    \end{aligned}
\end{equation}

Then we express $L_R(t)$, $M_R(t)$ and $U_R(t)$ respectively.
\begin{equation}
    \begin{aligned}
        &L_R(t)\\
        &=\int_{\mathcal F}\sigma_{\text{min}}^{-1}\frac{e^{\frac{R^2}{2}}}{c_R(t)}+\frac{e^{\frac{R^2}{2}}}{c_R(t)}(\sigma_{\text{min}}^{-2}\|f\|\\
        &+\sigma_{\text{min}}^{-2}\frac{e^{\frac{R^2}{2}}}{c_R(t)}F_1)(\sigma_{\min}^{-1}R+\|f\|(1+\sigma_{\min}^{-1}))\mathrm{d}\nu(f)\\
        &= (\sigma_{\text{min}}^{-1}\frac{e^{\frac{R^2}{2}}}{c_R(t)}+\sigma_{\text{min}}^{-3}\frac{e^{R^2}}{c_R^2(t)}F_1R) \\
        &+ (\sigma_{\text{min}}^{-3} \frac{e^{\frac{R^2}{2}}}{c_R(t)} R+\sigma_{\text{min}}^{-2}(1+\sigma_{\text{min}}^{-1})\frac{e^{\frac{R^2}{2}}}{c_R(t)}F_1)F_1\\
        &+\sigma_{\text{min}}^{-2}(1+\sigma_{\text{min}}^{-1})\frac{e^{\frac{R^2}{2}}}{c_R(t)}F_2.
    \end{aligned}
\end{equation}
Moreover, since $c_R(t)$ has a strictly positive lower bound for all $t \in [0,1]$
\begin{equation}
    \begin{aligned}
        c_{R}(t) &= \int_{\mathcal{F}} e^{(-R\|a_t^{f'}\| - \frac{1}{2}\|a_t^{f'}\|^2)}\mathrm{d}\nu(f')\\
        &= \int_{\mathcal{F}} e^{(-R\|\frac{tf}{\sigma_t^2}\| - \frac{1}{2}\|\frac{tf}{\sigma_t^2}\|^2)}\mathrm{d}\nu(f')\\
        &\ge \int_{\mathcal{F}} e^{(-R\|\frac{1}{\sigma_{\min}^2}\|\|f'\| - \frac{1}{2}\|\frac{1}{\sigma_{\min}^2}\|^2\|f'\|^2)}\mathrm{d}\nu(f')\\
        &>0.
    \end{aligned}
\end{equation}
Therefore, there exists $M_1>0$ such that $L_R(t)<M_1$ for all $t\in[0,1]$, and hence $L_R(t)\in L^1(0,1)$.  

And $M_R(t)$ and $U_R(t)$ can be computed similarly and likewise,since $c_R(t)$ has a strictly positive lower bound for all $t\in[0,1]$, there exists $M_2>0$ such that $M_R(t)<M_2$ for all $t\in[0,1]$ and there exists $U_2>0$ such that $U_R(t)<U_2$ for all $t\in[0,1]$, and thus $M_R(t)\in L^1(0,1)$ and $U_R(t)\in L^1(0,1)$.
\end{proof}
\subsection{Proof of \autoref{thm:flow-diff}}\label{sec:flow-diff}

\paragraph{\textbf{\autoref{thm:flow-diff}}}(Initial-Time Derivative of Two-Parameter Flow)
Assume that the dataset measure $\nu$ satisfies $\int_{\mathcal F}|f|^2\mathrm{d}\nu(f)<\infty$, and the conditions of Functional Flow Matching \cite{kerrigan2024functional} hold.  With the conditional flow and conditional velocity chosen in \autoref{eq:selection_conditional}, the corresponding marginal two-parameter flow $\phi_{t\to r}(g)$ is differentiable with respect to $t$ and Fréchet differentiable with respect to $g$ and satisfies, for any $0 <t < r<1$
\begin{equation}
    \begin{aligned}
        &\frac{\partial }{\partial t} \phi_{t\to r}(g) = -D \phi_{t\to r}(g)[u_t(g)],
    \end{aligned}
\end{equation}
where $D\phi_{t\to r}(g):\mathcal{F}\to \mathcal{F}$ is the Fréchet derivative of $\phi_{t\to r}$ at $g$.  This theorem follows from \textcolor{cvprblue}{Lemmas}~\ref{thm:differentiability_of_flow},\ref{thm:differentiability_of_flow2} and \ref{thm:differentiability_of_flow3} in \textcolor{cvprblue}{Appendix~\ref{sec:flow-diff_supporting}}.

\begin{proof}
Based on \textcolor{cvprblue}{Lemmas}~\ref{thm:differentiability_of_flow}, \ref{thm:differentiability_of_flow2}, and \ref{thm:differentiability_of_flow3}, we know that for any $R > 0$, if $g \in B_R$ and $\phi_{t\to \tau}(g) \in B_R$ for all $\tau \in [t, r]$, then $\phi_{t\to \tau}$ is Fréchet differentiable at $g$, $\phi_{t\to \tau}$ is continous at $t$ and $D\phi_{t\to \tau}$ is continous at $g$.  Moreover, by \textcolor{cvprblue}{Lemmas}~\ref{thm:differentiability_of_flow2}, since $u_t$ is bounded, for any $g \in \mathcal{F}$ we can always choose some $R > 0$ such that $\phi_{t\to \tau}(g) \in B_R$ for all $\tau \in [t, r]$, which ensures that $\phi_{t\to \tau}$ is differentiable at $g$, $\phi_{t\to \tau}$ is differentiable at $t$ and $D\phi_{t\to \tau}$ is continous at $g$.

Since $\phi_{t\to r}(g)$ is continuous at $t$, for sufficiently small $\epsilon>0$ we have $\phi_{t-\epsilon\to r}(g)\in B_R$. We then compute the left difference quotient of $\phi_{t\to r}(g)$ with respect to time $t$ as:
\begin{equation}
    \begin{aligned}
        &\frac{\phi_{t\to r}(g)-\phi_{t-h\to r}(g)}{h}\\
        &=\frac{\phi_{t\to r}(g)-\phi_{t\to r}(\phi_{t-\epsilon\to t}(g))}{h}\\
        &=\frac{D\phi_{t\to r}(g)[g-\phi_{t-\epsilon\to t}(g)]+o(\|g-\phi_{t-\epsilon\to t}(g)\|)}{h}.\\
    \end{aligned}
\end{equation}
Let $\delta^-_h(\tau) = \phi_{t-\epsilon\to \tau}(g)-g$, $\tau\in[t-\epsilon,t]$. Since $\phi_{t-\epsilon\to \tau}(g)-g = \int_{t-h}^\tau u_\xi(\phi_{t-h\to \xi}(g))\mathrm{d}\xi$, we have:
\begin{equation}\label{eq:final_proof}
    \begin{aligned}
        \frac{\delta^-_h(\tau)}{h}&= \frac{1}{h}\int_{t-h}^\tau u_\xi(\phi_{t-h\to \xi}(g))\mathrm{d}\xi\\
        &=\frac{1}{h}\int_{t-h}^\tau u_\xi(\phi_{t-h\to \xi}(g))-u_\xi(g)\mathrm{d}\xi \\
        &\qquad\qquad+ \frac{1}{h}\int_{t-h}^\tau u_\xi(g)\mathrm{d}\xi.\\
    \end{aligned}
\end{equation}
We want to determine the limit of $\frac{\delta^-_h(t)}{h}$ as $h\to 0$. Let $\frac{\delta^-_h(t)}{h} = R_h(t) + E_h(t)$, where $E_h(\tau) = \frac{1}{h}\int_{t-h}^\tau u_\xi(g)\mathrm{d}\xi$ and $R_h(\tau) = \frac{1}{h}\int_{t-h}^\tau u_\xi(\phi_{t-h\to \xi}(g))-u_\xi(g)\mathrm{d}\xi$. Based on \textcolor{cvprblue}{Lemmas}~\ref{thm:differentiability_of_flow2} and \ref{thm:differentiability_of_flow3}, $u_t(g)$ is integrable and measurable.  Thus based on Lebesgue differentiation theorem, we have $\|\frac{\int_{t-h}^t u_\xi(g)\mathrm{d}\xi}{h}- u_t(g)\|\to 0$ a.e. on $t$, which means $E_h(t)\to u_t(g)$ a.e. on $t$.  Next we prove $R_h(t)\to 0$.

Based on \textcolor{cvprblue}{Lemmas}~\ref{thm:differentiability_of_flow}, \ref{thm:differentiability_of_flow2}, and \ref{thm:differentiability_of_flow3}, $u_\xi$ is Lipschitz continuous $|u_\xi(\phi_{t-h\to \xi}(g)) - u_\xi(g)| \le L_R(\xi)|\phi_{t-h\to \xi}(g) - g|$ and $L_R\in L^1(0,1)$. Therefore we have
\begin{equation}
    \begin{aligned}
        &R_h(\tau)\\
        &=\|\int_{t-h}^\tau u_\xi(\phi_{t-h\to \xi}(g))-u_\tau(g)\mathrm{d}\xi\|\\
        &\le \frac{1}{h}\int_{t-h}^\tau\|u_\xi(\phi_{t-h\to \xi}(g))-u_\xi(g)\|\mathrm{d}\xi\\
        &\le \int_{t-h}^\tau L_R(\xi)\frac{\|\phi_{t-h\to \xi}(g)-g\|}{h}\mathrm{d}\xi\\
        &= \int_{t-h}^\tau L_R(\xi)\frac{\|\delta_h^-(\xi)\|}{h}\mathrm{d}\xi.\\
    \end{aligned}
\end{equation}
Combined with \autoref{eq:final_proof}, we have inequality for $\|\frac{\delta^-_{h}(\tau)}{h}\|$:
\begin{equation}
    \begin{aligned}
        &\| \frac{\delta^-_{h}(\tau)}{h}\|\\
        &\le \int_{t-h}^\tau L_R(\xi)\frac{\|\delta_{h,\xi}^-\|}{h}\mathrm{d}\xi + \|E_h(\tau)\|.
    \end{aligned}
\end{equation}
$\|E_h(\tau)\|$ is bounded arount $h=0$ for convergence.  By Gr\"onwall's inequality we have
\begin{equation}
    \begin{aligned}
        \| \frac{\delta^-_{h}(\tau)}{h}\| \le \|E_h(\tau)\| e^{\int_{t-h}^\tau L_R(\xi)\mathrm{d}\xi}.
    \end{aligned}
\end{equation}
Thus $R_h(t)$ can be calculated as
\begin{equation}
    \begin{aligned}
        R_h(t)&\le \int_{t-h}^t L_R(\tau)\frac{\|\delta_h^-(\tau)\|}{h}d\tau\\
        &\le \int_{t-h}^t L_R(\tau)\|E_h(\tau)\| e^{\int_{t-h}^\tau L_R(\xi)\mathrm{d}\xi}d\tau\\
        &\le e^{\int_{t-h}^t L_R(\xi)\mathrm{d}\xi} \int_{t-h}^t L_R(\tau)\frac{1}{h}\int_{t-h}^\tau \|u_\xi(g)\|\mathrm{d}\xi d\tau \\
        &\le e^{\int_{t-h}^t L_R(\xi)\mathrm{d}\xi} \int_{t-h}^t L_R(\tau) d\tau \frac{1}{h}\int_{t-h}^t \|u_\xi(g)\|\mathrm{d}\xi. \\
    \end{aligned}
\end{equation}
When $h\to 0$, $\int_{t-h}^t L_R(\xi)\mathrm{d}\xi\to 0$ as $L_R(\xi) \in L^1(0,1)$ integrable and $\frac{1}{h}\int_{t-h}^t \|u_\xi(g)\|\mathrm{d}\xi \to \|u_t(g)\|$ by Lebesgue differentiation theorem for $\|u_t(g)\|$ is measurable and integrable on $t$.  Therefore, we have $R_h(t)\to 0$.

Now we have $\frac{\phi_{t\to r}(g)-\phi_{t-h\to r}(g)}{h} = -D\phi_{t\to r}(g)[u_t(g)] + o(u_t(g))$.  Similarly, we can also prove that $\frac{\phi_{t+h\to r}(g)-\phi_{t\to r}(g)}{h} = -D\phi_{t\to r}(g)[u_t(g)] + o(u_t(g))$.  Therefore $\phi_{t\to r}(g)$ is differentiable with respect to $t$ and $\frac{\partial}{\partial t}\phi_{t\to r}(g)=-D\phi_{t\to r}(g)[u_t(g)]$.
\end{proof}

\subsection{Proof of \autoref{thm:loss-equivalence}}\label{sec:proof_loss_equivalence_u}
\paragraph{\textbf{\autoref{thm:loss-equivalence}}}(Equivalence of Mean Flow Conditional and Marginal Losses)
Under the assumptions of \autoref{thm:flow-diff}, we have $\mathcal{L}_c^M(\theta) = \mathcal{L}^M(\theta) + C$ where $C$ is independent of the model parameters~$\theta$.
\begin{proof}
    First, since we are working in a real Hilbert space, for any $f,g\in \mathcal{F}$, we have
    \begin{equation}
        \begin{aligned}
     &\|\bar u_{t\to r}(g) - \bar u_{t\to r}^\theta(g)\|^2 \\
     &= \langle \bar  u_{t\to r}(g) - \bar u_{t\to r}^\theta(g), \bar  u_{t\to r}(g) - \bar u_{t\to r}^\theta(g)\rangle\\
     &=\|\bar  u_{t\to r}(g)\|^2 + \|\bar u_{t\to r}^\theta(g)\|^2-2\langle\bar  u_{t\to r}(g), \bar u_{t\to r}^\theta(g)\rangle,
        \end{aligned}
    \end{equation}
    and similarly
    \begin{equation}
        \begin{aligned}
        &||(r-t)(\frac{\partial}{\partial t}\bar u_{t\to r}(g)\\
        &+D\bar u_{t\to r}(g)[u^f_t(g)])+u^f_t(g)-\bar u^\theta_{t\to r}(g)||^2\\
        &=\|(r-t)(\frac{\partial}{\partial t}\bar u_{t\to r}(g)+D\bar u_{t\to r}(g)[u^f_t(g)])+u^f_t(g)\|^2\\
        &+\|\bar u^\theta_{t\to r}(g)\|^2-2\langle (r-t)(\frac{\partial}{\partial t}\bar u_{t\to r}(g)\\
        &\qquad\qquad+D\bar u_{t\to r}(g)[u^f_t(g)])+u^f_t(g),\bar u^\theta_{t\to r}(g)\rangle.
        \end{aligned}
    \end{equation}

Note that the first term in both expressions is independent of the model parameters, so we focus on analyzing the remaining two terms.  First, we show that the second term in both expressions is identical, i.e., $\mathbb{E}_{t,r,g\sim\mu_t}[\|\bar u_{t\to r}^\theta(g)\|^2] = \mathbb{E}_{t,r,g\sim\mu^f_t,f\sim \mu_1}[\|\bar u_{t\to r}^\theta(g)\|^2]$
\begin{equation}
    \begin{aligned}
        &\mathbb{E}_{t,r,g\sim\mu_t}[\|\bar u_{t\to r}^\theta(g)\|^2]\\
        &= \int_0^1\int_0^1\int_g \|\bar u_{t\to r}^\theta(g)\|^2 \mathrm{d}\mu_t(g) \mathrm{d}t\mathrm{d}r\\
        &\overset{\textcircled{1}}{=} \int_0^1\int_0^1\int_f\int_g \|\bar u_{t\to r}^\theta(g)\|^2 \mathrm{d}\mu^f_t(g) \mathrm{d}\nu(f)\mathrm{d}t\mathrm{d}r\\
        &= \mathbb{E}_{t,r,g\sim\mu_t^f,f\sim \mu_1}[\|\bar u_{t\to r}^\theta(g)\|^2],
    \end{aligned}
\end{equation}
where $\textcircled{1}$ follows from the relationship between $\mu_t$ and $\mu_t^f$ given in \autoref{eq:ecpectional_conditional}.

Then, we show that the third term in both expressions is identical, i.e., $\langle\bar  u_{t\to r}(g), \bar u_{t\to r}^\theta(g)\rangle = \langle (r-t)(\frac{\partial}{\partial t}\bar u_{t\to r}(g)+D\bar u_{t\to r}(g)[u^f_t(g)])+u^f_t(g),\bar u^\theta_{t\to r}(g)\rangle$
\begin{equation}
    \begin{aligned}
        &\mathbb{E}_{t,r,g\sim\mu_t}[\langle\bar  u_{t\to r}(g), \bar u_{t\to r}^\theta(g)\rangle]\\
        &= \int_0^1\int_0^1\int_g \langle\bar  u_{t\to r}(g), \bar u_{t\to r}^\theta(g)\rangle \mathrm{d}\mu_t(g) \mathrm{d}t\mathrm{d}r\\
        &\overset{\textcircled{1}}{=}  \int_0^1\int_0^1\int_g\langle (r-t)\big(\frac{\partial}{\partial t}\bar{u}_{t\to r}(g) + D \bar u_{t\to r}(g)[u_t(g)]\big)\\
        &\qquad+  u_t(g), \bar u_{t\to r}^\theta(g)\rangle \mathrm{d}\mu_t(g) \mathrm{d}t\mathrm{d}r\\
        &\overset{\textcircled{2}}{=}  \int_0^1\int_0^1\int_g\langle (r-t)\big(\frac{\partial}{\partial t}\bar{u}_{t\to r}(g) \\
        &\qquad+ D \bar u_{t\to r}(g)[\int_{\mathcal{F}} u_t^f(g)\frac{\mathrm{d}\mu_t^f}{\mathrm{d}\mu_t}(g)\mathrm{d}\nu(f)]\big)\\
        &\qquad+  \int_{\mathcal{F}} u_t^f(g)\frac{\mathrm{d}\mu_t^f}{\mathrm{d}\mu_t}(g)\mathrm{d}\nu(f), \bar u_{t\to r}^\theta(g)\rangle \mathrm{d}\mu_t(g) \mathrm{d}t\mathrm{d}r\\
        &\overset{\textcircled{3}}{=}  \int_0^1\int_0^1\int_{\mathcal{F}}\int_g\langle (r-t)\big(\frac{\partial}{\partial t}\bar{u}_{t\to r}(g) + D \bar u_{t\to r}(g)[ u_t^f(g)]\big)\\
        &\qquad+  u_t^f(g), \bar u_{t\to r}^\theta(g)\rangle \frac{\mathrm{d}\mu_t^f}{\mathrm{d}\mu_t}(g)\mathrm{d}\mu_t(g) \mathrm{d}\nu(f) \mathrm{d}t\mathrm{d}r\\        &=  \int_0^1\int_0^1\int_{\mathcal{F}}\int_g\langle (r-t)\big(\frac{\partial}{\partial t}\bar{u}_{t\to r}(g) + D \bar u_{t\to r}(g)[ u_t^f(g)]\big)\\
        &\qquad+  u_t^f(g), \bar u_{t\to r}^\theta(g)\rangle \mathrm{d}\mu_t^f(g)) \mathrm{d}\nu(f) \mathrm{d}t\mathrm{d}r\\
        &=\mathbb{E}_{t,r,g\sim\mu_t^f,f\sim \mu_1}[\langle(r-t)(\frac{\partial}{\partial t}\bar u_{t\to r}(g)+D\bar u_{t\to r}(g)[u^f_t(g)])\\
        &\qquad+u^f_t(g),\bar u^\theta_{t\to r}(g)\rangle],
    \end{aligned}
\end{equation}
where $\textcircled{1}$ follows by substituting \autoref{eq:derivation_u_prediction}, $\textcircled{2}$ applies the relationship between $u_t$ and $u_t^f$ given in \autoref{eq:ecpectional_conditional}, and $\textcircled{3}$ uses the exchangeability between Bochner integrals and inner products, together with the Fubini–Tonelli theorem.  Therefore,  we have $\mathcal{L}_c^M(\theta) = \mathcal{L}^M(\theta) + C$ where $C$ is independent of the model parameters~$\theta$. 
\end{proof}
\subsection{Proof of \autoref{thm:loss-equivalence_x1} and Derivation of \autoref{eq:conditional_x1_flowmatching} and \autoref{eq:conditional_x1_prediction}}\label{sec:proof_for_all_x1}
We first present the derivation of \autoref{eq:conditional_x1_flowmatching}. Given the relationship between $u_{t}$ and $\hat{f}_{1,t}(g)$ as $\hat{f}_{1,t}(g) = (1-t)u_{t}(g)+g$, together with the relationship between $u_{t}(g)$ and $u_{t}^f(g)$ in \autoref{eq:ecpectional_conditional}, we have:
\begin{equation}
    \begin{aligned}
        \hat{f}_{1,t}(g)&=(1-t)u_{t}(g) + g\\
        &=(1-t)\int_\mathcal{F}u_{t}^f(g)\frac{\mathrm{d}\mu_t^f}{\mathrm{d}\mu_t}\mathrm{d}\nu(f)+g\\
        &=\int_\mathcal{F}((1-t)u_{t}^f(g)+g )\frac{\mathrm{d}\mu_t^f}{\mathrm{d}\mu_t}\mathrm{d}\nu(f).\\
    \end{aligned}
\end{equation}
Similar to $u_{t\to r}^f(g)$, $\hat{f}_{1,t\to r}^f(g)$ is defined through 
\begin{equation}\label{eq:maginal_conditional_x1}
    \begin{aligned}
        \hat{f}_{1,t}(g) = \int_\mathcal{F} \hat{f}_{1,t}^f(g) \frac{\mathrm{d}\mu_t^f}{\mathrm{d}\mu_t}\mathrm{d}\nu(f),        
    \end{aligned}
\end{equation}
so we take $\hat{f}_{1,t}^f(g) = (1-t)u_{t}^f(g)+g$.  Substituting the expression of $u_{t\to r}^f(g)$ from \autoref{eq:selection_conditional} gives:
\begin{equation}
    \begin{aligned}
        \hat{f}_{1,t}(g)&=(1-t)u_{t}^f(g)+g\\
        &=(1-t)\frac{1-\sigma_{\text{min}}}{1-(1-\sigma_{\text{min}})t}(tf-g)+(1-t)f+g\\
        &=(\frac{(1-t)(1-\sigma_{\text{min}})}{1-(1-\sigma_{\text{min}})t}-1)(tf-g)+f\\
        &= f- \frac{\sigma_{\text{min}}}{1-(1-\sigma_{\text{min}})t}(tf-g).
    \end{aligned}
\end{equation}

Then we  present the derivation of \autoref{eq:conditional_x1_prediction}.  Since $\bar{u}_{t\to r}$ satisfies \autoref{eq:derivation_u_prediction}, substituting the relationship between $\bar{u}_{t\to r}$ and $\hat{f}_{1,t\to r}$ given in \autoref{eq:relation_u_x1} yields:
\begin{equation}
    \begin{aligned}
    \frac{\hat{f}_{1,t\to r}(g)-g}{1-t} &= (r-t)\big(\frac{\partial}{\partial t}\frac{\hat{f}_{1,t\to r}(g)-g}{1-t} \\
    &\qquad+ D \frac{\hat{f}_{1,t\to r}(g)-g}{1-t}[u_t(g)]\big) +  u_t(g),
    \end{aligned}
\end{equation}
which gives:
\begin{equation}
    \begin{aligned}
    &\hat{f}_{1,t\to r}(g) \\
    &= (1-t)((r-t)\big(\frac{\partial}{\partial t}\frac{\hat{f}_{1,t\to r}(g)-g}{1-t} \\
    &+ D \frac{\hat{f}_{1,t\to r}(g)-g}{1-t}[u_t(g)]\big) +  u_t(g))+g\\
    &= (1-t)((r-t)\big(\frac{(1-t)\frac{\partial}{\partial t}\hat{f}_{1,t\to r}(g)+(\hat{f}_{1,t\to r}(g)-g)}{(1-t)^2} \\
    &+  \frac{D\hat{f}_{1,t\to r}(g)[u_t(g)]-u_t(g)}{1-t}\big) +  u_t(g))+g\\
    &=  (r-t)\frac{\partial}{\partial t}\hat{f}_{1,t\to r}(g)+\frac{r-t}{1-t}(\hat{f}_{1,t\to r}(g)-g)\\
    &+ (r-t)D\hat{f}_{1,t\to r}(g)[u_t(g)]+(1-r)u_t(g)+g\\
    &\overset{\textcircled{1}}{=}  (r-t)\frac{\partial}{\partial t}\hat{f}_{1,t\to r}(g)+\frac{r-t}{1-t}(\hat{f}_{1,t\to r}(g)-g)\\
    &+ (r-t)D\hat{f}_{1,t\to r}(g)[\frac{\hat{f}_{1,t}(g)-g}{1-t}]+(1-r)\frac{\hat{f}_{1,t}(g)-g}{1-t}+g\\
    &=  \frac{r-t}{1-t}((1-t)\frac{\partial}{\partial t}\hat{f}_{1,t\to r}(g)+D\hat{f}_{1,t\to r}(g)[\hat{f}_{1,t}(g)-g])\\
    &+ \frac{1-r}{1-t}\hat{f}_{1,t}(g)+\frac{r-t}{1-t}\hat{f}_{1,t\to r}(g),
    \end{aligned}
\end{equation}
where $\textcircled{1}$ is obtained by substituting the relationship between the flow-matching $x_1$-prediction target $\hat{f}_{1,t}(g)$ and $u_t(g)$, given by $\hat{f}_{1,t}(g) = (1 - t)u_t(g) + g$.  Moving the term $\frac{r - t}{1 - t}\hat{f}_{1,t\to r}(g)$ on the right-hand side of the above equation to the left-hand side and simplifying, we obtain:
\begin{equation}
    \begin{aligned}
        \hat{f}_{1,t\to r}(g) &= \frac{r-t}{1-t}((1-t)\frac{\partial}{\partial t}\hat{f}_{1,t\to r}(g)\\
        &+D\hat{f}_{1,t\to r}(g)[\hat{f}_{1,t}(g)-g])+\hat{f}_{1,t}(g).
    \end{aligned}
\end{equation}

Finally, we prove the \autoref{thm:loss-equivalence_x1}
\paragraph{\textbf{\autoref{thm:loss-equivalence_x1}}}(Equivalence of Mean Flow Conditional and Marginal Losses for $x_1$-prediction)
Under the assumptions of \autoref{thm:flow-diff}, we have $\mathcal{\tilde L}_c^\mathcal{F}(\theta) = \mathcal{\tilde L}^M(\theta) + C$ where $C$ is independent of the model parameters~$\theta$.
\begin{proof}
    The proof of this theorem proceeds similarly to \autoref{thm:loss-equivalence}.  First, for any $f, g \in \mathcal{F}$, expand the norm.
    \begin{equation}
        \begin{aligned}
     &\|\hat{f}_{1,t\to r}(g) - \hat{f}_{1,t\to r}^\theta(g)\|^2 \\
     &= \langle \hat{f}_{1,t\to r}(g) - \hat{f}_{1,t\to r}^\theta(g), \hat{f}_{1,t\to r}(g) - \hat{f}_{1,t\to r}^\theta(g)\rangle\\
     &=\|\hat{f}_{1,t\to r}(g) \|^2 + \|\hat{f}_{1,t\to r}^\theta(g)\|^2-2\langle\hat{f}_{1,t\to r}(g), \hat{f}_{1,t\to r}^\theta(g)\rangle,
        \end{aligned}
    \end{equation}
    and similarly
    \begin{equation}
        \begin{aligned}
        &||\frac{r-t}{1-r}((1-t)\frac{\partial}{\partial_t}\hat{f}_{1,t\to r}(g)\\
        &+D\hat{f}_{1,t\to r}(g)[\hat{f}_{1,t}^f(g)-g])+\hat{f}_{1,t}^f(g) -\hat{f}_{1,t\to r}^\theta(g)||_2^2\\
        &=\|\frac{r-t}{1-r}((1-t)\frac{\partial}{\partial_t}\hat{f}_{1,t\to r}(g)+D\hat{f}_{1,t\to r}(g)[\hat{f}_{1,t}^f(g)-g])\\
        &+\hat{f}_{1,t}^f(g)\|^2+\|\hat{f}_{1,t\to r}^\theta(g)\|^2-2\langle \frac{r-t}{1-r}((1-t)\frac{\partial}{\partial_t}\hat{f}_{1,t\to r}(g)\\
        &+D\hat{f}_{1,t\to r}(g)[\hat{f}_{1,t}^f(g)-g])+\hat{f}_{1,t}^f(g),\hat{f}_{1,t\to r}^\theta(g)\rangle.
        \end{aligned}
    \end{equation}

Note that the first term in both expressions is independent of the model parameters, so we focus on analyzing the remaining two terms.  First, we show that the second term in both expressions is identical, i.e., $\mathbb{E}_{t,r,g\sim\mu_t}[\|\hat{f}_{1,t\to r}^\theta(g)\|^2] = \mathbb{E}_{t,r,g\sim\mu^f_t,f\sim \mu_1}[\|\hat{f}_{1,t\to r}^\theta(g)\|^2]$
\begin{equation}
    \begin{aligned}
        &\mathbb{E}_{t,r,g\sim\mu_t}[\|\hat{f}_{1,t\to r}^\theta(g)\|^2]\\
        &= \int_0^1\int_0^1\int_g \|\hat{f}_{1,t\to r}^\theta(g)\|^2 \mathrm{d}\mu_t(g) \mathrm{d}t\mathrm{d}r\\
        &\overset{\textcircled{1}}{=} \int_0^1\int_0^1\int_f\int_g \|\hat{f}_{1,t\to r}^\theta(g)\|^2 \mathrm{d}\mu^f_t(g) \mathrm{d}\nu(f)\mathrm{d}t\mathrm{d}r\\
        &= \mathbb{E}_{t,r,g\sim\mu_t^f,f\sim \mu_1}[\|\hat{f}_{1,t\to r}^\theta(g)\|^2],
    \end{aligned}
\end{equation}
where (\textcircled{1}) follows from the relationship between $\mu_t$ and $\mu_t^f$ given in \autoref{eq:ecpectional_conditional}.

Then, we show that the third term in both expressions is identical, i.e., $\langle\hat{f}_{1,t\to r}(g), \hat{f}_{1,t\to r}^\theta(g)\rangle = \langle \frac{r-t}{1-r}((1-t)\frac{\partial}{\partial_t}\hat{f}_{1,t\to r}(g)+D\hat{f}_{1,t\to r}(g)[\hat{f}_{1,t}^f(g)-g])+\hat{f}_{1,t}^f(g),\hat{f}_{1,t\to r}^\theta(g)\rangle$
\begin{equation}
    \begin{aligned}
        &\mathbb{E}_{t,r,g\sim\mu_t}[\langle\hat{f}_{1,t\to r}(g), \hat{f}_{1,t\to r}^\theta(g)\rangle]\\
        &= \int_0^1\int_0^1\int_g \langle\hat{f}_{1,t\to r}(g), \hat{f}_{1,t\to r}^\theta(g)\rangle \mathrm{d}\mu_t(g) \mathrm{d}t\mathrm{d}r\\
        &\overset{\textcircled{1}}{=}  \int_0^1\int_0^1\int_g\langle \frac{r-t}{1-t}((1-t)\frac{\partial}{\partial t}\hat{f}_{1,t\to r}(g)\\
        &+D\hat{f}_{1,t\to r}(g)[\hat{f}_{1,t}(g)-g])+\hat{f}_{1,t}(g), \hat{f}_{1,t\to r}^\theta(g)\rangle \mathrm{d}\mu_t(g) \mathrm{d}t\mathrm{d}r\\
        &\overset{\textcircled{2}}{=}  \int_0^1\int_0^1\int_g\langle \frac{r-t}{1-t}((1-t)\frac{\partial}{\partial t}\hat{f}_{1,t\to r}(g)\\
        &+D\hat{f}_{1,t\to r}(g)[\int_\mathcal{F} \hat{f}_{1,t}^f(g) \frac{\mathrm{d}\mu_t^f}{\mathrm{d}\mu_t}\mathrm{d}\nu(f)-g])\\
        &+\int_\mathcal{F} \hat{f}_{1,t}^f(g) \frac{\mathrm{d}\mu_t^f}{\mathrm{d}\mu_t}\mathrm{d}\nu(f), \hat{f}_{1,t\to r}^\theta(g)\rangle \mathrm{d}\mu_t(g) \mathrm{d}t\mathrm{d}r\\
        &\overset{\textcircled{3}}{=}  \int_0^1\int_0^1\int_\mathcal{F}\int_g\langle \frac{r-t}{1-t}((1-t)\frac{\partial}{\partial t}\hat{f}_{1,t\to r}(g)\\
        &+D\hat{f}_{1,t\to r}(g)[ \hat{f}_{1,t}^f(g)-g])\\
        &+ \hat{f}_{1,t}^f(g), \hat{f}_{1,t\to r}^\theta(g)\rangle \frac{\mathrm{d}\mu_t^f}{\mathrm{d}\mu_t}\mathrm{d}\mu_t(g) \mathrm{d}\nu(f) \mathrm{d}t\mathrm{d}r\\
        &=  \int_0^1\int_0^1\int_\mathcal{F}\int_g\langle \frac{r-t}{1-t}((1-t)\frac{\partial}{\partial t}\hat{f}_{1,t\to r}(g)\\
        &+D\hat{f}_{1,t\to r}(g)[ \hat{f}_{1,t}^f(g)-g])\\
        &+ \hat{f}_{1,t}^f(g), \hat{f}_{1,t\to r}^\theta(g)\rangle \mathrm{d}\mu_t^f(g) \mathrm{d}\nu(f) \mathrm{d}t\mathrm{d}r\\
        &=\mathbb{E}_{t,r,g\sim\mu_t^f,f\sim \mu_1}[\langle \frac{r-t}{1-t}((1-t)\frac{\partial}{\partial t}\hat{f}_{1,t\to r}(g)\\
        &+D\hat{f}_{1,t\to r}(g)[ \hat{f}_{1,t}^f(g)-g])+ \hat{f}_{1,t}^f(g), \hat{f}_{1,t\to r}^\theta(g)\rangle],
    \end{aligned}
\end{equation}
where $\textcircled{1}$ follows by substituting \autoref{eq:conditional_x1_prediction}, $\textcircled{2}$ applies the relationship between $u_t$ and $u_t^f$ given in \autoref{eq:maginal_conditional_x1}, and $\textcircled{3}$ uses the exchangeability between Bochner integrals and inner products, together with the Fubini–Tonelli theorem.  Therefore,  we have $\mathcal{\tilde L}_c^\mathcal{F}(\theta) = \mathcal{\tilde L}^M(\theta) + C$ where $C$ is independent of the model parameters~$\theta$. 
\end{proof}

\section{Model Architecture and Details of Dataset, Training and Sampling}\label{appendix:model}
\subsection{Real-World Functional Generation}\label{appendix:model_scientific}

\paragraph{Models}  For the real-world Functional Generation experiments, including 1D time-series and 2D Navier–Stokes data, we follow the setup of \cite{kerrigan2024functional} and compare with FFM, DDO, FDDPM and GANO. The implementations of GANO, DDPM, and DDO are directly adopted from \cite{kerrigan2024functional}; please refer to \cite{kerrigan2024functional} for additional details.  FFM employs a 4-layer Fourier Neural Operator (FNO) implemented using the NeuralOperator library. Following \cite{kerrigan2024functional}, we use linearly interpolated spatial coordinates in $[0,1]$ as explicit position embeddings and scale the temporal condition $t$ by $10^{-3}$ as a time embedding. The spatial and temporal embeddings are concatenated with the input data, yielding a total input dimension of data channels$+ 2$.  Our method adopts the same architecture as FFM but introduces two temporal conditions, $t$ and $r$. Both are scaled and concatenated with spatial embeddings and input data, resulting in data channels$+ 3$ input dimensions.  Dataset-specific configurations, including the number of Fourier modes, input channels, hidden channels, projection channels, spatial dimensionality, and total parameter count, are summarized in \autoref{tab:fno_config}.

\begin{table}[H]
\centering
\caption{FNO configuration for different datasets in real-world functional generation experiments. 
Each model uses a 4-layer FNO implemented with the \texttt{neuralop} library.}
\vspace{0.3em}
\resizebox{\linewidth}{!}{
\begin{tabular}{lcccccc}
\toprule
\textbf{Dataset} & \textbf{Fourier Modes} & \textbf{Input Channels} & \textbf{Hidden Channels} & 
\textbf{Projection Channels} & \textbf{Spatial Dim.} & \textbf{Total Params} \\
\midrule
AEMET & 64 & 4 & 256 & 128 & 1D & 9.4M \\
Gene & 16 & 4 &  256 & 128 & 1D & 3.2M \\
Population & 32 & 4 &  256 & 128 & 1D & 5.3M \\
GDP & 32 & 4 &  256 & 128& 1D & 5.3M \\
Labor & 32 & 4 & 256 & 128 & 1D & 5.3M \\
Navier–Stokes (2D) & (32, 32) & 5 & 128 & 256 & 2D & 35.9M \\
\bottomrule
\end{tabular}
}
\label{tab:fno_config}
\end{table}

\paragraph{Dataset\&Metrics}Following \cite{kerrigan2024functional}, our experiments cover both 1D time-series and 2D Navier–Stokes functional datasets.
The Navier–Stokes dataset consists of 2D incompressible fluid flow solutions on a $64\times64$ periodic grid, originally introduced by \cite{li2022learning}. To reduce redundancy and improve training efficiency, we randomly sample 20,000 frames from the original dataset for training.  The 1D time-series category includes five datasets: AEMET, Gene, Population, GDP, and Labor.  The AEMET dataset contains 73 temperature curves recorded by weather stations in Spain between 1980 and 2009, each represented over 365 daily points.  The Gene Expression dataset comprises 156 gene-activity time series measured across 20 time steps.  The Population dataset provides population trajectories for 169 countries from 1950 to 2018 (69 time points).
The GDP dataset records GDP-per-capita time series for 145 countries over the same 69-year span.  The Labor dataset contains quarterly labor-force measurements from 2017 to 2022 (24 time points) for 35 countries.

%Following \cite{kerrigan2024functional}, our experiments cover both 1D time-series and 2D Navier–Stokes functional datasets.  The Navier–Stokes dataset consists of 2D incompressible fluid flow solutions on a $64\times64$ periodic grid, originally introduced by \cite{li2022learning}. To reduce redundancy and improve training efficiency, we randomly sample 20,000 frames from the original dataset for training.  The 1D time-series category includes five datasets: AEMET, Gene, Population, GDP, and Labor.  The AEMET dataset contains 73 temperature curves recorded by weather stations in Spain between 1980 and 2009, each represented as a function over 365 uniformly sampled daily points.  The Gene Expression dataset comprises 156 time series of gene activity across 20 evenly spaced time steps, filtered to retain only those with large temporal variation (standard deviation above 0.3 after normalization).  The Population dataset provides time series of population growth for 169 countries from 1950 to 2018 (69 points in time), normalized by each country’s mean population to capture relative trends.  The GDP dataset records GDP per capita for 145 countries from 1950 to 2018 (69 time points), processed similarly to the population dataset.  The Labor dataset contains quarterly labor force measurements from 2017 to 2022 (24 time points) for 35 countries.

Following \cite{kerrigan2024functional}, for the five 1D time-series datasets, we evaluate the quality of generated functions using a set of statistical functionals, including mean, variance, skewness, kurtosis, and autocorrelation. For each functional, we compute its value over all generated functions and compare it with the corresponding ground-truth statistics from the real dataset using mean squared error (MSE). This captures the model’s ability to reproduce key statistical characteristics of temporal signals.  For the 2D Navier–Stokes dataset, we employ two complementary distribution-level metrics. \textbf{Density MSE} measures the statistical discrepancy between the marginal value distributions of real and generated samples. Each dataset is flattened into scalar values representing pointwise function evaluations, from which continuous probability densities are estimated via kernel density estimation (KDE). The mean-squared difference between the estimated densities quantifies how well the generated data reproduce the overall statistical distribution of function values. \textbf{Spectrum MSE} evaluates the discrepancy between the average Fourier energy spectra of real and generated samples. Each sample is transformed into the frequency domain using a 2D FFT, and spectral energies are aggregated over wavenumber bands and averaged across the dataset. The resulting mean-squared error reflects the model’s ability to match the multi-scale energy distribution of the target fluid dynamics.  

\paragraph{Training\&Sampling}  During training and sampling, Gaussian processes with a Matérn kernel are used to sample the initial noise functions accurately.  For the 1D datasets, we use a kernel length of 0.01 and a kernel variance of 0.1, while for the 2D Navier–Stokes dataset we use a kernel length of 0.01 and a kernel variance of 1.0.  All models are trained using the Adam optimizer. The training and sampling procedures for GANO, DDPM, DDO, and FFN follow \cite{kerrigan2024functional}, and we refer readers to that work for implementation details.  For our method, in the 1D setting we use an initial learning rate of $1\times10^{-3}$. For the AEMET dataset, the learning rate is reduced by a factor of 0.1 after 50 epochs, while no decay is applied for the other datasets.  In the 2D setting, we use an initial learning rate of $5\times10^{-4}$, which is decayed by a factor of 0.5 every 40 epochs.  Consistent with Mean Flow, we employ an adaptive loss function $\mathcal{L} = w|\Delta|_2^2$, $w = \frac{1}{(|\Delta|_2^2 + c)^p}$ where $\Delta$ denotes the regression error, $c > 0$ is a small stabilizing constant (set to $10^{-3}$ in our experiments), and $p = 0.75$.  The time variables $t$ and $r$ are sampled using a lognormal distribution with mean $-0.4$ and variance $0.01$ for the 1D datasets, and uniformly from the interval $[0,1]$ for the 2D dataset. Their values are swapped whenever $t > r$.  By default, $r$ is set equal to $t$ with a probability of 0.25, except for the Population, GDP, and Labor datasets, where a probability of 0.125 is used.

\subsection{Image Generation Based on Functional}\label{appendix:model_image}
\paragraph{Models} As illustrated in \autoref{fig:illustration_image_generation_detals}, we adopt the neural architecture of Infty-Diff\cite{bond2024infty}, where both the input and output are continuous image functions represented by randomly sampled subsets of coordinates. To handle such sparse input–output mappings, the network consists of two components: a Sparse Neural Operator and a Dense U-Net/UNO. 

The Sparse Neural Operator processes the irregularly sampled pixels and maps them into feature vectors on the same subsets of coordinates. These features are then interpolated onto a lower-resolution dense grid using k-nearest neighbor (KNN) interpolation with neighborhood size 3. On the dense grid, a Dense U-Net/UNO is applied to extract high-level representations.  Following Infty-Diff’s observation that U-Net and UNO yield comparable results, we employ the U-Net for simplicity. The dense U-Net operates on a $128^2$ base grid for image datasets with a resolution of $256^2$, with $128$ base channels and five resolution levels whose channel multipliers are $[1, 2, 4, 8, 8]$. Self-attention modules are inserted after the $16^2$ and $8^2$ resolution stages to enhance global context aggregation.  After dense processing, the resulting features are inversely interpolated back to the coordinate subsets using KNN. The reconstructed features are further refined through another Sparse Neural Operator, and the final output is obtained via a residual connection with the initial sparse features.

Following the implementation guidelines of Infty-Diff, we adopt a linear-kernel Sparse Neural Operator implemented with TorchSparse for efficiency. Each Sparse Operator module consists of one pointwise convolution layer, three linear-kernel convolution operator layers, and another pointwise convolution layer. Each operator layer consists of a sparse depthwise convolution with 64 channels and a kernel size of 7 for $256^2$-resolution images, followed by two pointwise convolution layers with 128 internal channels.  

For time conditioning in both the Sparse Neural Operator and the Dense Network, we use positional embeddings \cite{vaswani2017attention} to encode the time variable following Mean Flow~\cite{geng2025mean}. The resulting embeddings of $t$ and $r$ are added to replace the original time-embedding conditional input in Infty-Diff. In total, the network comprises $\sim 420M$ trainable parameters.

\begin{figure}[t]
  \centering
  \includegraphics[width=\linewidth]{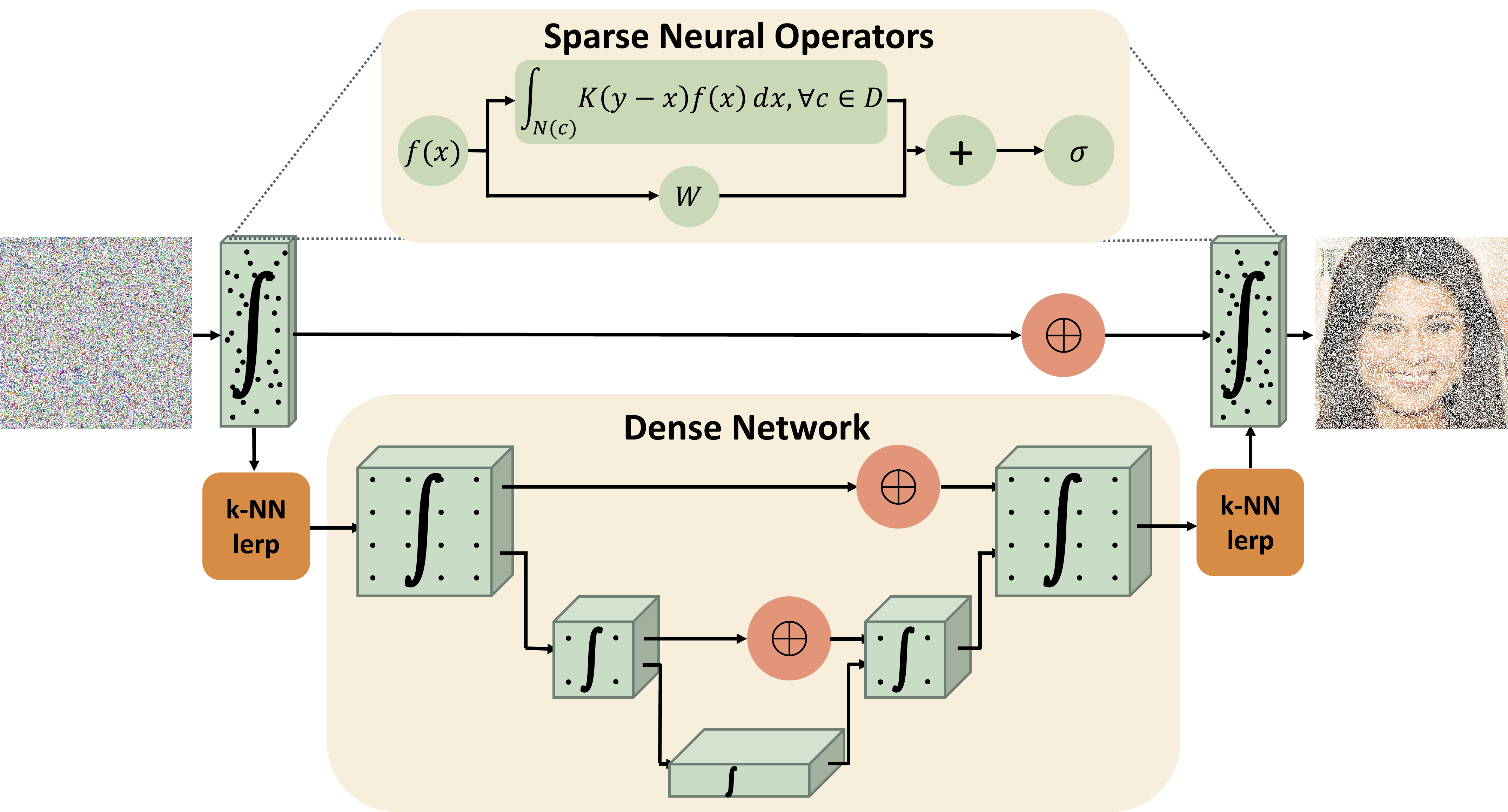}
    \vspace{-0.2cm}
  \caption{\textbf{Hybrid Sparse–Dense Neural Operator for Infty-Diff.}  The model for functional-based image generation follows the hybrid sparse–dense Neural Operator design in Infty-Diff, where both the input and output are functions represented by randomly sampled pixels. The architecture consists of a Sparse Neural Operator and a Dense Network, and this figure presents the internal structure corresponding to the schematic in \autoref{fig:illustration_image_generation}}\label{fig:illustration_image_generation_detals}
  \vspace{-0.6cm}
\end{figure}

\paragraph{Dataset\&Metrics} We trained our models on three unconditional datasets CelebA-HQ \cite{karras2017progressive}, FFHQ \cite{karras2019style}, and LSUN-Church \cite{yu2015lsun} and one conditional dataset AFHQ \cite{choi2020stargan}.  All datasets were resized to a resolution of $256 \times 256$ with LSUN-Church center-cropped along the shorter side. Following Infty-Diff \cite{bond2024infty}, we randomly sampled one quarter of the pixels during training to validate functional-based generation.  CelebA-HQ is a high-quality derivative of CelebA containing 30,000 high-resolution human face images. FFHQ includes 70,000 diverse face images, LSUN-Church provides around 126,000 images of churches, and AFHQ comprises 15,000 animal images from three categories: cats, dogs, and wild animals, which are used as conditional labels for conditional generation.  To evaluate the models, we generated 50K samples from each trained model and compared them with the corresponding real datasets.  Following \cite{kynkaanniemi2023role}, we computed the Fréchet Inception Distance (FID)~\cite{heusel2017gans} using CLIP features~\cite{radford2021learning} extracted from the clip\_vit\_b\_32 encoder, which correlates better with human perception of image quality, especially for multiple-resolution generation.  We denote this metric as $\text{FID}_{\text{CLIP}}$.  For completeness and comparison with previous works, we also report the standard FID computed using Inception-V3 features, denoted as $\text{FID}$.

%\begin{figure*}[t]
%  \centering
%  \includegraphics[width=\textwidth]{images/functional_diffusion_picture3.png}
%  \vspace{-0.2cm}
%  \caption{
%  \textbf{The Architecture of \textbf{Functional Diffusion}}.  Both the input and output of the model are functions, which are represented by randomly sampled spatial points and their corresponding function values.  The input function is referred to as the context, and the output function as the query.  Each context point and its associated value jointly form the functional representation vector $\mathcal{X}$. 
%  The context points and their values first interact with the functional representation vector $\mathcal{X}$ through cross-attention, followed by cascaded cross- and self-attention modules that progressively yield the latent vector representing the output function. The latent vector then interacts with the query points through another cross-attention layer to predict the query values, together forming the output function.  Following \cite{zhang2024functional}, the model is conditioned on 64 target surface points, from which it reconstructs the target surface.
%  }
%  \label{fig:illustration_shape_generation_detals}
%  \vspace{-0.5cm}
%\end{figure*}
\paragraph{Training\&Sampling}   In the functional-based image generation experiments, we follow Infty-Diff \cite{bond2024infty} and use mollified white noise to approximate an infinite-dimensional Gaussian process sampled from a Gaussian measure $\mathcal{N}(0,C_0)$ as the initial noise function for training and sampling. Specifically, the white noise is convolved with a Gaussian kernel $k(\cdot)$ to ensure that the resulting samples lie in the Hilbert space $\mathcal{F}$. The mollification is expressed as
\begin{equation}
    \begin{aligned}
h(c) &= \int_{\mathbb{R}^n} K(c - y, l) x(y) dy\\
K(y, l) &= \frac{1}{(4\pi l)^{n/2}} e^{-\frac{|y|^2}{4l}},        
    \end{aligned}
\end{equation}
where $l > 0$ is the smoothing parameter, which we set to one pixel width in our experiments. For image data, we take $n = 2$. Here, $x(y)$ denotes the original white noise before smoothing, and $h(c)$ represents the mollified function after applying the Gaussian kernel. We also apply the same Gaussian mollification to the training images in the dataset to improve regularity and ensure the data distribution satisfies the integrability requirements of the function space. Following Infty-Diff, the generated mollified output is then restored to a sharper image using a Wiener-filter–based approximate inverse defined in the Fourier domain as
\begin{equation}
    \begin{aligned}
\tilde{x}(\omega) =
\frac{e^{-\omega^2 t}}
{e^{-2(\omega^2 t)} + \epsilon^2}
\hat{h}(\omega),        
    \end{aligned}
\end{equation}
where $\epsilon$ denotes an estimate of the inverse signal-to-noise ratio (SNR).
Here, $\hat{h}(\omega)$ and $\tilde{x}(\omega)$ denote the Fourier transforms of the mollified function $h(c)$ and the reconstructed output $x(c)$, respectively.  This operation effectively recovers high-frequency details while maintaining numerical stability. 

All models are trained on four NVIDIA L40s GPUs using the Adam optimizer for a total of 800K steps, with the learning rate gradually reduced from an initial value of $5.0\times10^{-5}$ to a final value of $7.8\times10^{-7}$. The total batch size is set to 16 for all experiments.  Consistent with Mean Flow, we employ an adaptive loss $\mathcal{L}=w|\Delta|_2^2$, where $\Delta$ denotes the regression error and $w=1/(|\Delta|_2^2+c)^p$ with a small constant $c>0$ (e.g., $10^{-3}$). We set $p=0.5$ for all experiments.  The time variables $t$  and $r$ are sampled  uniformly from the interval $[0,1]$, swapping their values whenever $t>r$, and setting $r=t$ with a probability of 0.5.  Following Infty-Diff \cite{bond2024infty}, we employ the manner of Diffusion Autoencoder \cite{preechakul2022diffusion} to mitigate stochasticity arising from the high variance of randomly sampled coordinate subsets.

\subsection{3D Shape Generation}\label{appendix:model_3D}

\begin{figure}[t]
  \centering
  \includegraphics[width=\linewidth]{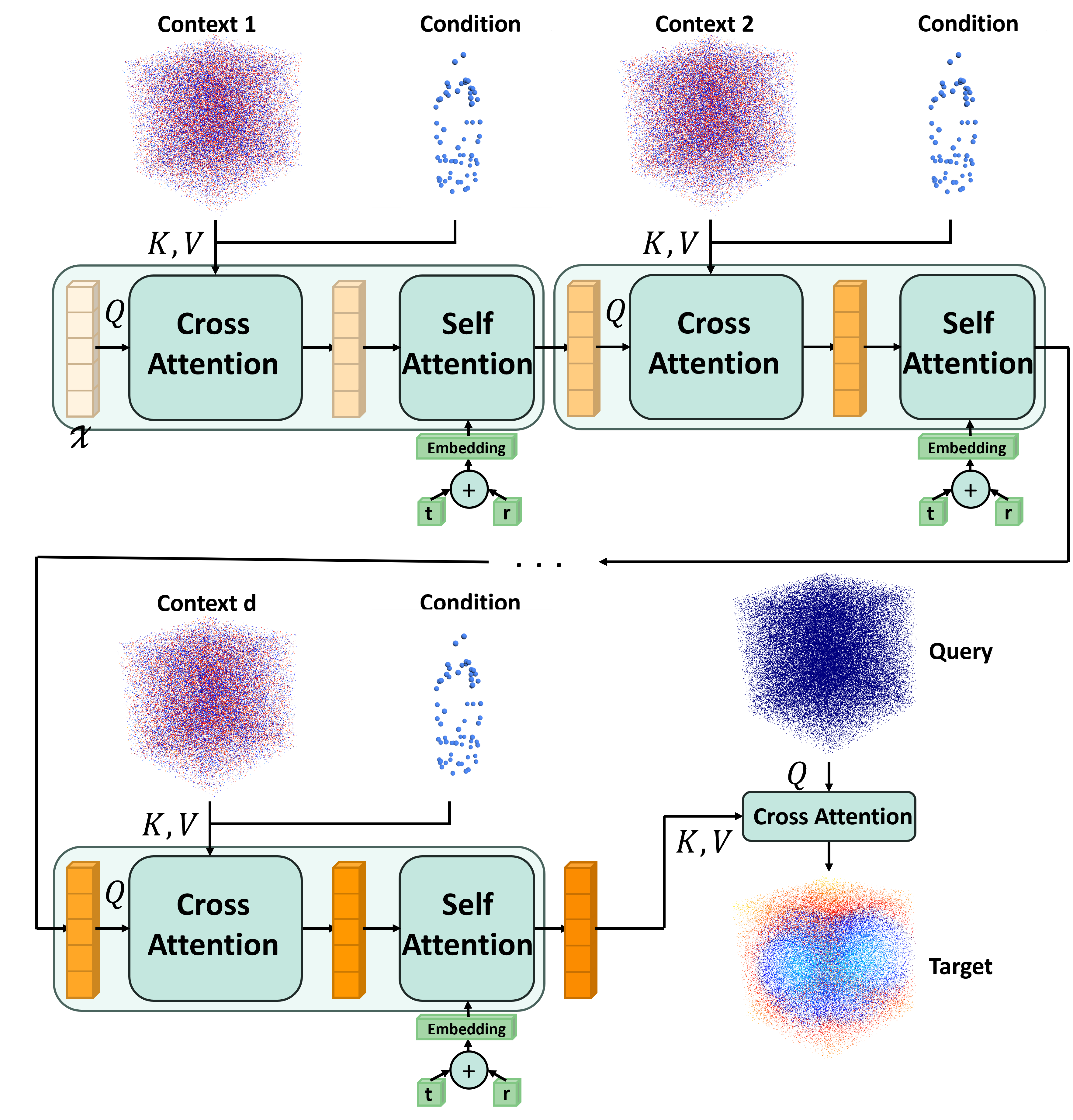}
    \vspace{-0.7cm}
  \caption{
  \textbf{The Architecture of \textbf{Functional Diffusion}}.  Both the input and output of the model are functions, which are represented by randomly sampled spatial points and their corresponding function values.  The input function is referred to as the context, and the output function as the query.  Each context point and its associated value jointly form the functional representation vector $\mathcal{X}$. 
  The context points and their values first interact with the functional representation vector $\mathcal{X}$ through cross-attention, followed by cascaded cross- and self-attention modules that progressively yield the latent vector representing the output function. The latent vector then interacts with the query points through another cross-attention layer to predict the query values, together forming the output function.  Following \cite{zhang2024functional}, the model is conditioned on 64 target surface points, from which it reconstructs the target surface.
  }
  \vspace{-0.2cm}
  \label{fig:illustration_shape_generation_detals}
\end{figure}

\paragraph{Models.}  We use the architecture of Functional Diffusion \cite{zhang2024functional} for 3D shape generation as shown in \autoref{fig:illustration_shape_generation_detals}. In this model, both the input and output functions are represented by randomly sampled spatial points and their corresponding function values, enabling a continuous functional representation independent of discretized grids. Specifically, the input function $f_c$ is represented on context points $\{x_c^i\}_{i=1}^n$ with values $\{v_c^i\}_{i=1}^n$, where $v_c^i = f_c(x_c^i)$, and the output function $f_q$ is represented on query points $\{x_q^j\}_{j=1}^m$ with values $\{v_q^j\}_{j=1}^m$ where $v_q^j = f_q(x_q^j)$. $n$ and $m$ are the number of context points and query points respectively.  The Functional Diffusion framework naturally supports the case where the context and query points differ, allowing flexible mappings between input and output functions.

Following \cite{zhang2024functional}, we evenly partition the context set $(\{x_c^i\}_{i=1}^n,\{v_c^i\}_{i=1}^n)$ into $d$ groups $I_1,\dots,I_d$, where $\cup_{k=1}^d I_k=
\{1,\dots,n\}$. Each group is processed by an attention block consisting of a cross-attention and a self-attention module. The cross-attention module takes the embedding of $(\{x_c^i\}_{i\in I_k},\{v_c^i\}_{i\in I_k})$ and a latent vector. The latent vector is propagated from the previous attention block, while the first block initializes it with a learnable latent variable $\mathcal{X}$ that represents the functional itself.  For each group $(\{x_c^i\}_{i\in I_k},\{v_c^i\}_{i\in I_k})$, a point-wise embedding is obtained by summing the Fourier positional encoding of spatial coordinates $\{x_c^i\}_{i\in I_k}$ and the embedding of the corresponding values $\{v_c^i\}_{i\in I_k}$. The resulting embeddings are concatenated with a conditional embedding, which may come from semantic labels or partially observed conditional points. In our experiments, we use 64 partially observed surface points as conditional inputs. These combined embeddings are passed through $K$ and $V$ networks to produce the keys and values for the cross-attention, while the latent vector, transformed by a $Q$ network, serves as the query.  The temporal embedding follows Mean Flow \cite{geng2025mean}, employing standard sinusoidal positional encodings \cite{vaswani2017attention}. The sum of the embeddings of $t$ and $r$ replaces the time-embedding input originally used in Functional Diffusion.

After cascading $d$ attention blocks, the resulting latent vector encodes the representation of the target function. This latent is then used as the key and value input to a final cross-attention block, where the query corresponds to the Fourier positional encodings of the query points $\{x_q^j\}_{j=1}^m$. The attention output yields the predicted query values $\{v_q^j\}_{j=1}^m$, forming the representation of the generated function $f_q$ with $\{x_q^j\}_{j=1}^m$.  In our implementation, the embedding dimension is set to 784, the number of groups $d$ is 24, and the attention layers adopt a multi-head mechanism with 8 heads, each with an internal head dimension of 64. 

%The entire network contains approximately $A$ learnable parameters.

\paragraph{Task\&Dataset\&Metrics} We follow the surface reconstruction task setup in Functional Diffusion \cite{zhang2024functional}, where the model is required to reconstruct a target surface given 64 observed points sampled from that surface. Specifically, the generative model is conditioned on these 64 target surface points to generate the SDF function of the corresponding complete surface from the initial noise.  Consistent with Functional Diffusion \cite{zhang2024functional}, we use the ShapeNet-CoreV2~\cite{chang2015shapenet} dataset, which contains 57,000 3D models across 55 object categories. Following the same preprocessing procedure \cite{zhang2024functional,zhang20233dshape2vecset,zhang20223dilg}, we convert each ShapeNet mesh into a signed distance field (SDF) and randomly sample $n=49152$ points from the domain $[0,1]^3$ to obtain the context points and their corresponding SDF values. Separately, we sample another $m =2048$ points and their SDF values to form the query points and query values, while a distinct set of surface points is sampled near the zero-level set of the SDF to serve as the conditioning input. During training, each data instance consists of context points $\{\tilde x_c^i\}_{i=1}^n$ and values $\{\tilde v_c^i\}_{i=1}^n$, query points $\{\tilde x_q^j\}_{j=1}^m$ and values $\{\tilde v_q^j\}_{j=1}^m$, and a randomly selected set of 64 surface points $\{C^l\}_{l=1}^{64}$ used as the conditional input.  Note that $(\{\tilde x^i_c\}_{i=1}^n, \{\tilde v^i_c\}_{i=1}^n)$ and $(\{\tilde x^j_q\}_{j=1}^m, \{\tilde v^j_q\}_{j=1}^m)$ correspond to the same sample, the reference SDF field $f \sim \nu$ from the dataset, rather than the context and query functions used as the input $g$ and output $\hat{f}_{1,t\to r}(g)$ in Functional Mean Flow training; the construction of the input and reference output functions from each data instance is detailed in the Training\&Sampling paragraph.

We evaluate the model using Chamfer Distance, F1-score, and Boundary Loss, following \cite{zhang2024functional,zhang20233dshape2vecset,zhang20223dilg}. The Chamfer Distance measures the average bidirectional distance between the generated and target point sets, while the F1-score quantifies the precision–recall trade-off of the reconstructed surface points. The Boundary Loss assesses geometric fidelity near the zero-level surface and is formally defined as $\text{Boundary}(f) = \frac{1}{|\mathcal{E}_\Omega|} \sum_{i \in \mathcal{E}_\Omega}| f(\mathbf{x}_i) - q(\mathbf{x}_i) |^2$  where $\mathcal{E}_\Omega$ denotes the set of sampled spatial points near the surface boundary, $f(\mathbf{x}_i)$ represents the predicted SDF value, and $q(\mathbf{x}_i)$ is the ground-truth SDF. This metric measures the mean squared deviation between predicted and true SDFs in the boundary region, capturing the fine-grained accuracy of surface reconstruction. For evaluation, both Chamfer Distance and F1-score are computed by uniformly sampling 50K points on each surface, whereas Boundary Loss is computed using 100K sampled points.  We follow the same data split as \cite{zhang2024functional}, training the model on the training split and evaluating it on the test split.

\paragraph{Training\&Sampling.} In our 3D shape generation experiments, we follow Functional Diffusion \cite{zhang2024functional} and approximate samples from the Gaussian measure by linearly interpolating white noise defined on a coarse $64^3$ lattice over the 3D domain. This interpolated field serves as an efficient estimate of a Gaussian process sample, providing a computationally practical alternative to direct Gaussian process sampling and substantially improving sampling efficiency in the 3D setting.  Since we adopt the $x_1$-prediction variant of Functional Mean Flow for model training, the neural network input is $g =(1-(1-\sigma_{\text{min}})t)f_0+tf$ in \autoref{eq:selection_conditional}, where $f$ denotes the function sampled from the dataset. In practice, $f$ corresponds to each function represented by the instance $(\{\tilde x_c^i\}_{i=1}^n, \{\tilde v^i_c\}_{i=1}^n)$ and $(\{\tilde x^j_q\}_{j=1}^m, \{\tilde v^j_q\}_{j=1}^m)$ from dataset. The network output is then used jointly with the predicted $\hat{f}_{1,t}^f(g)$ to compute the training loss \autoref{eq:meanflow_x1_loss}. Therefore, context points and context values that represent the input function $g$ to the network are calculated as $x^i_c=\tilde x^i_c$ and $v^i_c=t\tilde v^i_c+(1-(1-\sigma_\text{min})t)r^i_c$ from the the instance $(\{\tilde x^i_c\}_{i=1}^n, \{\tilde v^i_c\}_{i=1}^n)$ and $(\{\tilde x^j_q\}_{j=1}^m, \{\tilde v^j_q\}_{j=1}^m)$, whereas the query points representing the output function and the corresponding reference query values of $\hat{f}_{1,t}^f(g)$ are obtained following $x^j_q=\tilde x^j_q$ and $v^j_q = \tilde v^j_q + \sigma_{\text{min}}r^j_q$,  because $\hat{f}_{1,t}^f(g)$ from \autoref{eq:conditional_x1_flowmatching} can be computed as
\begin{equation}
\begin{aligned}
\hat{f}_{1,t}^f(g)
&= \frac{\sigma_{\min}}{1 - (1 - \sigma_{\min})t}(g - t f) + f \\
&= \sigma_{\min} f_0 + f
\end{aligned}
\end{equation}
Here $r_c^i$ and $r_q^j$ represent the values of initial noise function $f_0 \sim \mu_0$ in \autoref{eq:selection_conditional} and \autoref{eq:conditional_x1_flowmatching} evaluated at the context and query points $x_c^i$ and $x_q^j$, respectively.  Besides the separate reference term $\hat{f}_{1,t}^f(g)$ in \autoref{eq:meanflow_x1_loss}, the term $D\hat{f}_{1,t\to r}(g)[\hat{f}_{1,t}^f(g) - g]$ also involves an $\hat{f}_{1,t}^f(g)$. This instance of $\hat{f}_{1,t}^f(g)$ should be evaluated using the context points, denoted as $(\{\tilde x^i_c\}_{i=1}^n, \{\tilde v^i_c + \sigma_{\text{min}}r^i_c)\}_{i=1}^n)$, since it serves as part of the functional derivative with respect to the network input $g$.

The model is trained on four H200 GPUs using the Adam optimizer for a total of 200K steps. The learning rate is gradually reduced from an initial value of $A$ to a final value of $B$. The total batch size is set to 16 for all experiments. Consistent with Mean Flow, we employ an adaptive loss function $\mathcal{L} = w|\Delta|_2^2$, $w = \frac{1}{(|\Delta|_2^2 + c)^p}$ where $\Delta$ denotes the regression error, $c > 0$ is a small stabilizing constant (set to $10^{-3}$ in our experiments), and $p = 0.5$.  The time variables $t$ and $r$ are uniformly sampled from the interval $[0,1]$, with their values swapped whenever $t > r$, and $r$ is set equal to $t$ with a probability of 0.5.  

During sampling, we randomly draw $n = 49152$ points from the domain $[0,1]^3$ as context points, and use a dense $128^3$ grid as query points. After predicting the SDF values on the $128^3$ grid, the final mesh surface is reconstructed using the Marching Cubes algorithm.

\section{Example Python Implementation}\label{appendix:implementation}

\subsection{Unified Implementation}\label{appendix:unified_implementation}
\subsubsection{Unified Implementation for $u$-prediction Functional Mean Flow}
In the following training code, \verb|gp_like(g)| is the Gaussian Process sampling function (for finite-dimensional cases it can be replaced by \verb|randn_like()|), and \verb|sample_t_r()| is the time sampling function. \verb|u(g, r, t)| denotes the learned model whose input and output are functions in a specific representation, while \verb|f| denotes a batch of training data under the same representation. The parameter \verb|sigma_min| corresponds to $\sigma_{\text{min}}$ in \autoref{eq:selection_conditional}, and \verb|metric| denotes the loss function.
\begin{lstlisting}[style=pythonstyle]
    t, r = sample_t_r()
    f_0 = gp_like(f)
    
    coef = 1-sigma_min
    g = (1 - coef*t) * f_0 + t * f
    v = coef/(1-coef*t)*(t*f-g)+ f
    
    u, dudt = jvp(u, (g, t, r), (v, 1, 0))
    u_tgt = (r - t) * dudt + v
    error = u - stopgrad(u_tgt)
    
    loss = metric(error)
\end{lstlisting}
The following code is for inference
\begin{lstlisting}[style=pythonstyle]
    f_0 = gp_like(f)
    f=u(f_0,0,1)+f_0
\end{lstlisting}
\subsubsection{Unified Implementation for $x_1$-prediction Functional Mean Flow}
In the following code, \verb|gp_like(g)| is the Gaussian Process sampling function (for finite-dimensional cases it can be replaced by \verb|randn_like()|), and \verb|sample_t_r()| is the time sampling function. \verb|x1(g, r, t)| denotes the learned model whose input and output are functions in a specific representation, while \verb|f| denotes a batch of training data under the same representation. The parameter \verb|sigma_min| corresponds to $\sigma_{\text{min}}$ in \autoref{eq:conditional_x1_prediction}, and \verb|metric| denotes the loss function.
\begin{lstlisting}[style=pythonstyle]
    t, r = sample_t_r()
    f_0 = gp_like(f)
    
    coef = 1-sigma_min
    g = (1 - coef*t) * f_0 + t * f
    f1_f = sigma_min/(1-coef*t)*(g-t*f)+ f
    
    f1, df1dt = jvp(x1, (g, t, r), (f1_f-g, 1-t, 0))
    f1_tgt = (r - t)/(1-r) * df1dt + f1_f
    error = f1 - stopgrad(f1_tgt)
    
    loss = metric(error)
\end{lstlisting}
The following code is for inference
\begin{lstlisting}[style=pythonstyle]
    f_0 = gp_like(f)
    f=x1(f_0,0,1)
\end{lstlisting}
\subsection{3D SDF-Specific Implementation}\label{appendix:specific_implementation}
As discussed in \autoref{appendix:model_3D}, each training instance in Functional Diffusion consists of a tuple $(\{\tilde x^i_c\}_{i=1}^n, \{\tilde v^i_c\}_{i=1}^n,\{\tilde x^j_q\}_{j=1}^m, \{\tilde v^j_q\}_{j=1}^m,\{C^l\}_{l=1}^{64})$, representing (context points, context values, query points, query values, and condition). The pairs $\{\tilde x^i_c\}_{i=1}^n, \{\tilde v^i_c\}_{i=1}^n$ and $\{\tilde x^j_q\}_{j=1}^m, \{\tilde v^j_q\}_{j=1}^m$ correspond to two different samplings of the same reference SDF function $f \sim \nu$, differing only in their spatial locations $\tilde x^i_c$ and $\tilde x^j_q$. It is therefore crucial in implementation to clearly distinguish between these two representations to avoid ambiguity during training. In practice, the input function $g \sim \mu^f_t$ is constructed as $(\{\tilde x_c^i\}_{i=1}^n,\{t\tilde v_c^i+(1-(1-\sigma_{\text{min}})t)r_c^i\}_{i=1}^n)$, the separate reference function $\hat f_{1,t}^f(g)$ is evaluated as $(\{\tilde x_q^j\}_{j=1}^m,\{\tilde v_q^j+\sigma_{\text{min}}r_q^j\}_{j=1}^m)$, and the reference function $\hat f_{1,t}^f(g)$ in $D\hat{f}_{1,t\to r}(g)[\hat{f}_{1,t}^f(g) - g]$ is evaluated as $(\{\tilde x_c^i\}_{i=1}^n,\{\tilde v_c^i+\sigma_{\text{min}}r_c^i\}_{i=1}^n)$, where $r^i_c$ and $r^j_q$ denote the values of the initial noise function $f_0 \sim \mu_0$ in \autoref{eq:selection_conditional} and \autoref{eq:conditional_x1_flowmatching}, evaluated at the context and query points $x^i_c$ and $x^j_q$, respectively.

In the code, \verb|xc|, \verb|vc|, \verb|xq|, \verb|vq|, and \verb|cond| respectively denote the context point $\{\tilde x^i_c\}_{i=1}^n$, context value $\{\tilde v^i_c\}_{i=1}^n$, query point $\{\tilde x^j_q\}_{j=1}^m$, query value $\{\tilde v^j_q\}_{j=1}^m$, and condition $\{C^l\}_{l=1}^{64}$, respectively. As mentioned earlier, Functional Diffusion constructs the initial Gaussian measure using linear interpolation over a random value on coarse grid.  In the training code, \verb|s| specifies the coarse grid resolution, and \verb|interpolate(rg, xc)| performs interpolation from the random grid values \verb|rg| to the sample points \verb|xc|.
The function \verb|sample_t_r()| is the time sampling function, and \verb|x1(xc, g, xq, r, t, cond)| denotes the learned model, where \verb|(xc, g)| represents the input function $g$, and \verb|xq| specifies the query points for the output function. \verb|B| is batch size.
\begin{lstlisting}[style=pythonstyle]
    rg = torch.randn(B, 1,s, s, s)    
    rc = interpolate(rg,xc)
    rq = interpolate(rg,xq)
    t, r = sample_t_r()
    coef = 1-sigma_min
    g = (1 - coef*t) * rc + t * vc
    f1_f_c = sigma_min*rc+vc
    f1_f_q = sigma_min*rq+vq
    x1_partial = partial(x1, xc = xc, xq = xq, cond=cond)
    f1, df1dt = jvp(x1_partial, (g, t, r), (f1_f_c-g,1-t,0))
    f1_tgt = (r - t)/(1-r) * df1dt + f1_f_q
    error = f1 - stopgrad(f1_tgt)
    loss = metric(error)
\end{lstlisting}
As mentioned earlier, during inference we use a dense $128^3$ grid as query points. In the following code, these query points are denoted as \verb|xqg|. The following code is for inference
\begin{lstlisting}[style=pythonstyle]
    rg = torch.randn(B, 1,s, s, s)  
    rc = interpolate(rg,xc)
    vqg = x1(xc, rc, xqg, 0, 1, cond)
\end{lstlisting}
\section{Additional Results\&Experiments}\label{appendix:additional_results}
\subsection{Instability of $u$-Prediction Mean Flow in Shape Generation with SDF}\label{appendix:additional_instability}
In other tasks, the performance of the $u$-prediction variant of Functional Mean Flow is generally comparable to that of the $x_1$-prediction version. However, in the shape generation experiments, we observe that the Functional Diffusion framework becomes highly unstable when trained with $u$-prediction Functional Mean Flow, indicating that $u$-prediction Functional Mean Flow is not well-suited for 3D shape generation within the Functional Diffusion framework. 

To illustrate this finding, we conduct a 2D experiment using the MNIST \cite{lecun1998mnist} dataset, which is converted into signed distance fields (SDFs) and trained under the Functional Diffusion setup for 2D shape generation. The embedding dimension is set to $256$, the number of groups $d$ is $8$, and each attention layer adopts a multi-head mechanism with 8 heads and an internal head dimension of $64$. The entire network contains approximately $19$M learnable parameters. During training, we use a batch size of $64$ and initialize the time embedding with a standard variance of $0.01$.  We set $t=r$ with $100\%$ probability for time sampling, representing the most stable limiting form of Mean Flow

To monitor potential training failures, we track the batch-averaged spatial variance of the network outputs. A persistent collapse of this variance indicates instability, as an SDF is expected to satisfy $|\nabla f| = 1$ and thus maintain meaningful spatial variation. Once the variance approaches zero and remains there, the model's output effectively degenerates into a constant field, diverging from the ground-truth function and failing to recover. As illustrated in \autoref{fig:training_behavior}, the $x_1$-prediction formulation remains stable across a wide range of learning rates, whereas the $u$-prediction model exhibits variance collapse even at relatively small learning rates.

After training for 10K steps, we visualize the generated samples in \autoref{fig:mnist_three_compare}. The number of sampling steps is fixed to 64.  It can be observed that the $x_1$-prediction Mean Flow successfully generates valid SDFs, from which the zero-level sets can be extracted as clear binary handwritten digits.  In contrast, the $u$-prediction variant consistently produces noisy outputs regardless of the learning rate.

\begin{figure*}[t]
    \centering
    \begin{subfigure}{0.32\textwidth}
        \centering
        \includegraphics[width=\linewidth]{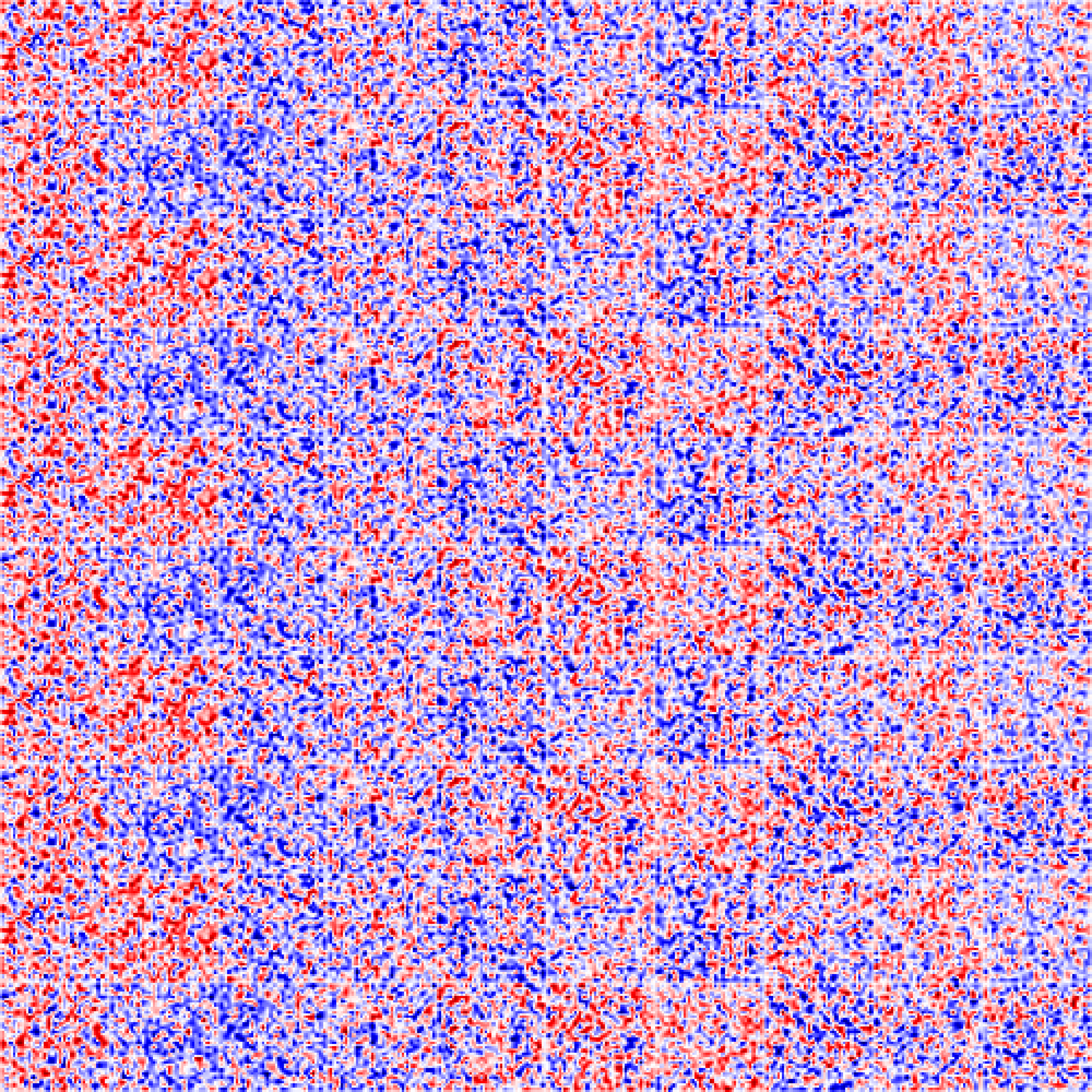}
        \caption{MNIST $u$-prediction}
        \label{fig:mnist_u_predict}
    \end{subfigure}
    \hfill
    \begin{subfigure}{0.32\textwidth}
        \centering
        \includegraphics[width=\linewidth]{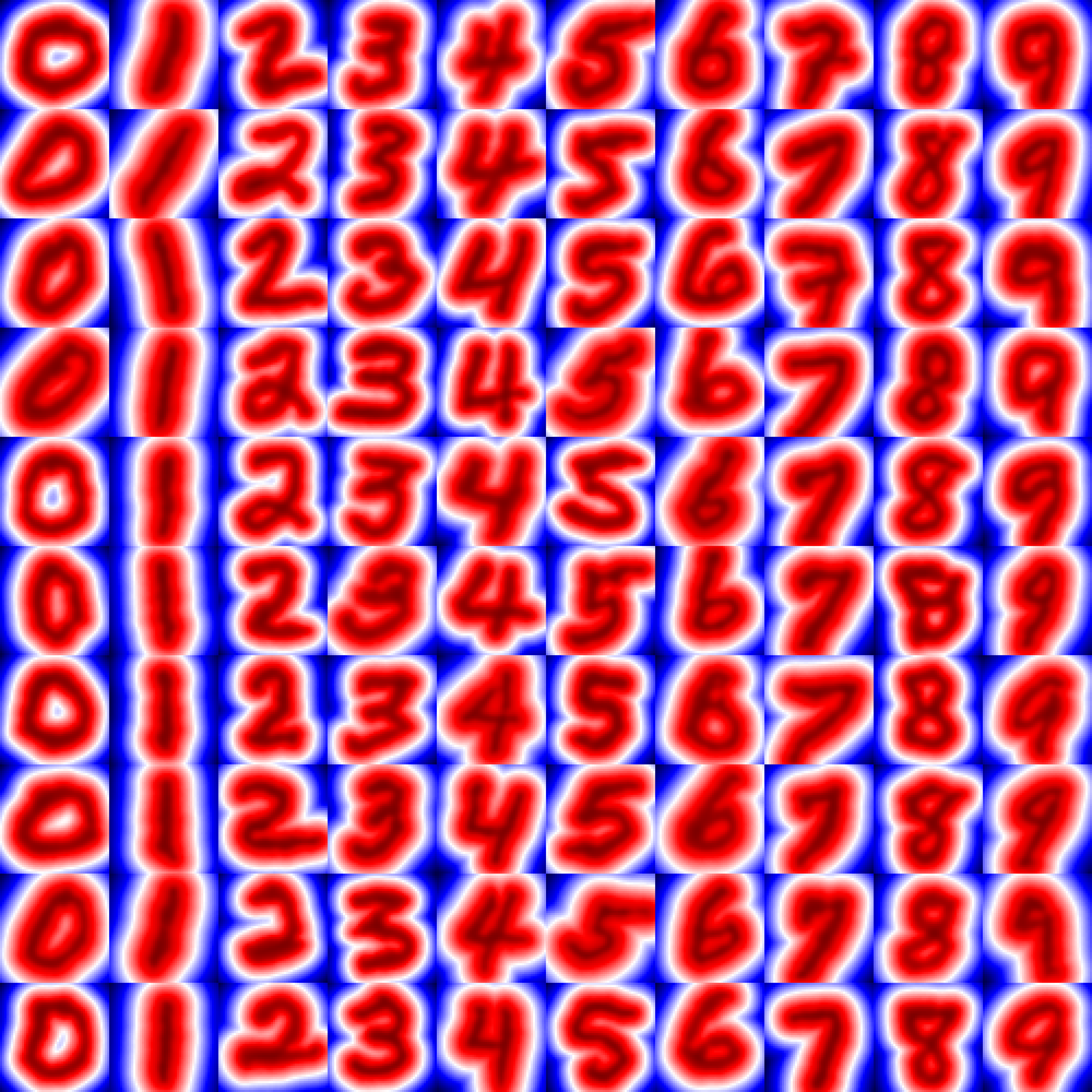}
        \caption{MNIST $x_1$-prediction (SDF)}
        \label{fig:mnist_x1_predict_sdf}
    \end{subfigure}
    \hfill
    \begin{subfigure}{0.32\textwidth}
        \centering
        \includegraphics[width=\linewidth]{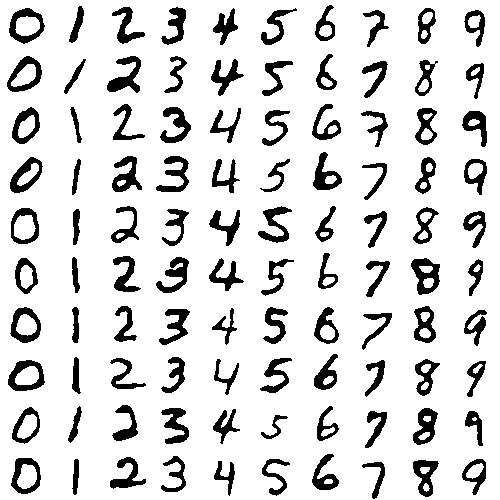}
        \caption{MNIST $x_1$-prediction (binary)}
        \label{fig:mnist_x1_predict_bw}
    \end{subfigure}

    \caption{
            Comparison of MNIST generation results across the $u$-prediction and $x_1$-prediction FMF variants within the Functional Diffusion framework.  After 10K training steps, the $x_1$-prediction FMF produces valid SDFs (b), from which clear binary digits can be extracted via their zero-level sets (c).  In contrast, the $u$-prediction model fails completely: once the output variance collapses, the model cannot update the initial noise into a meaningful SDF, as shown in (a) and rows correspond to learning rates from $10^{-4}$ to $10^{-6}$.
    }
    \label{fig:mnist_three_compare}
\end{figure*}

%\begin{figure*}[t]
%    \centering
%    \begin{subfigure}[t]{0.25\textwidth}
%        \centering
%        \includegraphics[width=\textwidth]{images/u_pred_var.png}
%        \caption{Variance of $u$-prediction Output}
%    \end{subfigure}\hfill
%    \begin{subfigure}[t]{0.25\textwidth}
%        \centering
%        \includegraphics[width=\textwidth]{images/x1_pred_var.png}
%        \caption{Variance of $x_1$-prediction Output}
%    \end{subfigure}\hfill
%    \begin{subfigure}[t]{0.25\textwidth}
%        \centering
%        \includegraphics[width=\textwidth]{images/u_pred_loss.png}
%        \caption{Loss of $u$-prediction Output}
%    \end{subfigure}\hfill
%    \begin{subfigure}[t]{0.25\textwidth}
%        \centering
%        \includegraphics[width=\textwidth]{images/x1_pred_loss.png}
%        \caption{Loss of $x_1$-prediction Output}
%    \end{subfigure}
%    \caption{Training behavior of $u$- vs. $x_1$-prediction networks: spatial variance differences and collapse.  Both networks are trained for 10K steps under learning rates ranging from $1\times10^{-4}$ to $1\times10^{-6}$.  The $u$-prediction network exhibits extremely low spatial variance in its outputs and oscillating losses after a rapid initial decrease, indicating a collapse to a near-constant function from which a valid SDF cannot be recovered.  In contrast, the $x_1$-prediction network maintains normal spatial variance consistent with SDFs satisfying $\|\nabla f\|=1$, and its loss decreases smoothly, suggesting stable optimization behavior.}
%    \label{fig:training_behavior}
%\end{figure*}

\subsection{Additional Results for Functional-Based Image Generation}\label{appendix:additional_image_results}
In \autoref{fig:ffhq_multi_res}, \autoref{fig:church_multi_res}, and \autoref{fig:afhq_multi_res}, we provide additional qualitative results on the FFHQ~\cite{karras2019style}, LSUN-Church~\cite{yu2015lsun}, and AFHQ~\cite{choi2020stargan} datasets, respectively.
The same model is used to synthesize images at arbitrary resolutions under different noise levels.
Notably, the model is trained only on randomly sampled 1/4 subsets of pixels from $256 \times 256$ images.

\begin{figure*}[t]
  \includegraphics[width=1.0\textwidth]{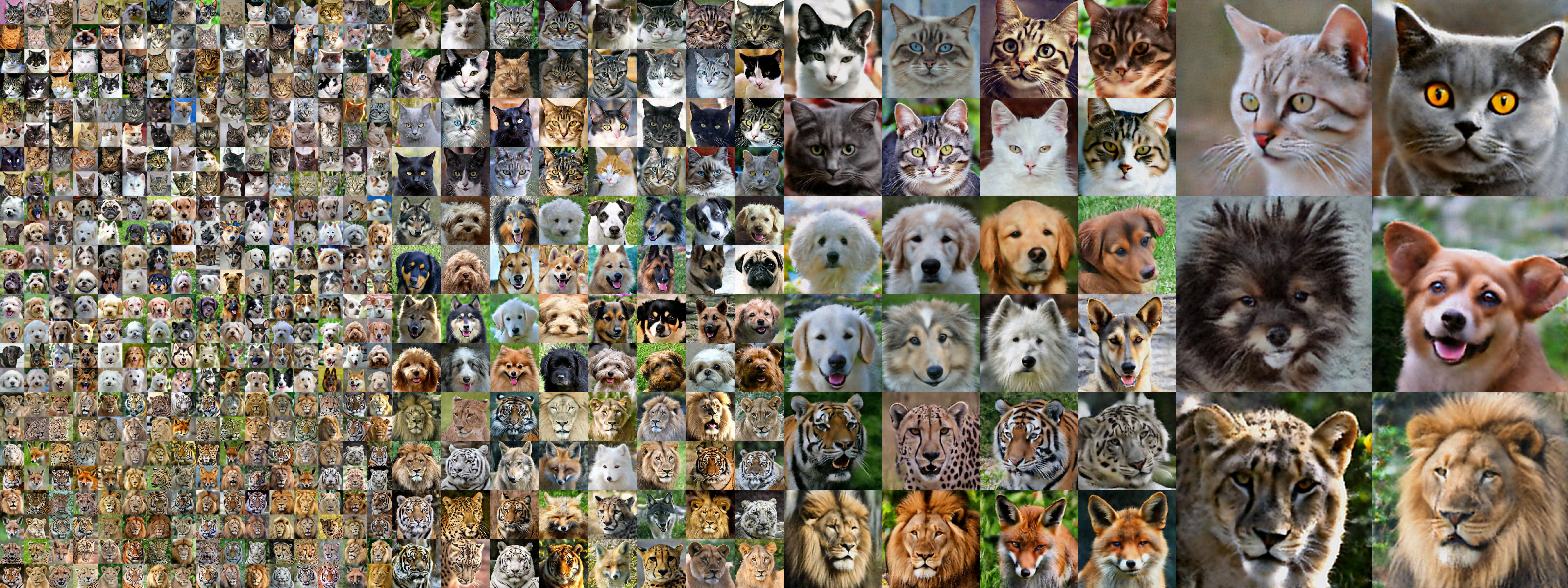}
  \caption{Additional results on AFHQ. The model is trained on randomly sampled 1/4 subsets of pixels from $256 \times 256$ images and evaluated at different resolutions.  Left to right: $64 \times 64$, $128 \times 128$, $256 \times 256$, and $512 \times 512$. Top to bottom: cat, dog, and wild animal categories.}
  \label{fig:afhq_multi_res}
\end{figure*}
\begin{figure*}[t]
  \includegraphics[width=1.0\textwidth]{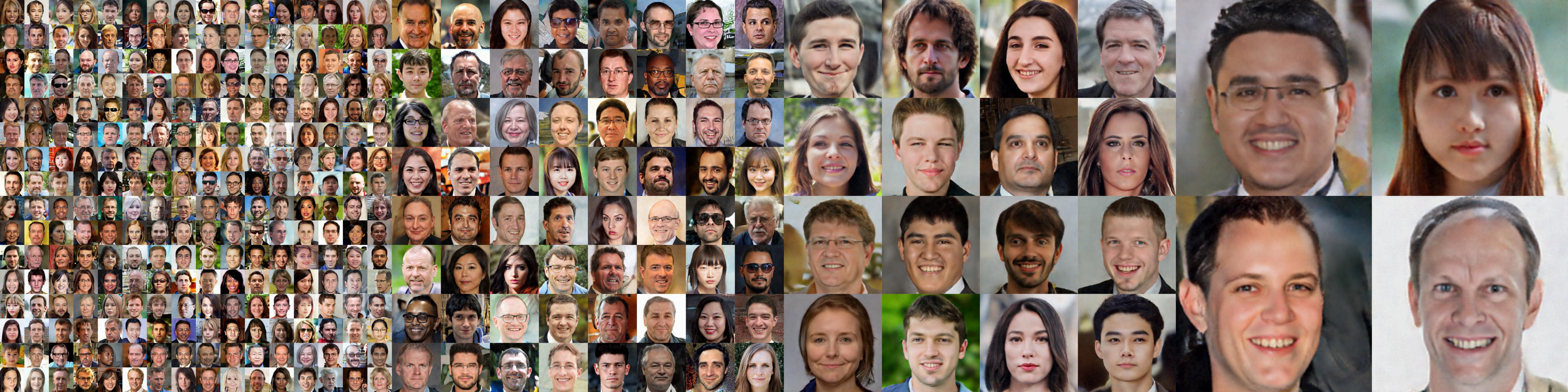}
  \caption{Additional results on FFHQ. The model is trained on randomly sampled 1/4 subsets of pixels from $256 \times 256$ images and evaluated at different resolutions.  Left to right: $64 \times 64$, $128 \times 128$, $256 \times 256$, and $512 \times 512$.}
  \label{fig:ffhq_multi_res}
\end{figure*}
\begin{figure*}[t]
  \includegraphics[width=1.0\textwidth]{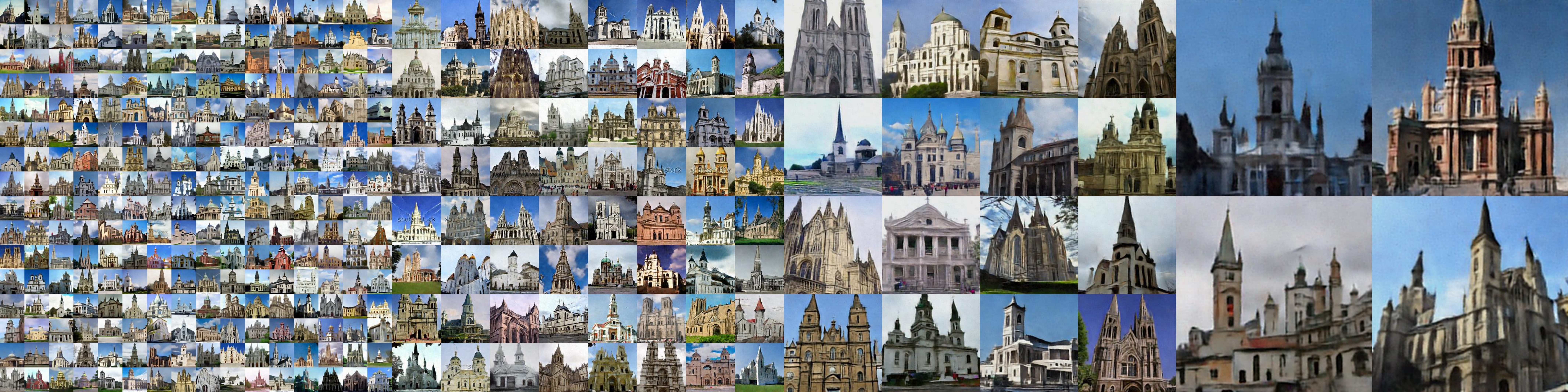}
  \caption{Additional results on LSUN-Church. The model is trained on randomly sampled 1/4 subsets of pixels from $256 \times 256$ images and evaluated at different resolutions.  Left to right: $64 \times 64$, $128 \times 128$, $256 \times 256$, and $512 \times 512$.}
  \label{fig:church_multi_res}
\end{figure*}

\end{document}